\documentclass[11pt,a4paper]{article}
\usepackage[utf8]{inputenc}
\usepackage[T1]{fontenc}
\usepackage{lmodern}
\usepackage[english]{babel}
\usepackage[margin=2.5cm]{geometry} % Page margins
\usepackage{graphicx} % For including figures
\usepackage{booktabs} % For better-looking tables
\usepackage{natbib} % For author-year citations
\usepackage{amsmath} % For mathematical equations
\usepackage{amssymb} % For AMS symbols
\usepackage{amsthm} % For AMS symbols
\usepackage{setspace} % For line spacing
\usepackage{caption} % For customizing captions
\usepackage{subcaption}
\usepackage{xltabular}
\usepackage[colorlinks=true, citecolor=blue, linkcolor=blue, urlcolor=blue]{hyperref} % For hyperlinks
\usepackage{algorithm, algpseudocode}

\newcommand{\norm}[1]{\left\lVert #1\right\rVert}
\newcommand{\E}{\mathbb E}
\DeclareMathOperator*{\argmin}{arg\,min}
\newcommand{\diag}{\mathrm{diag}}
\newcommand{\T}{T^{-\alpha}}

\newtheorem{condition}{Condition}
\newtheorem{definition}{Definition}
\newtheorem{remark}{Remark}
\newtheorem{example}{Example}
\newtheorem{model}{Model}
\newtheorem{theorem}{Theorem}
\newtheorem{corollary}{Corollary}
\newtheorem{lemma}{Lemma}
\newtheorem{proposition}{Proposition}

\setcitestyle{round}

\let\oldfootnote\footnote
\renewcommand{\footnote}{\fontsize{9}{11}\selectfont\oldfootnote}

\title{Spectral Embeddings of Degree-$\alpha$ Laplacians in Random Dot Product Graphs}
\author{
    John Park \\
    University of Hong Kong \\
    \and
    Ning Hao \thanks{Corresponding author: nhao@arizona.edu}\\
    University of Arizona \\
}
\date{} % Remove date

\begin{document}
\maketitle
\begin{abstract}
Spectral clustering methods for network data are commonly based on a few matrix representations, such as the adjacency matrix and the symmetric Laplacian. We study a continuum of degree-normalized spectral embeddings that includes these commonly used choices as special cases. Under a random dot product graph model, we establish a row-wise central limit theorem for this family of embeddings. The result provides an explicit description of how degree normalization affects both population geometry and the local uncertainty of embedded nodes. We use the limiting distributions to compare different normalizations in two-community stochastic block models through a projected-Gaussian Bayes-error diagnostic. These comparisons show that no single normalization is uniformly preferred. Instead, the favored normalization depends on network density, community imbalance, and block-probability structure. Typically, stronger normalization is favored in lower-density or more imbalanced settings. These results provide a unified distributional understanding of when and why alternative normalizations may improve spectral clustering.
\end{abstract}

% AMS subject classification


\section{Introduction}
\label{Sec::Introduction}
Community detection is a fundamental statistical problem in network analysis. Given a single observed graph, the goal is to partition the nodes into groups with distinct connectivity patterns. A widely used approach to this task is spectral clustering \citep{Von2007, Ng2001}, which first embeds each node as a point in a low-dimensional space using the eigenvectors of a graph matrix, then partitions the resulting points with a clustering algorithm such as $k$-means. The performance of spectral clustering therefore depends on the graph matrix used to construct the embedding.

Two standard choices are the adjacency matrix $A$ and the symmetric Laplacian matrix $D^{-1/2}AD^{-1/2}$, where $D$ is the diagonal matrix of observed node degrees. A large body of literature has studied alternative normalizations and regularizations of the adjacency matrix \citep{Jing2022, Amini2013, Qing24, QinRohe2013, sarkar2015role, priebe2019two, Modell2021}, demonstrating that the choice of graph matrix can substantially affect the performance of spectral clustering. In this paper, we focus on one basic component of that choice: the amount of degree normalization applied before embedding.

A natural one-parameter family generalizing the adjacency and symmetric Laplacian matrices is
\[
D^{-\alpha}AD^{-\alpha},\qquad \alpha\in[0,1].
\]
The choices $\alpha=0$ and $\alpha=1/2$ recover the adjacency and symmetric Laplacian matrices, respectively, while larger values of $\alpha$ apply stronger degree normalization. To the best of our knowledge, within the statistical literature on community detection, this continuous degree-power family was first studied by \citet{ali2018improved}, who considered a centered modularity formulation. Related uncentered formulations, including regularized variants, were subsequently studied by \citet{Fan26} and \citet{Park2025spectral}.

In this paper, we study the uncentered family and investigate how the amount of normalization simultaneously changes the geometry and uncertainty of the spectral embedding. Our central question is whether different values of $\alpha$ can be systematically preferred for community detection and, if so, which features of the underlying network determine that preference.

To preview the main geometric phenomenon, consider a degree-corrected stochastic block model (DCSBM) \citep{Karrer2011}. In this model, nodes are partitioned into communities, and each node is assigned an individual degree-correction parameter. Figure~\ref{Fig::DCSBM_Alpha_Comparison} shows the embeddings of a simple two-community DCSBM in which the nodes are randomly assigned one of three degree-correction parameters; Section~\ref{Sec::BlockModels} contains full modeling details. For the adjacency spectral embedding, corresponding to $\alpha=0$ and shown in Figure~\ref{Fig::DCSBM_alpha000}, nodes from the same community, indicated by color, lie along a common ray from the origin, while the degree-correction parameter, indicated by shade, determines their position along that ray. Figures~\ref{Fig::DCSBM_alpha050} and~\ref{Fig::DCSBM_alpha1} illustrate the first-order effect of increasing the amount of normalization. As $\alpha$ increases, the radial dependence on the degree-correction parameter is progressively reduced; at $\alpha=1$, it disappears entirely, and the population centers depend only on community membership.

This radial variation makes it problematic to apply a clustering algorithm directly to the spectral embedding, because nodes in the same community may be separated by their degree parameters. This motivates the common preprocessing step of normalizing each embedding row to unit norm before clustering \citep{Lei2015,QinRohe2013}. Projecting the rows onto the unit sphere removes the within-community radial variation, so that the population centers from each community collapse to a common direction. After the usual orthogonal alignment, these community directions are independent of both the degree-correction parameter and $\alpha$. Thus, after spherical row normalization, the different degree-$\alpha$ embeddings are equivalent at the first order, i.e., they have the same population community directions.

They are not, however, equivalent at the second order. The Gaussian probability ellipses in Figure~\ref{Fig::DCSBM_Alpha_Comparison} show that the magnitude, shape, and orientation of the local fluctuations change with $\alpha$. In this example, the ellipses become increasingly elongated in the direction of the line joining the corresponding community centers. Different values of $\alpha$ can thus produce different amounts of overlap around the same population community directions. Population geometry and consistency results alone therefore cannot explain why one value of $\alpha$ may outperform another, and a more careful analysis of the second-order deviations is required. 

\begin{figure}[htpb]
     \centering
     \begin{subfigure}[b]{0.31\textwidth}
         \centering
         \includegraphics[width=\textwidth]{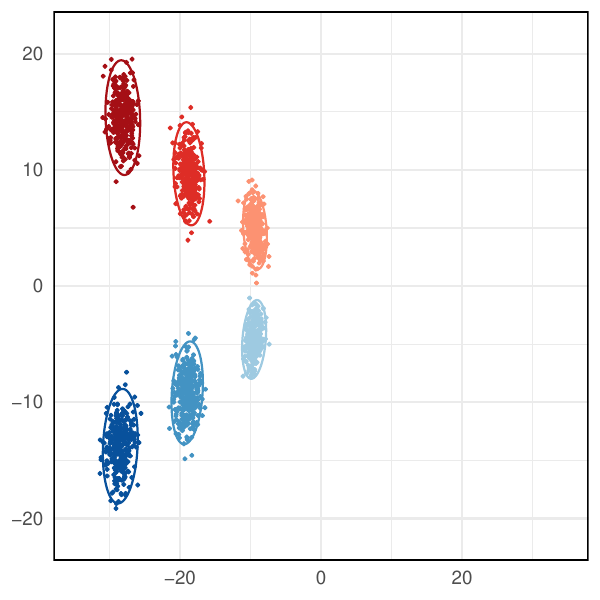}
         \caption{$\alpha=0$}
         \label{Fig::DCSBM_alpha000}
     \end{subfigure}
     \hfill 
     \begin{subfigure}[b]{0.31\textwidth}
         \centering
         \includegraphics[width=\textwidth]{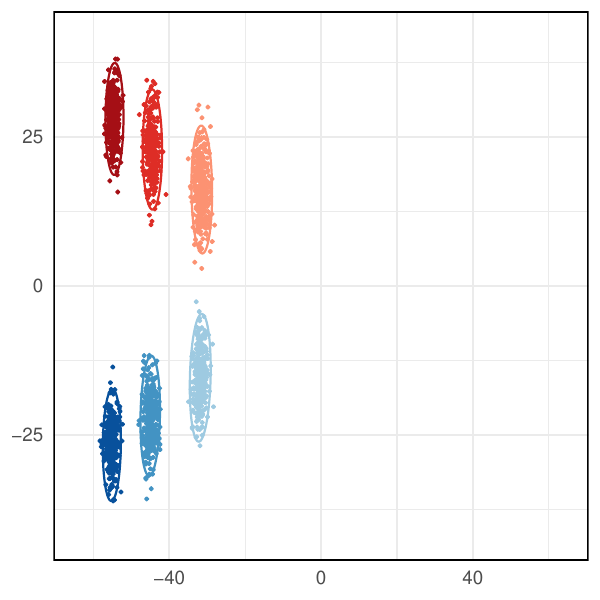}
         \caption{$\alpha=0.5$}
         \label{Fig::DCSBM_alpha050}
     \end{subfigure}
     \hfill
     \begin{subfigure}[b]{0.31\textwidth}
         \centering
         \includegraphics[width=\textwidth]{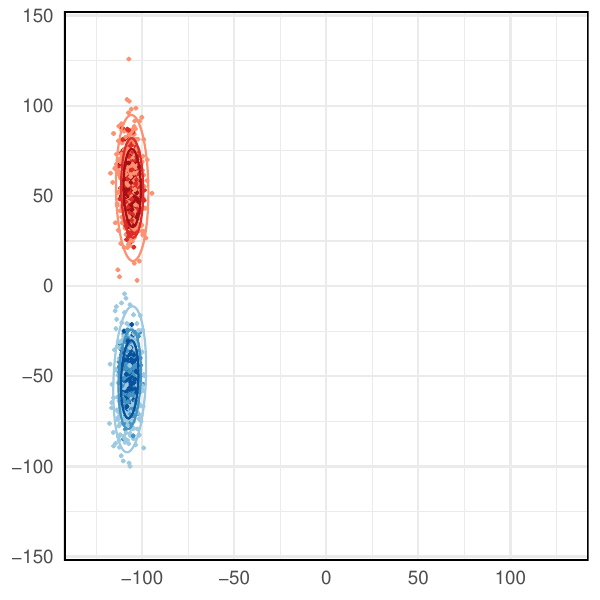}
         \caption{$\alpha=1$}
         \label{Fig::DCSBM_alpha1}
     \end{subfigure}
     \caption{Comparison of the degree-$\alpha$ spectral embeddings for a two-community DCSBM. Points are colored by community, with shading indicating the degree-correction parameter. The overlaid contours represent the 95\% probability ellipses obtained from Corollary~\ref{Cor::DCS}. Note that the axis scales differ between plots.}
     \label{Fig::DCSBM_Alpha_Comparison}
\end{figure}

We analyze this phenomenon under the random dot product graph (RDPG) model, which includes the DCSBM with positive-semidefinite block-probability matrices as a special case. Our main theoretical result is a conditional row-wise central limit theorem for the uncentered degree-$\alpha$ spectral embedding. The result covers both dense and sparse regimes, requiring the expected-degree scale to grow faster than \((\log n)^2\), and gives the limiting mean and covariance explicitly in terms of the latent-position distribution.

We then use the limiting distributions implied by the central limit theorem to study how degree normalization affects row-normalized spectral clustering. For each value of $\alpha$, we project the Gaussian approximations onto the unit sphere and compute the Bayes error for distinguishing the resulting community distributions. The minimizer of this error defines a normalization preferred under this Bayes-error criterion. 

Because the criterion relies on an asymptotic approximation and an idealized Bayes rule, it need not minimize the actual error for an observed network. We therefore compare its prediction with empirical misclassification rates in simulated networks. Across the two-community block models studied, stronger normalization tends to be favored when expected degrees are smaller or community proportions are more imbalanced.

Row-wise central limit theorems for adjacency and symmetric Laplacian spectral embeddings are well established; see, respectively, \citet{athreya2016limit} and \citet{Tang2018}, as well as the overview of \citet{agterberg2026overview}. More recently, \citet{Fan26} establish asymptotic normality for a broad class of generalized and regularized Laplacian matrices that includes the degree-$\alpha$ family. Our result complements this general theory by working directly under the RDPG model, yielding explicit characterizations of the population centers and limiting covariances that can be specialized and interpreted under block models.

Two complementary works are particularly relevant to our comparison of degree normalizations. \citet{ali2018improved} study a centered modularity-based version of the degree-power family under a dense heterogeneous DCSBM and select the amount of normalization according to a spectral detectability criterion. \citet{Cape2019} use Chernoff information to compare adjacency and symmetric Laplacian embeddings under stochastic block models. Our analysis addresses a different question: how does degree normalization change the node-level overlap that remains after spherical row normalization? In this sense, we extend the distributional perspective of \citet{Cape2019} across the continuous uncentered family, while targeting a performance criterion distinct from that of \citet{ali2018improved}.

The present paper focuses on explaining the statistical consequences of degree normalization under a specified model. Selecting $\alpha$ from a single observed network is a related but distinct methodological problem. \citet{Park2026degree} study degree-$\alpha$ spectral clustering from this perspective and propose a subsampling-based procedure for tuning $\alpha$.
 
Our contributions are threefold. First, we derive explicit conditional row-wise limiting distributions for degree-$\alpha$ spectral embeddings under the RDPG, covering both dense and sparse regimes. Second, under the DCSBM, we show that spherical row normalization makes the population community directions identical across $\alpha$, while the projected uncertainty remains $\alpha$-dependent. Our formulas characterize how normalization changes its magnitude, shape, and orientation. Third, we compare degree normalizations using a Bayes-error criterion, benchmark this comparison using Chernoff information, and investigate whether the predicted ordering is reflected in finite-sample spectral clustering.

The remainder of the paper is organized as follows. In Section~\ref{Sec::DegreeAlphaLaplacian}, we formally define the degree-$\alpha$ Laplacian and corresponding spectral embedding. In Section~\ref{Sec::RDPGs}, we define random dot product graphs and state the central limit theorem for the embedding in terms of the relevant latent position distribution. Section~\ref{Sec::BlockModels} specializes the theorem to the DCSBM and formally analyzes the behavior observed in Figure~\ref{Fig::DCSBM_Alpha_Comparison}. In Section~\ref{Sec::ComparingDegreeNorm}, we study the implications of degree normalization for spectral clustering, and we conclude in Section~\ref{Sec::Discussion}. 

\section{\texorpdfstring{Degree-$\alpha$ Laplacian Spectral Embeddings}
{Degree-alpha Laplacian Spectral Embeddings}}
\label{Sec::DegreeAlphaLaplacian}
Given an undirected, unweighted graph with no isolated vertices, let $A\in\{0,1\}^{n\times n}$ represent the adjacency matrix, where $a_{ij}=1$ indicates an edge between nodes $i$ and $j$. The degree matrix $D$ is an $n\times n$ diagonal matrix whose diagonal entries are the node degrees, $D_{ii}=\sum_j a_{ij}$. We first define the symmetric Laplacian to be 
\[
    \hat{L}_{\mathrm{sym}}=D^{-1/2}AD^{-1/2}.
\] 
This definition differs from the definition of the symmetric Laplacian sometimes seen in the literature, $L_{\mathrm{sym}}=I_n-D^{-1/2}AD^{-1/2}$. These matrices share eigenvectors, with the leading eigenvectors of the former corresponding to the trailing eigenvectors of the latter \citep{Von2007}. Following common practice in spectral clustering, we work with the normalized matrix itself. We analogously define the observed degree-$\alpha$ Laplacian to be 
\[
    \hat{L}_\alpha=D^{-\alpha}AD^{-\alpha},\qquad\alpha\in[0,1].
\] 
The degree-$\alpha$ Laplacian generalizes other common graph matrices, reducing to the adjacency matrix when $\alpha=0$ and the symmetric Laplacian when $\alpha=1/2$. We now define a spectral embedding for the degree-$\alpha$ Laplacian.

\begin{definition}[Degree-$\alpha$ Laplacian Spectral Embedding]
    \label{Def::DALSE}
    Let $A$ be an $n\times n$ adjacency matrix and let $\hat{U}_{\alpha,R}\hat{\Lambda}_{\alpha,R}\hat{U}_{\alpha,R}^\top$ be the rank-$R$ truncated eigendecomposition of $\hat{L}_\alpha$. That is, if $\hat{\lambda}_{\alpha,1}\geq\hat{\lambda}_{\alpha,2}\geq\cdots\geq\hat{\lambda}_{\alpha,R}$ are the $R$ largest positive eigenvalues of $\hat{L}_\alpha$, and $\hat{u}_{\alpha,1},\hat{u}_{\alpha,2},\ldots,\hat{u}_{\alpha,R}$ are the corresponding orthonormal eigenvectors, then $\hat{U}_{\alpha,R}$ is the $n\times R$ matrix whose columns are $\hat{u}_{\alpha,1},\ldots,\hat{u}_{\alpha,R}$ and $\hat{\Lambda}_{\alpha,R}=\operatorname{diag}(\hat{\lambda}_{\alpha,1},\ldots,\hat{\lambda}_{\alpha,R})$. A spectral embedding of $\hat{L}_\alpha$ is given by $\hat{X}_{\alpha,R}:=\hat{U}_{\alpha,R}\hat{\Lambda}_{\alpha,R}^{1/2}$.
\end{definition}

The dimension of the embedding is often clear from context. In those cases, we suppress the dependence of these matrices on $R$. 

\section{Random Dot Product Graphs}
\label{Sec::RDPGs}
In this section, we introduce the random dot product graph model and make its connection with the spectral embedding defined in the previous section explicit. In the random dot product graph model, each node is assigned a latent position in $\mathbb{R}^r$, and the edge probability between two nodes is given by the dot product between the latent positions \citep{Young2007,Hoff2002,Handcock2007}.

\begin{definition}[Random Graph]
    \label{Def::RG}
    Let $P=(p_{ij})$ be a symmetric $n\times n$ matrix with entries in $[0,1]$. We define an undirected random graph with adjacency matrix $A=(a_{ij})$ as follows. For $i<j$, the entries are independent and satisfy
    \[
        a_{ij}\sim \mathrm{Bern}(p_{ij}).
    \]
    The remaining entries are
    \[
        a_{ij}=
        \begin{cases}0, & i=j,\\
            a_{ji}, & i>j.
        \end{cases}
    \]
\end{definition}

\begin{definition}[Random Dot Product Graph Model]
    \label{Def::RDPG}
    Let $F$ be a distribution on $\mathbb{R}^r$ satisfying
    \[
        x^\top y \in [0,1]\text{ for all }x,y\in\operatorname{Supp}(F).
    \]
    Let $\xi_1,\ldots,\xi_n\overset{\mathrm{iid}}{\sim}F$, and let $\Xi\in\mathbb{R}^{n\times r}$ be the matrix whose $i$th row is $\xi_i^\top$. A random graph follows a random dot product graph model if, conditional on $\xi_1,\ldots,\xi_n$,
    \[
        p_{ij}=\langle \xi_i,\xi_j\rangle.
    \]
    Equivalently, $P=\Xi\Xi^\top$.
\end{definition}

\begin{remark}
    Using the convention that self-loops are excluded, we set the diagonal entries of $A$ to zero. Therefore, in the RDPG model,
    \[
        \mathbb{E}[A\mid \xi_1,\ldots,\xi_n]= P-\operatorname{diag}(P),
    \]
    where $\operatorname{diag}(P)$ denotes the diagonal matrix with the same diagonal as $P$. We denote this version of $P$ by $P^\circ=P-\operatorname{diag}(P)$. In the asymptotic regime of our theoretical results, the diagonal portion of $P$ is negligible for the quantities of interest. Thus, in the main discussion, we sometimes suppress this distinction for clarity. In the appendix, where the distinction is needed for the proofs, we work explicitly with $P^\circ$. 
\end{remark}

The random dot product graph model contains several familiar graph models as special cases. We briefly describe three examples to illustrate the role of the latent positions. 
\begin{example}[Erd\H{o}s--R\'enyi graph] 
    In an Erd\H{o}s--R\'enyi graph, for a given $p\in[0,1]$, every pair of distinct nodes is connected with probability $p$, so $p_{ij}=p$. This model can be written as an RDPG by taking $r=1$ and assigning every node the same latent position $\xi_i=\sqrt{p}$. 
\end{example} 
\begin{example}[Stochastic block model]
    In a stochastic block model (SBM), each node belongs to one of $K$ communities \citep{Holland_Laskey_Leinhardt_1983}. The community labels $z_1,\ldots,z_n$ are independent copies of a random variable $Z$ satisfying $\mathbb{P}(Z=k)=\pi_k$ for $k=1,\ldots,K$. Given a symmetric matrix $B\in[0,1]^{K\times K}$ of community connection probabilities, the edge probabilities are
    \[
        p_{ij}=B_{z_i z_j}.
    \]
    If $B$ is positive semidefinite with rank $r$, choose $\nu_1,\ldots,\nu_K\in\mathbb{R}^r$ such that $B_{k\ell}=\langle \nu_k,\nu_\ell\rangle$. Assigning node $i$ the latent position $\xi_i=\nu_{z_i}$, we may write
    \[
        p_{ij}=B_{z_i z_j}=\langle \nu_{z_i},\nu_{z_j}\rangle=\langle \xi_i,\xi_j\rangle.
    \]
    Thus, the SBM can be represented as an RDPG, with community $k$ associated with the latent position $\nu_k$.
\end{example}
\begin{example}[Degree-corrected stochastic block model]
    \label{Ex::DCSBM}
    The degree-corrected stochastic block model (DCSBM) generalizes the stochastic block model by assigning each node a degree-correction parameter \citep{Karrer2011}. Let $Z$ be a random community label satisfying $\mathbb{P}(Z=k)=\pi_k$ for $k=1,\ldots,K$, and let $\Theta$ be a random variable taking values in $(0,1]$, independent of $Z$. Let $(z_1,\theta_1),\ldots,(z_n,\theta_n)$ be independent copies of $(Z,\Theta)$. Let $B\in[0,1]^{K\times K}$ be a symmetric matrix of community connection probabilities. The edge probabilities are then given by
    \[
        p_{ij}=\theta_i\theta_jB_{z_i z_j}.
    \]
    If $B$ is positive semidefinite with rank $r$, choose $\nu_1,\ldots,\nu_K\in\mathbb{R}^r$ such that $B_{k\ell}=\langle \nu_k,\nu_\ell\rangle$. Assigning node $i$ the latent position $\xi_i=\theta_i\nu_{z_i}$, we may write
    \[
        p_{ij}=\theta_i\theta_jB_{z_i z_j}=\langle \theta_i\nu_{z_i},\theta_j\nu_{z_j}\rangle=\langle \xi_i,\xi_j\rangle.
    \]
    Thus, the DCSBM can be represented as an RDPG, with community $k$ associated with the latent direction $\nu_k$. The SBM is recovered as the special case $\Theta\equiv1$.
\end{example}

To formulate our asymptotic results, we define a sequence of random graphs $\left\{A_n\right\}_{n\in\mathbb{N}}$, generated by the following process.

\begin{model}\label{Model::ModelGeneration}
    For each $n\in\mathbb{N}$, draw $\xi_1,\ldots,\xi_n\overset{\mathrm{iid}}{\sim}F$. Let $\rho_n\in(0,1]$ and define $X_i=\rho_n^{1/2}\xi_i$. Conditional on these latent positions, generate $A_n$ with independent upper-triangular entries satisfying
    \[
        a_{ij}\sim\operatorname{Bern}\left(\rho_n\langle\xi_i,\xi_j\rangle\right),\qquad i<j.
    \]
\end{model}

In Model~\ref{Model::ModelGeneration}, we distinguish the unscaled latent positions $\xi_i\sim F$ from the latent positions used to generate the graph, $X_i=\rho_n^{1/2}\xi_i$. The quantity $\rho_n$ drives the sparsity of the graph. We consider two asymptotic regimes for the sparsity factor. In the dense regime, $\rho_n\equiv 1$, and in the sparse regime, $\rho_n\to0$. These two cases are sufficient for our purposes, since if $\rho_n\to \rho_*\in(0,1]$, the limiting constant $\rho_*$ can be absorbed into the latent distribution, or for block models, into $B$.

Letting $X$ denote the matrix whose $i$th row is $X_i^\top=(\rho_n^{1/2}\xi_i)^\top$, we have established that $P=XX^\top$. In this case, the adjacency spectral embedding $\hat{X}_0$ serves as a natural estimate of $X$, up to an orthogonal transformation. This relationship extends to the degree-$\alpha$ Laplacian. We first define the population degree-$\alpha$ Laplacian to be $L_\alpha=T^{-\alpha}PT^{-\alpha}$, where $T$ is a diagonal matrix with $T_{ii}=\sum_{j\neq i}p_{ij}$, the expected degree of each node. We also write $X_\alpha=T^{-\alpha}X$, obtaining the population decomposition $L_\alpha=X_\alpha X_\alpha^\top=U_\alpha\Lambda_\alpha U_\alpha^\top$. Thus, since $\hat L_\alpha$ is the natural sample analogue of $L_\alpha$, the degree-$\alpha$ Laplacian spectral embedding $\hat X_\alpha$ serves as a natural estimate of $X_\alpha$, again up to an orthogonal transformation. We formalize this intuition with the following theorem, which establishes a row-wise central limit theorem for the degree-$\alpha$ Laplacian spectral embedding. 

We now introduce the technical conditions required for our results to hold, which we assume throughout the theoretical results. Let $\xi\sim F$ denote a generic unscaled latent position, and define $\mu=\E[\xi]$.

\begin{condition}
\label{Assump::Standing}
The following assumptions hold.
\begin{enumerate}
    \renewcommand{\labelenumi}{\textup{(\alph{enumi})}}
    \renewcommand{\theenumi}{\thecondition\alph{enumi}}
    \item\label{Assump::BoundedSupport}
    The latent position distribution $F$ has bounded support in $\mathbb{R}^r$, where $r$ is fixed. Additionally,
    \[
        x^\top y\in[0,1]
        \qquad\text{for all }x,y\in\operatorname{Supp}(F).
    \]

    \item\label{Assump::Sparsity}
    The sparsity factor satisfies either $\rho_n\equiv 1$ or $\rho_n\to0$. In either case,
    $\delta_n:=n\rho_n$ satisfies
    \[
        \frac{\delta_n}{(\log n)^2}\longrightarrow\infty.
    \]

    \item\label{Assump::Nonzero}
    There exists a constant $c>0$ such that
    \[
        \langle x,\mu\rangle\geq c
        \qquad\text{for all }x\in\operatorname{Supp}(F).
    \]
    \item\label{Assump::FullSpan}
    The support of $F$ spans $\mathbb{R}^r$.
\end{enumerate}
\end{condition}

\begin{remark}
    Condition~\ref{Assump::BoundedSupport} bounds the size of the latent positions and ensures that all dot products define valid edge probabilities. Condition~\ref{Assump::Sparsity} requires the expected-degree scale $\delta_n=n\rho_n$ to grow faster than $(\log n)^2$. Conditions of this form are standard in the spectral embedding literature. For comparison, \citet{Tang2018} require the corresponding expected-degree scale to grow faster than $(\log n)^4$ for the adjacency and symmetric Laplacian cases. Condition~\ref{Assump::Nonzero} ensures that the expected-degree factors are uniformly bounded away from zero, which is needed to keep quantities such as $\langle x,\mu\rangle^{-1}$ uniformly bounded. Condition~\ref{Assump::FullSpan} ensures that the latent positions occupy all $r$ dimensions. If the support of $F$ spans a lower-dimensional subspace, we may simply take that subspace to be the latent dimension. Thus, we take $R=r$ for the matrices in Definition~\ref{Def::DALSE}. For $\alpha\in[0,1]$, define
    \[
        \Upsilon_\alpha=\E\left[\frac{\xi\xi^\top}{\langle \xi,\mu\rangle^{2\alpha}}\right].
    \]
    Under Conditions~\ref{Assump::Nonzero} and~\ref{Assump::FullSpan}, $\Upsilon_\alpha$ is positive definite.

    Conditions~\ref{Assump::BoundedSupport} and~\ref{Assump::Sparsity} ensure that the observed degrees are uniformly concentrated around their expectations (see Appendix~\ref{Sec::CommonBounds}), while Condition~\ref{Assump::Nonzero} ensures the expected degrees are uniformly bounded away from zero. Thus, the observed degrees are uniformly positive with probability tending to one, and $\hat L_\alpha$ is well defined with probability tending to one.
\end{remark}

\begin{theorem}
    \label{Thm::Main}
    Fix $\alpha\in[0,1]$, and assume Condition~\ref{Assump::Standing} holds. Let $\{A_n\}_{n\geq1}$ be a sequence of RDPGs constructed according to Model~\ref{Model::ModelGeneration}, and let $\xi'$ be an independent copy of $\xi$. Then there exists a sequence of orthogonal matrices $Q_n$ such that, for each fixed index $i$, conditional on $\xi_i$,

    \[
        n^{\alpha+1/2}\rho_n^\alpha\left[\left(\hat X_\alpha Q_n\right)_{i*}-\left(X_\alpha\right)_{i*}\right]^\top\overset{d}{\longrightarrow}N\left(0,\Sigma_\alpha(\xi_i)\right),
    \]
    where
    \[
        \Sigma_\alpha(x)=\frac{\Upsilon_\alpha^{-1}\Gamma_{\rho,\alpha}(x)\Upsilon_\alpha^{-1}}{\langle x,\mu\rangle^{2\alpha}},
    \]
    with
    \[ 
    \Gamma_{\rho,\alpha}(x)= 
    \begin{cases} 
        \E_{\xi'}\left[ \langle x,\xi'\rangle \left(1-\langle x,\xi'\rangle\right) \left( \dfrac{\xi'}{\langle \xi',\mu\rangle^{2\alpha}} - \dfrac{\alpha \Upsilon_\alpha x}{\langle x,\mu\rangle} \right) \left( \dfrac{\xi'}{\langle \xi',\mu\rangle^{2\alpha}} - \dfrac{\alpha \Upsilon_\alpha x}{\langle x,\mu\rangle} \right)^\top \right], & \rho_n\equiv 1, \\[1.2em] 
        \E_{\xi'}\left[ \langle x,\xi'\rangle \left( \dfrac{\xi'}{\langle \xi',\mu\rangle^{2\alpha}}-\dfrac{\alpha \Upsilon_\alpha x}{\langle x,\mu\rangle} \right)\left( \dfrac{\xi'}{\langle \xi',\mu\rangle^{2\alpha}} - \dfrac{\alpha \Upsilon_\alpha x}{\langle x,\mu\rangle} \right)^\top \right], & \rho_n\to0. 
    \end{cases} 
    \]
\end{theorem}

\paragraph{Interpretation}
Theorem~\ref{Thm::Main} gives a row-wise Gaussian approximation for the degree-$\alpha$ spectral embedding. In general, the statements of conditional convergence should be understood as convergence in probability of the corresponding distributional laws. After orthogonal alignment, the transpose of the $i$th row is approximately distributed as
\[
  \left(\hat X_\alpha Q_n\right)_{i*}^\top\approx N\left((X_\alpha)_{i*}^\top,\frac{1}{n^{2\alpha+1}\rho_n^{2\alpha}}\Sigma_\alpha(\xi_i)\right).
\]
Thus, the choice of $\alpha$ affects the embedding in two ways. First, it changes the population embedding to $X_\alpha=T^{-\alpha}X$, reweighting latent positions according to their expected degrees. Second, it changes the covariance $\Sigma_\alpha(x)$, which describes the shape and orientation of the uncertainty around each point. Consequently, different values of $\alpha$ can produce different signal-to-noise tradeoffs, and need not be simple rescalings of the same embedding. This provides a basis for comparing choices of $\alpha$ in specific models, which we pursue in the following section.

\paragraph{Proof Sketch}
The proof of Theorem~\ref{Thm::Main} follows the approach of \citet{Tang2018}, with additional technical considerations when $\alpha>1/2$. We first write
\begin{equation}
    \label{Eq::SpecRemainder}
    \hat{X}_\alpha Q_n-X_\alpha=\left(\hat{L}_\alpha-L_\alpha\right)\T X(X^\top T^{-2\alpha}X)^{-1}+\mathcal{R}_{\mathrm{sp},\alpha}V_\alpha^\top,
\end{equation}
where $\mathcal{R}_{\mathrm{sp},\alpha}$ represents a spectral remainder term that is shown to be negligible after multiplication by $n^{\alpha+1/2}\rho_n^\alpha$. We then expand $\hat{L}_\alpha-L_\alpha$ using a Taylor expansion, and show that many of the resulting terms are also negligible after multiplication by $n^{\alpha+1/2}\rho_n^\alpha$. Collecting the remaining terms ultimately allows us to write the leading terms in the expansion of the $i$th row as 
\begin{equation}
    \label{Eq::FinalRowExpansion}
    t_i^{-\alpha}\left(\sum_{j\neq i}(a_{ij}-p_{ij})t_j^{-2\alpha}X_{j*}\right)\left(X^\top T^{-2\alpha}X\right)^{-1}+\alpha t_i^{-\alpha-1}(t_i-d_i)X_{i*}+\mathrm{remainder},
\end{equation}
where $t_i$ and $d_i$ denote the $i$th diagonal entry of $T$ and $D$, respectively. Conditional on the latent positions, the two leading terms can be combined into a sum of independent random vectors, and a central limit theorem may be applied to produce the final result. 

The primary difficulty when considering general $\alpha\in[0,1]$ is the presence of factors involving $\Lambda_\alpha^{-1/2}$. The nonzero eigenvalues of $L_\alpha$ are of order
$\delta_n^{1-2\alpha}$, and hence
\[
    \norm{\Lambda_\alpha^{-1/2}}=O(\delta_n^{\alpha-1/2}).
\]
For $\alpha\leq1/2$, this factor remains bounded, whereas for $\alpha>1/2$ it diverges. Consequently, several remainder terms that can be controlled by direct norm inequalities for $\alpha\leq1/2$ require more structure-sensitive arguments when $\alpha>1/2$.

\section{Spectral Embeddings Under Block Models}
\label{Sec::BlockModels}
Under the DCSBM, defined in Example~\ref{Ex::DCSBM}, the latent position distribution has a particularly simple structure. Let $\Theta$ be a random variable taking values in $(0,1]$, with support bounded away from zero, representing a generic degree-correction parameter. Let $Z\in\{1,\ldots,K\}$ be a community label independent of $\Theta$ satisfying $\mathbb{P}(Z=k)=\pi_k$. A generic latent position may then be written $\xi=\Theta\nu_Z$, where $\nu_1,\ldots,\nu_K\in\mathbb{R}^r$ satisfy $B_{k\ell}=\langle\nu_k,\nu_\ell\rangle$. Thus, nodes in the same community have latent positions lying along a common ray, with their distance from the origin scaled by the degree-correction parameter. 

In the remainder of the section, we formalize the two features previewed in Figure~\ref{Fig::DCSBM_Alpha_Comparison} from Section~\ref{Sec::Introduction}. We first summarize the model parameters used to generate the figure. A DCSBM is used with $n=2000$, $\pi_1=\pi_2=0.5$, and degree-correction parameters $\theta_i\overset{\mathrm{iid}}{\sim}\mathrm{Unif}\{1,2/3,1/3\}$, independently of the community labels. The block probability matrix is given by
\[
    B = \begin{bmatrix} 0.5 & 0.3 \\ 0.3 & 0.5 \end{bmatrix}.
\]
The rows are scaled by $n^{\alpha+1/2}$, and the asymptotic distributions are computed using Corollary~\ref{Cor::DCS} under the dense regime, $\rho_n\equiv1$. The plotted points are the embedded rows, and the ellipses show the corresponding $95\%$ probability ellipses.

\subsection{Population Geometry Under the DCSBM}
To specialize Theorem~\ref{Thm::Main} to the DCSBM, we first introduce some notation for the relevant quantities. Write
\[
    \bar\theta=\mathbb E[\Theta],\qquad\mu=\bar\theta\sum_{m=1}^K\pi_m\nu_m,
\]
and for $k\in\{1,\ldots,K\}$, define
\[
    \gamma_k=\langle\nu_k,\mu\rangle=\bar\theta\sum_{m=1}^K\pi_mB_{km}.
\]
Then $\langle\xi,\mu\rangle=\Theta\gamma_Z$, and the matrix $\Upsilon_\alpha$ in Theorem~\ref{Thm::Main} becomes
\[
    \tilde\Upsilon_\alpha=\mathbb E\!\left[\Theta^{2-2\alpha}\right]\sum_{\ell=1}^K\frac{\pi_\ell\nu_\ell\nu_\ell^\top}{\gamma_\ell^{2\alpha}}.
\]

With this notation, Theorem~\ref{Thm::Main} specializes to the DCSBM as follows.
\begin{corollary}
\label{Cor::DCS}
Assume the setting of Theorem~\ref{Thm::Main}, suppose that the random dot product graph follows the DCSBM described above, and let $(\Theta',Z')$ be an independent copy of $(\Theta,Z)$. For $\vartheta\in\operatorname{Supp}(\Theta)$ and $k\in\{1,\ldots,K\}$, define
\[
    m_\alpha(\vartheta,k)=\frac{\vartheta^{1-\alpha}\nu_k}{\gamma_k^\alpha}\qquad\text{and}\qquad\Delta_{\alpha,k}=\sqrt{\Theta'B_{kZ'}}\,\tilde\Upsilon_\alpha^{-1}\left(\frac{(\Theta')^{1-2\alpha}\nu_{Z'}}{\gamma_{Z'}^{2\alpha}}-\frac{\alpha\tilde\Upsilon_\alpha\nu_k}{\gamma_k}\right).
\]
Then, for each fixed index $i$, conditional on $z_i=k$ and $\theta_i=\vartheta$, there exists a sequence of orthogonal matrices $Q_n$ such that
\[
    n^{\alpha+1/2}\rho_n^\alpha\left[\left(\hat{X}_\alpha Q_n\right)_{i*}-\left(X_\alpha\right)_{i*}\right]^\top\overset{d}{\longrightarrow}N\!\left(0,\Sigma_\alpha(\vartheta,k)\right),
\]
where
\[
    \Sigma_\alpha(\vartheta,k)=\frac{\vartheta^{1-2\alpha}}{\gamma_k^{2\alpha}}
    \begin{cases}
        \displaystyle\E_{\Theta',Z'}\!\left[\left(1-\vartheta\Theta' B_{kZ'}\right)\Delta_{\alpha,k}\Delta_{\alpha,k}^\top\right],
        & \rho_n\equiv1,\\[1.2em]
        \displaystyle\E_{\Theta',Z'}\!\left[\Delta_{\alpha,k}\Delta_{\alpha,k}^\top\right],
        & \rho_n\to0.
    \end{cases}
\]
Moreover,
\[
    n^\alpha\rho_n^{\alpha-1/2}(X_\alpha)_{i*}^\top\overset{p}{\longrightarrow}m_\alpha(\vartheta,k).
\]
\end{corollary}

\begin{proof}
    The stated covariance formula and central limit theorem follow by substituting the relevant DCSBM model parameters into Theorem~\ref{Thm::Main}. It remains to derive the limiting rescaled population center. Fix a node $i$ with $z_i=k$ and $\theta_i=\vartheta$, and define
    \[
        \gamma_{k,n}^{(-i)}=\frac{1}{n}\sum_{j\neq i}\theta_j B_{kz_j}.
    \]
    Then
    \begin{align}
        n^\alpha\rho_n^{\alpha-1/2}(X_\alpha)_{i*}^\top&=n^\alpha\rho_n^{\alpha-1/2}t_i^{-\alpha}X_i \nonumber\\
        &=n^\alpha\rho_n^{\alpha}\left(\sum_{j\neq i}p_{ij}\right)^{-\alpha}\vartheta\nu_k \nonumber\\
        &=n^\alpha\rho_n^{\alpha}\left(n\rho_n\vartheta\,\gamma_{k,n}^{(-i)}\right)^{-\alpha}\vartheta\nu_k \nonumber\\
        &=\frac{\vartheta^{1-\alpha}\nu_k}{\left(\gamma_{k,n}^{(-i)}\right)^\alpha}\label{Eq::ScaledMeanFinite}.
    \end{align}
    By the law of large numbers,
    \[
        \gamma_{k,n}^{(-i)}\overset{p}{\longrightarrow}\gamma_k=\bar{\theta}\sum_{\ell=1}^K\pi_\ell B_{k\ell}.
    \]
    Plugging this into Equation~\eqref{Eq::ScaledMeanFinite} gives the desired result,
    \[
        m_\alpha(\vartheta,k)=\frac{\vartheta^{1-\alpha}\nu_k}{\gamma_k^\alpha}.
    \]
\end{proof}

The expression for $m_\alpha(\vartheta,k)$ makes explicit how degree normalization changes the population geometry. At the population level, for $\alpha<1$, the center of a node with degree-correction parameter $\vartheta$ is proportional to $\vartheta^{1-\alpha}$. Thus, within each community, the population centers lie along the ray spanned by $\nu_k$, with $\vartheta$ determining the position along the ray. As $\alpha$ increases, this effect is reduced, and at $\alpha=1$ it disappears entirely. Consequently, all nodes in community $k$ share the same limiting population center. 

This collapse to a common limiting population center is already implicit from the factorization $X_\alpha=T^{-\alpha}X$ and can be established without a row-wise central limit theorem. However, Corollary~\ref{Cor::DCS} additionally provides the corresponding second-order description in the form of the limiting covariance $\Sigma_\alpha(\vartheta,k)$. To explain the elongating effect observed in Figure~\ref{Fig::DCSBM_Alpha_Comparison}, it is useful to view the population centers in two ways. For a fixed community, varying $\vartheta$ traces the ray described above. For a fixed $\vartheta$, the $K$ community centers form a cross-section of these rays. In Figure~\ref{Fig::DCSBM_Alpha_Comparison}, this is the vertical line joining the two centers with the same shade. The following corollary shows that, more generally, these fixed $\vartheta$ centers lie in an explicit affine hyperplane. 

\begin{corollary}
\label{Cor::DCSPopulationHyperplane}
Assume the setting of Corollary~\ref{Cor::DCS} and suppose that $B$ has rank $K$. For each $\alpha\in[0,1]$, let $h_\alpha$ be the unique vector satisfying
\[
    h_\alpha^\top\nu_k=\gamma_k^\alpha,
    \qquad k=1,\ldots,K.
\]
Then, for every $\vartheta\in\operatorname{Supp}(\Theta)$, the population centers $m_\alpha(\vartheta,1),\ldots,m_\alpha(\vartheta,K)$ lie in the
affine hyperplane
\[
    H_{\alpha,\vartheta}=\left\{x\in\mathbb R^K:h_\alpha^\top x=\vartheta^{1-\alpha}\right\}.
\]
\end{corollary}

\begin{proof}
Let $V\in\mathbb{R}^{K\times r}$ have rows $\nu_1^\top,\ldots,\nu_K^\top$, so that $B=VV^\top$. Since $\operatorname{rank}(B)=K$, we have $\operatorname{rank}(V)=K$. On the other hand, Condition~\ref{Assump::FullSpan} implies that $\nu_1,\ldots,\nu_K$ span $\mathbb{R}^r$, and hence $\operatorname{rank}(V)=r$. Thus, $r=K$ and $\nu_1,\ldots,\nu_K$ form a basis of $\mathbb R^K$. Consequently, $h_\alpha$ exists and is unique. Plugging in the definitions of $h_\alpha$ and $m_\alpha(\vartheta,k)$ immediately gives
\[
    h_\alpha^\top m_\alpha(\vartheta,k)=\vartheta^{1-\alpha}\frac{h_\alpha^\top\nu_k}{\gamma_k^\alpha}=\vartheta^{1-\alpha},
\]
which proves that $m_\alpha(\vartheta,k)\in H_{\alpha,\vartheta}$ for every $k=1,\ldots,K$.
\end{proof}

Corollary~\ref{Cor::DCSPopulationHyperplane} is useful for two reasons. First, when $\alpha=1$, we have $h_1=\mu$. Thus, all the limiting population centers lie in the hyperplane
\[
    H_{1,\vartheta}=\left\{x\in\mathbb{R}^K:\mu^\top x=1\right\},
\]
and $H_{1,\vartheta}$ is independent of $\vartheta$. Second, these hyperplanes provide a natural reference for describing the orientation of the limiting covariance. We therefore decompose the covariance into components normal and tangential to $H_{\alpha,\vartheta}$. The resulting ratio quantifies the observed elongation of probability ellipses in Figure~\ref{Fig::DCSBM_Alpha_Comparison}. Let
\[
    u_\alpha=\frac{h_\alpha}{\|h_\alpha\|},\qquad\Pi_\alpha=I_K-u_\alpha u_\alpha^\top,
\]
so that $u_\alpha$ is normal to $H_{\alpha,\vartheta}$ and $\Pi_\alpha$ projects onto its tangent space. We may then define the normal variance and average tangential variance relative to this hyperplane by
\[
    \sigma^2_{N,\alpha,k}(\vartheta)=u_\alpha^\top\Sigma_\alpha(\vartheta,k)u_\alpha\qquad\text{and}\qquad\sigma^2_{T,\alpha,k}(\vartheta)=\frac{1}{K-1}\operatorname{tr}\!\left(\Pi_\alpha\Sigma_\alpha(\vartheta,k)\Pi_\alpha\right).
\]

Then, whenever $\sigma^2_{T,\alpha,k}(\vartheta)>0$, we write
\[
    \mathcal F_{\alpha,k}(\vartheta):=\frac{\sigma^2_{N,\alpha,k}(\vartheta)}{\sigma^2_{T,\alpha,k}(\vartheta)}.
\]
When $\rho_n\to0$, the only dependence of $\Sigma_\alpha(\vartheta,k)$ on $\vartheta$ is through the scalar factor $\vartheta^{1-2\alpha}$. Consequently, the ratio $\mathcal F_{\alpha,k}(\vartheta)$ does not depend on $\vartheta$, and we may write it as $\mathcal{F}_{\alpha,k}$. Smaller values of $\mathcal{F}_{\alpha,k}$ indicate that the limiting covariance is more strongly concentrated in directions tangent to $H_{\alpha,\vartheta}$. 

\subsection{A Symmetric Two-Community Example}
We illustrate this ratio using a symmetric two-community DCSBM in the sparse regime, $\rho_n\to0$. An equivalent analysis could be carried out in the dense regime, though the computations are somewhat more involved. Consider the model
\[
    \pi_1=\pi_2=\frac12,\qquad B=\begin{bmatrix}p&q\\q&p\end{bmatrix},\qquad p>q>0.
\]
A direct computation (Appendix~\ref{Sec::FlatteningRatioDerivation}) gives
\[
    \mathcal F_\alpha=\frac{\kappa\left(1-\alpha(2-\alpha)r_\alpha\right)}{1-\kappa^2\alpha(2-\alpha)r_\alpha},
\]
where 
\[
    \kappa=\frac{p-q}{p+q}\qquad\text{and}\qquad r_\alpha=\dfrac{\left(\E\left[\Theta^{2-2\alpha}\right]\right)^2}{\bar\theta\,\E\!\left[\Theta^{3-4\alpha}\right]}.
\]

If $\Theta$ is constant, i.e., if there is no degree heterogeneity and the model reduces to an SBM, $r_\alpha=1$ and the expression simplifies to 
\[
    \mathcal F_\alpha=\frac{\kappa(1-\alpha)^2}{1-\left(\kappa\right)^2\alpha(2-\alpha)}.
\]
This quantity is strictly decreasing in $\alpha$ and is exactly $0$ when $\alpha=1$. Qualitatively, in this setting, the covariance matrices become progressively flatter and are singular at $\alpha=1$, with all variation lying parallel to the hyperplane containing the limiting population centers. 

When $\Theta$ is not constant, an application of Cauchy--Schwarz gives
\[
    0\leq r_\alpha\leq1.
\] 
Holding $r_\alpha$ fixed, the expression for $\mathcal{F}_\alpha$ is decreasing in $\alpha$. Although $r_\alpha$ does affect how $\mathcal{F}_\alpha$ changes with $\alpha$, this nevertheless identifies one mechanism contributing to the general flattening seen in the example of Figure~\ref{Fig::DCSBM_Alpha_Comparison}.

While $\mathcal{F}_{\alpha,k}$ summarizes the shape of the within-community covariance, it does not fully capture the overall separation between communities. In the next section, we thus compare the resulting community distributions directly.

\section{Comparing Degree Normalizations for Community Detection}
\label{Sec::ComparingDegreeNorm}
To study how different degree normalizations affect community detection, we consider a row-normalized spectral clustering algorithm \citep{Jin2015, QinRohe2013}. Under the DCSBM, the row-normalization step is included to remove the effect of the degree-correction parameter at the population level. For each choice of $\alpha$, the degree-$\alpha$ spectral embedding is computed, its rows are normalized, and $k$-means is applied to the resulting points. The algorithm is summarized in Algorithm~\ref{Alg::SpectralClustering}. In this section, we continue under the assumption that $B$ has rank $K$. 

\begin{algorithm}[htpb]
\caption{Row-Normalized Degree-$\alpha$ Spectral Clustering}
\label{Alg::SpectralClustering}
\begin{algorithmic}[1]
    \Require Degree-$\alpha$ Laplacian $\hat{L}_\alpha$; number of communities $K$
    \Ensure Estimated community labels $\hat z_1,\ldots,\hat z_n$
    \State Compute the degree-$\alpha$ spectral embedding $\hat X_{\alpha,K}\in\mathbb{R}^{n\times K}$ from $\hat L_\alpha$.
    \State Normalize the rows of $\hat{X}_{\alpha,K}$ and let $\tilde{X}_{\alpha,K}$ denote the normalized embedding.
    \State Apply $k$-means to the rows of $\tilde X_{\alpha,K}$.
    \State Return the resulting cluster labels $\hat z_1,\ldots,\hat z_n$.
\end{algorithmic}
\end{algorithm}

It is natural to ask how the choice of $\alpha$ affects spectral clustering performance. In the following section, we study this question by comparing the projected Gaussian distributions predicted by asymptotic theory. 

\subsection{Comparing the Projected Distributions}
We now restrict our attention to two-community SBMs and describe a diagnostic for comparing different values of $\alpha$ in spectral clustering. Following the idealization used by \citet{Cape2019}, we adopt two approximations. First, we treat the asymptotic Gaussian approximation from Corollary~\ref{Cor::DCS} as the exact finite-sample distribution of an embedded row. Second, we condition on a fixed community assignment, with deterministic community sizes chosen to match the specified community proportions. Together, these assumptions define the surrogate distributions compared below. Consequently, the value of $\alpha$ preferred by the diagnostic need not coincide with the value that minimizes the actual finite-sample error of spectral clustering. Finally, because this section concerns inference from a single observed graph rather than a sequence of graphs, we set $\rho_n\equiv1$ and interpret sparsity through the magnitudes of the entries of $B$.

We now compare values of $\alpha$ by applying the row-normalization step from Algorithm~\ref{Alg::SpectralClustering} to these Gaussian approximations and evaluating the Bayes error of the resulting projected distributions. Concretely, for each $\alpha\in[0,1]$ and $k\in\{1,2\}$, let $m_{\alpha,k}$ and $\Sigma_{\alpha,k}$ denote the limiting population center and covariance matrix of a node in community $k$, respectively. Let $\mathcal{Q}_{k,\alpha}^{(n)}$ denote the distribution obtained by projecting a draw from
\[
    N\left(m_{\alpha,k},\frac{1}{n}\Sigma_{\alpha,k}\right)
\]
onto the unit circle using $\psi(x)=x/\norm{x}$. Strictly speaking, the distribution here is a scaled version of the one derived in Corollary~\ref{Cor::DCS}. Because this distribution is projected to get $\mathcal{Q}_{k,\alpha}^{(n)}$, this scaling is irrelevant. However, this scale makes clear how the distribution changes as $n$ increases. To quantify the overlap between the distributions, we consider the Bayes error
\[
    \mathcal E_\alpha^{(n)}=\inf_{\phi:\mathbb S^1\to\{1,2\}}\sum_{k=1}^2\pi_k\mathcal{Q}_{k,\alpha}^{(n)}\left(\left\{u\in\mathbb S^1:\phi(u)\neq k\right\}\right),
\]
where $\pi_k$ denotes the probability of a node belonging to community $k$. Smaller values of $\mathcal E_\alpha^{(n)}$ indicate better distinguishability between the communities. We are therefore interested in
\begin{equation}
    \label{Eq::AlphaStar}
    \alpha_n^*\in\argmin_{\alpha\in[0,1]}\mathcal E_\alpha^{(n)}.
\end{equation}
We refer to $\alpha_n^*$ as the ``preferred'' degree normalization for the specified model and value of $n$, where the comparison is understood to be with respect to the Bayes-error criterion rather than actual finite-sample clustering error. Although $\mathcal{E}_\alpha^{(n)}\to0$ for every fixed $\alpha$, the relative ordering across $\alpha$ may stabilize as $n$ increases. We next describe how we compute $\alpha_n^*$ in practice.

\subsection{Numerical Evaluation}
\label{Sec::NumEval}
In the next two sections, we consider a family of models corresponding to a range of parameter settings. For each setting, we use the explicit expressions for the means and covariance matrices from Corollary~\ref{Cor::DCS} to numerically evaluate Equation~\eqref{Eq::AlphaStar}. We approximate $\alpha_n^*$ by the grid point that minimizes $\mathcal{E}_\alpha^{(n)}$. For the remaining computations, we set $n=200$ and suppress the dependence on $n$, writing the resulting minimizer as $\alpha^*$. Further computational details, including the stability of $\alpha^*$ with respect to $n$, can be found in Appendix~\ref{Sec::AdditionalNumericalDetails}.

This diagnostic is closely related to the Chernoff-information comparison used by \citet{Cape2019} for adjacency and symmetric Laplacian embeddings. Chernoff information is a natural criterion because it characterizes the asymptotic exponential decay rate of the Bayes error and has an explicit form for multivariate Gaussian distributions. We target the Bayes error directly because we apply the row-normalization step used in Algorithm~\ref{Alg::SpectralClustering} before evaluating community separation. The resulting projected Gaussian distributions are no longer Gaussian, so the corresponding Chernoff information does not retain the same convenient closed form. Further comparisons with the Chernoff-information benchmark are provided in Appendix~\ref{Sec::PropertiesofDiagnostic}.

For selected parameter settings, we also simulate networks and compare the predicted $\alpha^*$ with the finite-sample misclassification rates of Algorithm~\ref{Alg::SpectralClustering}. Although the error of the $k$-means decision rule need not approach the Bayes error, these comparisons are intended to assess whether the broad ordering across values of $\alpha$ predicted by the projected Gaussian distributions is reflected in the performance of the actual clustering procedure typically used in practice.

\subsection{Imbalanced Communities}
\label{Sec::ImbalancedCommunities}
We begin by considering an SBM on $n$ nodes with
\[
    B=\begin{bmatrix}p&q\\q&p\end{bmatrix},\qquad p>q.
\]
Let $\pi$ denote the proportion of nodes in community 1, so that $\pi_1=\pi$ and $\pi_2=1-\pi$. Figure~\ref{Fig::pqplots} illustrates how $\alpha^*$ varies with $p$, $q$, and $\pi$. Each heat map shows a clear dark region in which adjacency spectral clustering, corresponding to $\alpha=0$, is preferred. This region generally occurs when $p$ and $q$ are both relatively large, with $\alpha^*$ increasing as these values decrease.

\begin{figure}[t]
    \centering
    \begin{subfigure}[t]{0.27\textwidth}
        \centering
        \includegraphics[width=\linewidth]{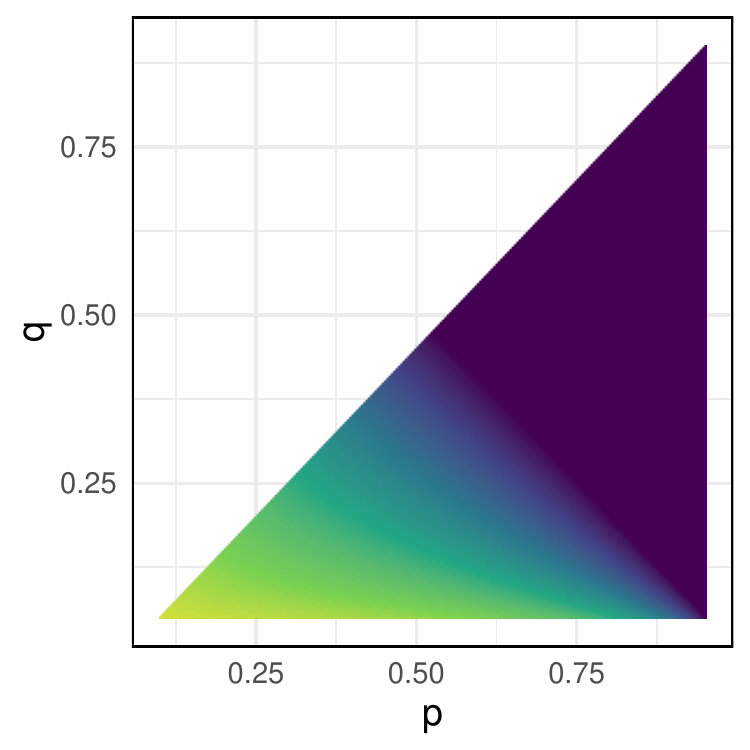}
        \caption{$\pi=0.5$}
        \label{Fig::05pqplot}
    \end{subfigure}
    \hfill
    \begin{subfigure}[t]{0.27\textwidth}
        \centering
        \includegraphics[width=\linewidth]{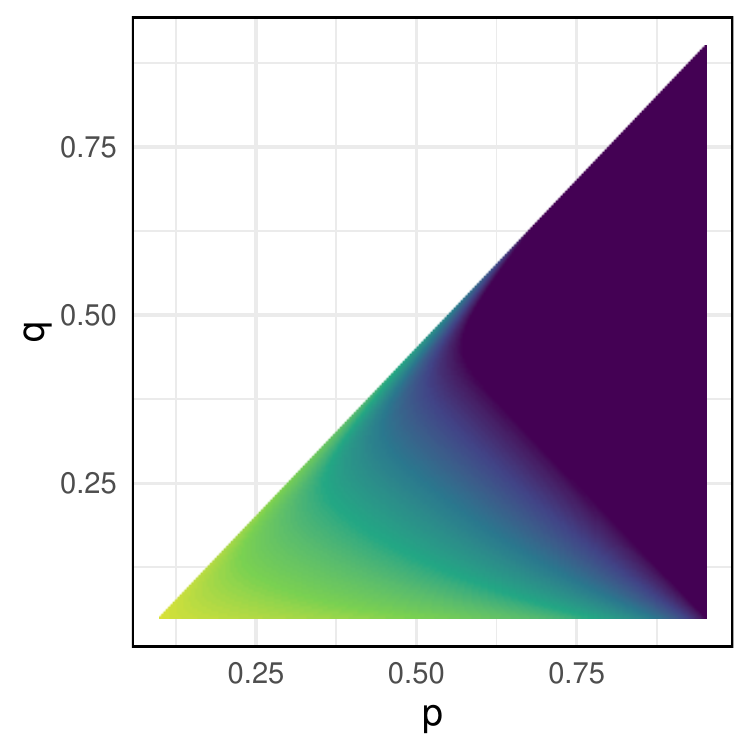}
        \caption{$\pi=0.7$}
        \label{Fig::07pqplot}
    \end{subfigure}
    \hfill
    \begin{subfigure}[t]{0.27\textwidth}
        \centering
        \includegraphics[width=\linewidth]{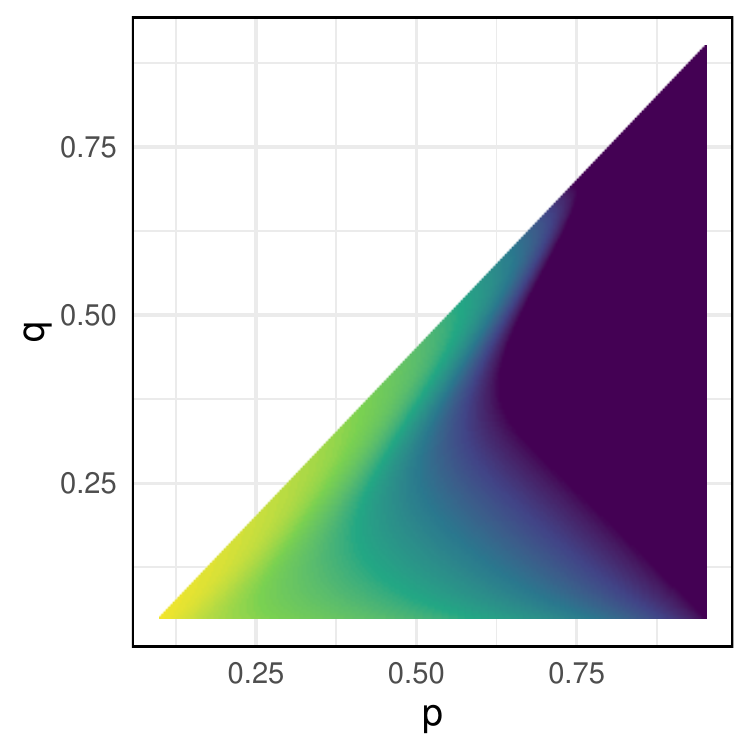}
        \caption{$\pi=0.9$}
        \label{Fig::09pqplot}
    \end{subfigure}
    \hfill
    \begin{minipage}[t]{0.07\textwidth}
        \centering
        \raisebox{1.4em}[0pt][0pt]{%
            \includegraphics[width=\linewidth]{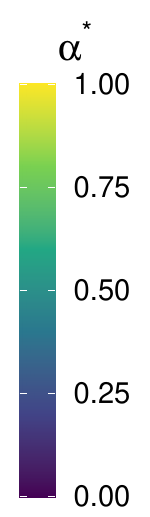}%
        }
    \end{minipage}
    \vspace{1em}
    \begin{subfigure}[t]{0.55\textwidth}
        \centering
        \includegraphics[width=\linewidth]{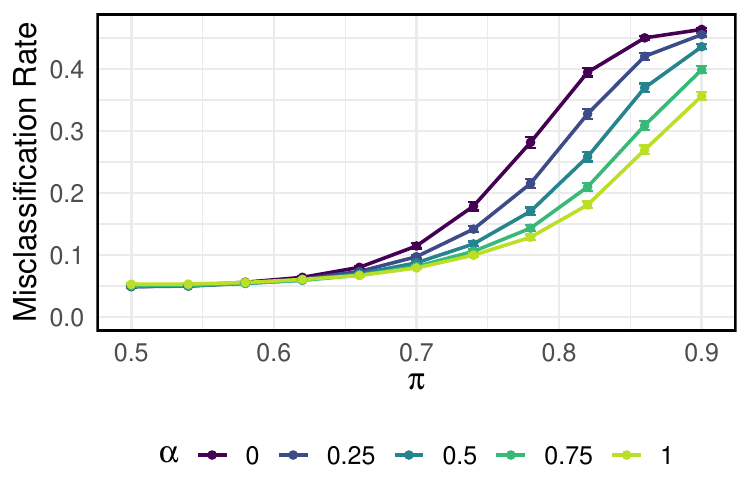}
        \caption{Numerical Results}
        \label{Fig::experimentvarypi}
    \end{subfigure}
     \caption{The top row shows $\alpha^*$ as a function of $p$, $q$, and $\pi$, while the bottom panel reports the average misclassification rates over $500$ runs of spectral clustering for varying $\alpha$. Error bars indicate $\pm2$ standard errors across the $500$ replications. Both results suggest that stronger normalization is favored as the communities become more imbalanced.}
     \label{Fig::pqplots}
\end{figure}

We now evaluate whether this behavior is reflected in finite-sample performance of Algorithm~\ref{Alg::SpectralClustering}. To further isolate the effect of community imbalance, we fix $p$ and $q$ and vary only the community proportions. We use 
\[
    B=\begin{bmatrix}0.30&0.21\\0.21&0.30\end{bmatrix}
\]
with $n=350$, $\pi=0.5+0.04c$, and $c\in\{0,\ldots,10\}$. The results of this experiment are shown in Figure~\ref{Fig::experimentvarypi}, which plots the overall misclassification rate as a function of $\pi$. Lighter shades correspond to stronger degree normalization. This finite-sample experiment shows that when the communities are nearly balanced, the choice of $\alpha$ has little effect. However, as the imbalance increases, smaller values of $\alpha$ perform progressively worse. Overall, both the analysis of $\alpha^*$ and numerical evidence support the idea that stronger normalization can mitigate the deterioration of spectral clustering performance caused by community imbalance. 

\subsection{Unequal Within-Community Probabilities}
\label{Sec::Nonhomogeneous}
In this section, we consider an SBM on $n$ nodes with 
\[
    \pi_1=\pi_2=\frac{1}{2},\qquad B=\begin{bmatrix}p&q\\q & r\end{bmatrix},\qquad pr-q^2>0.
\]
Here, $r$ denotes a block-probability parameter rather than the latent position dimension introduced earlier. Since the embedding dimension is fixed at two throughout this section, the distinction should be clear from context. The condition $pr-q^2>0$ ensures that $B$ is positive definite. When $p\neq r$, the two communities have different expected-degree scales, so degree normalization may affect the two blocks differently. 

The left panel of Figure~\ref{Fig::pqplotstheory} shows that the preferred amount of normalization depends strongly on $B$. Larger values of $\alpha$ are favored throughout much of the low-$q$ region, whereas little or no normalization is preferred when both $q$ and $r$ are relatively large. Thus, even with balanced communities, no single value of $\alpha$ is uniformly preferable.

We next examine finite-sample performance in the setting
\[
    B=\begin{bmatrix}0.25&0.10\\0.10&0.10\end{bmatrix}
\]
with $n=200$ and $\alpha\in\{0,0.1,\ldots,1\}$. For each value of $\alpha$, we apply Algorithm~\ref{Alg::SpectralClustering} and average the overall misclassification rate over $500$ replications. As shown in Figure~\ref{Fig::fixedpmeans}, the average misclassification rate is minimized near $\alpha=0.4$, with $\alpha=0.5$ giving comparable performance. This minimizer does not coincide with the predicted $\alpha^*\approx0.8$, though some discrepancy is expected due to the assumptions used in computing $\alpha^*$. Nevertheless, both suggest that an intermediate amount of normalization is preferred, and possible sources of this discrepancy are discussed in Section~\ref{Sec::Discussion}.

\begin{figure}[t]
     \centering
     \begin{subfigure}[b]{0.48\textwidth}
         \centering
         \includegraphics[width=\textwidth]{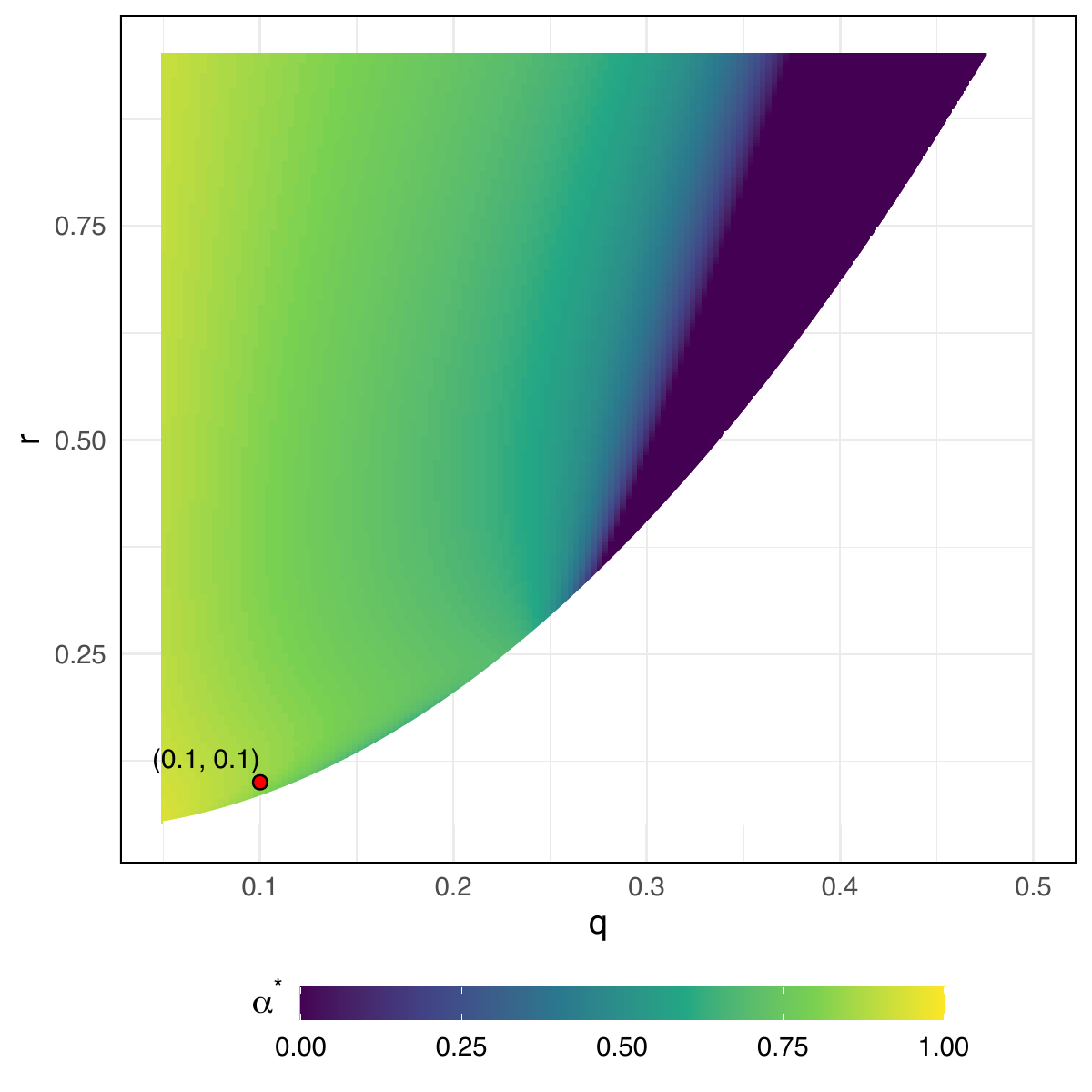}
         \caption{Fixed $p=0.25$}
         \label{Fig::025fixplot}
     \end{subfigure}
     \hfill 
     \begin{subfigure}[b]{0.48\textwidth}
         \centering
         \includegraphics[width=\textwidth]{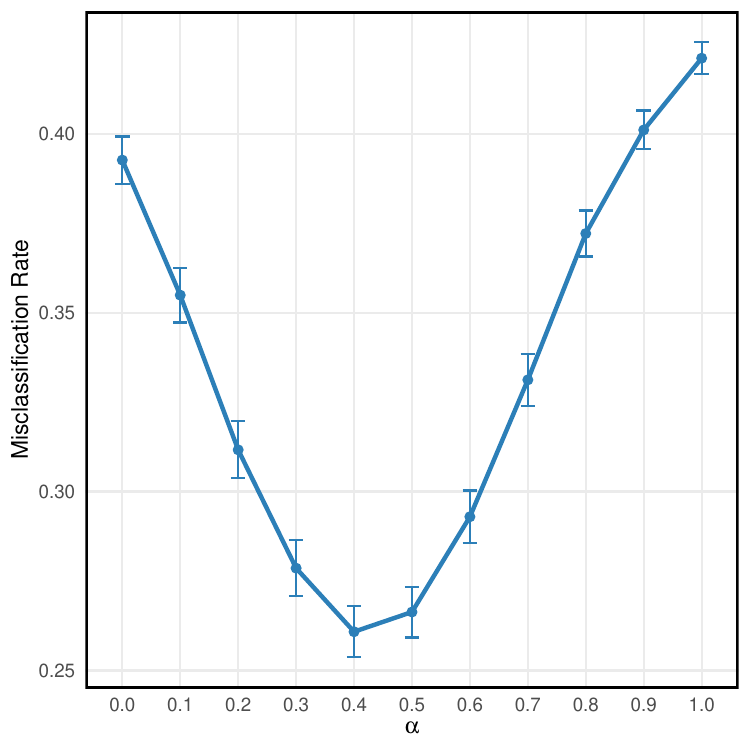}
         \caption{Numerical Results}
         \label{Fig::fixedpmeans}
     \end{subfigure}
     
     \caption{The left panel shows $\alpha^*$ as a function of $q$ and $r$ with $p$ fixed at $0.25$. The right panel reports the average misclassification rates over $500$ runs of spectral clustering for varying $\alpha$ with $p=0.25$, $q=0.1$, and $r=0.1$, corresponding to the labeled point in the left panel. Error bars indicate $\pm2$ standard errors across the $500$ replications. Both results suggest that an intermediate value of $\alpha$ is appropriate for these parameters, though the values do not agree exactly.}
     \label{Fig::pqplotstheory}
\end{figure}

\section{Discussion}
\label{Sec::Discussion}
We have studied a continuum of degree-normalized spectral embeddings that contains the adjacency and symmetric Laplacian embeddings as special cases. Under the random dot product graph model, we derived a row-wise central limit theorem describing how the choice of $\alpha$ changes both the population geometry and the local uncertainty of the embedded nodes. Specializing these expressions to the DCSBM, we showed that increasing $\alpha$ progressively removes the effect of degree heterogeneity from the population embedding. In particular, when $\alpha=1$, the population centers depend only on community membership and, when $B$ is full rank, lie in a common affine hyperplane. The same expressions describe the accompanying change in the shape and orientation of the limiting covariance, which can flatten toward that hyperplane as $\alpha$ increases.

These results also provide a basis for comparing degree normalizations in community detection. For the single-graph comparisons in Section~\ref{Sec::ComparingDegreeNorm}, we take $\rho_n\equiv1$ and interpret sparsity through the magnitudes of the entries of $B$. In the balanced two-community SBM, shown in Figure~\ref{Fig::05pqplot}, stronger normalization is preferred when the entries of $B$ are relatively small, whereas little or no normalization is preferred when they are large. Community imbalance (Section~\ref{Sec::ImbalancedCommunities}) and unequal within-community probabilities (Section~\ref{Sec::Nonhomogeneous}) alter the preferred normalization, so that it depends on community proportions and block-probability structure as well as on overall density. The finite-sample experiments indicate that the broad normalization regime suggested by $\alpha^*$ is also reflected in the performance of row-normalized spectral clustering. 

We now discuss several limitations and possible extensions of this work. Our analysis focused on the RDPG, which requires $B$ to be positive semidefinite, so many disassortative and other indefinite structures lie outside the present model. The generalized RDPG \citep{rubin2022statistical} would be the natural setting for such an extension.

The numerical comparison in Section~\ref{Sec::Nonhomogeneous} also illustrates how $\alpha^*$ should be interpreted. The quantity is computed under two idealizations. The first is that the row-wise Gaussian approximation of Theorem~\ref{Thm::Main} is exact at the sample size considered. The second is that clustering is carried out by the Bayes rule for the projected distributions rather than by $k$-means. Neither holds in finite samples, and $\alpha^*$ should therefore not be expected to coincide with the empirical minimizer. One possible source of the finite-sample discrepancy is the contribution of higher-order terms arising from the use of observed rather than expected degrees. Such higher-order terms may become more visible as $\alpha$ increases, particularly when the first-order covariance is nearly singular.

A related practical extension is degree regularization. Replacing $D$ by $D_\tau=D+\tau I$ for some $\tau>0$ (typically the average observed degree) is a standard approach in spectral clustering \citep{Amini2013,QinRohe2013}, as it limits the influence of nodes with very small observed degrees. We have deliberately worked throughout with the unregularized family $D^{-\alpha}AD^{-\alpha}$ in order to isolate the effect of degree normalization without introducing a second tuning parameter. Extending this family to
\[
    D_\tau^{-\alpha} A D_\tau^{-\alpha}
\]
and studying the joint roles of $\alpha$ and $\tau$ is a natural direction for future work. However, the population target associated with $D_\tau$ involves $T+\tau I$, which does not cancel the degree-correction parameter exactly, so the collapse of the population centers at $\alpha=1$ described in Corollary~\ref{Cor::DCSPopulationHyperplane} is recovered only when $\tau/\delta_n\to0$. Regularization would therefore trade part of the population geometry that motivates large $\alpha$ against the finite-sample stability described above. Characterizing this tradeoff, together with the corresponding projected Gaussian diagnostic, warrants further investigation.

Taken together, our results suggest that the amount of degree normalization affects both first-order geometry and second-order uncertainty of the corresponding embeddings. It should therefore be viewed as a model-dependent design choice rather than as a universal default.

% Bibliography using author-year style
\bibliographystyle{plainnat} % Author-year style
\bibliography{main} % Your BibTeX file

% Appendix section
\appendix
\section{Proof of Theorem~\ref{Thm::Main}}
The proof of Theorem~\ref{Thm::Main} is technical and notation-heavy, but, as outlined in Section~\ref{Sec::RDPGs}, the overall structure is relatively simple. This section of the appendix is organized as follows. In Section~\ref{Sec::CommonNotation}, we first list commonly used notation throughout the proof. Section~\ref{Sec::HelperLemmas} then proves several technical lemmas that will be used frequently throughout the proof of the theorem. Section~\ref{Sec::SpectralBounds} derives Equation~\eqref{Eq::SpecRemainder} and shows that the spectral remainder is negligible after scaling by $n^{\alpha+1/2}\rho_n^\alpha$. Finally, Section~\ref{Sec::TheoremProof} proves Theorem~\ref{Thm::Main}.

\subsection{Notation and Preliminaries for Proof}
\label{Sec::CommonNotation}
We first introduce some notation used throughout the appendix. For a matrix $M$, let $M_{i*}$ denote its $i$th row. We write $\norm{M}$ for the spectral norm of a matrix, $\norm{M}_F$ for the Frobenius norm, $\norm{M}_\infty:=\max_i\sum_j\left\lvert M_{ij}\right\rvert$ for the induced matrix infinity norm, and $\norm{M}_{2\to\infty}:=\max_i\norm{M_{i*}}_2$ for the two-to-infinity norm of a matrix \citep{Cape20192}. 

Throughout the proof, we fix $\alpha$. Additionally, we now make explicit the matrix $Q_n$ that appears in the statement of Theorem~\ref{Thm::Main}. Let
\[
    U_\alpha^\top\hat U_\alpha=W_{1,n}S_n W_{2,n}^\top
\]
be a singular value decomposition, and define $W_n:=W_{2,n}W_{1,n}^\top$. We choose $V_\alpha$ so that $U_\alpha=\T X V_\alpha\Lambda_\alpha^{-1/2}$. It follows that $X_\alpha=U_\alpha\Lambda_\alpha^{1/2}V_\alpha^\top$. We then define $Q_n:=W_nV_\alpha^\top$. Because $\alpha$ is fixed, we suppress the dependence of these matrices on $\alpha$ in the notation. The table below contains additional notation and descriptions.

\begin{table}[ht]
\setlength{\tabcolsep}{5pt}
\renewcommand{\arraystretch}{1.08}

\begin{tabularx}{\textwidth}{@{}
    >{\raggedright\arraybackslash}p{0.15\textwidth}
    >{\raggedright\arraybackslash}p{0.34\textwidth}
    >{\raggedright\arraybackslash}X
@{}}
\toprule
\textbf{Quantity} & \textbf{Definition} & \textbf{Description} \\
\midrule
$P_{\diag},P^\circ$&$P_{\diag}=\diag(P)$,\quad$P^\circ=P-P_{\diag}$&Diagonal correction and loop-free population matrix \\[0.5em]
$E_A$&$E_A=A-P^\circ$&Centered adjacency noise\\[0.5em]$d,D$&$d=A\mathbf 1$,\quad $D=\diag(d)$&Observed degrees\\[0.5em]
$t,T$&$t=P^\circ\mathbf 1$,\quad $T=\diag(t)$&Expected degrees\\[0.5em]
$\delta_n$&$\delta_n=n\rho_n$&Expected-degree scale\\[0.5em]
$E_\alpha$&$E_\alpha=\hat L_\alpha-L_\alpha$&Perturbation of the degree-$\alpha$ Laplacian\\[0.5em]
$G_{n,\alpha}$&$G_{n,\alpha}=X^\top T^{-2\alpha}X$&Gram matrix associated with $L_\alpha$\\[0.5em]$U_\alpha,V_\alpha,\Lambda_\alpha$&\(\begin{array}[t]{@{}l@{}}
L_\alpha=U_\alpha\Lambda_\alpha U_\alpha^\top,\\[2pt]
G_{n,\alpha}=V_\alpha\Lambda_\alpha V_\alpha^\top
\end{array}\)&Population eigendecompositions \\\addlinespace[0.35em]
$W_n$&\(\begin{array}[t]{@{}l@{}}U_\alpha^\top\hat U_\alpha=W_{1,n}S_nW_{2,n}^\top,\\[2pt]W_n=W_{2,n}W_{1,n}^\top\end{array}\)&Procrustes alignment of the sample and population eigenspaces.\\\addlinespace[0.35em]
$Q_n$&$Q_n=W_nV_\alpha^\top$&Alignment of $\hat{X}_\alpha$ and $X_\alpha$.\\\addlinespace[0.35em]$\Pi^\perp(U_\alpha)$&$\Pi^\perp(U_\alpha)=I_n-U_\alpha U_\alpha^\top$&Projection onto the orthogonal complement of the population
eigenspace. \\[0.5em]
$B_D,R_D$&\(\begin{array}[t]{@{}l@{}}B_D=\alpha T^{-\alpha-1}(T-D),\\[2pt]R_D=D^{-\alpha}-T^{-\alpha}-B_D\end{array}\)&First-order and remainder terms in the expansion of \(D^{-\alpha}\).\\
\bottomrule
\end{tabularx}

\caption{Notation used repeatedly in the proof.}
\label{Tab::CommonNotation}
\end{table}

\subsection{Helper Lemmas}
\label{Sec::HelperLemmas}
\subsubsection{Commonly Used Bounds}
\label{Sec::CommonBounds}
\begin{lemma}
    \label{Lm::CommonBounds}
    Under Condition~\ref{Assump::Standing}, the following bounds hold with high probability.
    
    \begin{subequations}
    \label{Eq::CommonBounds}
    \smallskip
    \noindent\emph{Population bounds}
    \begin{align}
        c\delta_n\leq\min_i t_i\leq\max_i t_i\leq C\delta_n,\qquad\norm{T^{-\alpha}}=O(\delta_n^{-\alpha})\label{Eq::TBound}\\
        \norm{P_{\diag}}=O(\rho_n),\qquad\norm{P}=O(\delta_n),\label{Eq::PBound}\\
        \lambda_{\min}(G_{n,\alpha})\asymp\lambda_{\max}(G_{n,\alpha})\asymp\delta_n^{1-2\alpha},\qquad \norm{G_{n,\alpha}^{-1}}=O(\delta_n^{2\alpha-1}),\label{Eq::GBound}\\
        \min\Lambda_\alpha\asymp \max \Lambda_\alpha\asymp\delta_n^{1-2\alpha},\qquad\norm{\Lambda_\alpha^{-1/2}}=O(\delta_n^{-1/2+\alpha}),\label{Eq::LambdaBound}\\
        \norm{U_\alpha}_{2\to\infty}=O(n^{-1/2}).\label{Eq::URowBound}
        \end{align}
        \smallskip
        \noindent\emph{Degree and adjacency concentration}
        \begin{align}
        \left\lvert d_i-t_i\right\rvert=O_P(\delta_n^{1/2})\text{ for each }i,\qquad\norm{D-T}=O_P(\delta_n^{1/2}\sqrt{\log{n}})\label{Eq::D-TBound}\\
        c\delta_n\leq\min_i d_i\leq\max_i d_i\leq C\delta_n,\qquad\norm{D^{-\alpha}}=O_P(\delta_n^{-\alpha}),\label{Eq::DBound}\\
        \norm{D^{-\alpha}-T^{-\alpha}}= O_P(\delta_n^{-\alpha-1/2}\sqrt{\log{n}}),\label{Eq::D-TAlphaBound}\\
        \norm{E_A}=O_P(\delta_n^{1/2}\sqrt{\log{n}}),\label{Eq::EABound}\\
        \max_i\sum_j |(E_A)_{ij}|=O_P(\delta_n).\label{Eq::EARowBound}
        \end{align}
        \smallskip
        \noindent\emph{Taylor expansion terms}
        \begin{align}
            |(B_D)_{ii}|=O_P(\delta_n^{-\alpha-1/2})\text{ for each }i,\qquad \norm{B_D}=O_P(\delta_n^{-\alpha-1/2}\sqrt{\log{n}}),\label{Eq::BDBound}\\
            |(R_D)_{ii}|=O_P(\delta_n^{-\alpha-1})\text{ for each }i,\qquad\norm{R_D}=O_P(\delta_n^{-\alpha-1}\log{n}),\label{Eq::RDBound}
        \end{align}
        \smallskip
        \noindent\emph{Error bounds}
        \begin{align}
            \norm{E_\alpha}=O_P(\delta_n^{1/2-2\alpha}\sqrt{\log{n}}),\label{Eq::EBound}\\
            \max_i\sum_j|(E_\alpha)_{ij}|=O_P(\delta_n^{1-2\alpha}),\label{Eq::ERowBound}
        \end{align}
    \end{subequations}
\end{lemma}

\begin{proof}
We first establish the population bounds, where probability is with respect to the latent positions. For the remaining bounds, we condition on the latent positions and work on the event that the population bounds hold. Since this event has probability tending to one and the conditional bounds are uniform over this event, the corresponding unconditional bounds hold.

\smallskip
\noindent\emph{Population bounds}\newline
For Equation~\eqref{Eq::TBound}, write 
\begin{align*}
    t_i&=\rho_n\sum_{j\neq i}\xi_i^\top\xi_j\\
    \frac{t_i}{\delta_n}&=\xi_i^\top\left(\frac{1}{n}\sum_{j=1}^n\xi_j\right)-\frac{1}{n}\xi_i^\top\xi_i.
\end{align*}
By the law of large numbers and Condition~\ref{Assump::BoundedSupport}, the right-hand side converges to $\xi_i^\top\mu$ uniformly in $i$. Thus,
\begin{equation}\label{Eq::TDeltaRatio}
    \max_i\left\lvert\frac{t_i}{\delta_n}-\xi_i^\top\mu\right\rvert\overset{p}{\to}0.
\end{equation}
By Conditions~\ref{Assump::BoundedSupport} and~\ref{Assump::Nonzero}, $\xi_i^\top\mu$ is uniformly bounded above and away from zero. Equation~\eqref{Eq::TDeltaRatio} therefore implies the first inequality in Equation~\eqref{Eq::TBound} with high probability. The second statement follows immediately.

For Equation~\eqref{Eq::PBound}, $p_{ii}=X_i^\top X_i=\rho_n\xi_i^\top\xi_i$. By Condition~\ref{Assump::BoundedSupport}, this is bounded above by a constant times $\rho_n$, and the first statement follows. For the second statement, the Gershgorin circle theorem states that the eigenvalues of a symmetric matrix with nonnegative entries are bounded above by its row sums. The bound then follows from Equation~\eqref{Eq::TBound} and the first statement, since the row sums of $P$ are given by $t_i+p_{ii}$. 

For Equation~\eqref{Eq::GBound}, we use the facts that $X_i=\rho_n^{1/2}\xi_i$ and $t_i\asymp\delta_n$ to write 
\begin{align*}
    G_{n,\alpha}=X^\top T^{-2\alpha}X&=\sum_{i=1}^n t_i^{-2\alpha}X_iX_i^\top\\
    &=\rho_n\delta_n^{-2\alpha}\sum_{i=1}^n\frac{\xi_i\xi_i^\top}{(t_i/\delta_n)^{2\alpha}}.
\end{align*}
By Equation~\eqref{Eq::TDeltaRatio},
\[
    \delta_n^{2\alpha-1}G_{n,\alpha}=\frac{1}{n}\sum_{i=1}^n\frac{\xi_i\xi_i^\top}{\langle \xi_i,\mu\rangle^{2\alpha}}+o_P(1).
\]
By Conditions~\ref{Assump::BoundedSupport} and~\ref{Assump::Nonzero}, the right-hand side converges to $\Upsilon_\alpha$, which is a fixed, $r$-dimensional positive definite matrix. The first statement follows, and the second statement follows immediately from the first.  

For Equation~\eqref{Eq::LambdaBound}, note that we want to bound the eigenvalues of $L_\alpha=T^{-\alpha}XX^\top T^{-\alpha}$. Since $G_{n,\alpha}=X^\top T^{-2\alpha}X$, the nonzero eigenvalues of $L_\alpha$ and $G_{n,\alpha}$ are the same. The statements then follow immediately from Equation~\eqref{Eq::GBound}. 

For Equation~\eqref{Eq::URowBound}, write $G_{n,\alpha}=V_\alpha\Lambda_\alpha V_\alpha^\top$ and $U_\alpha=T^{-\alpha}X V_\alpha\Lambda_\alpha^{-1/2}$. It is straightforward to verify $U_\alpha^\top U_\alpha=I_r$ and $L_\alpha U_\alpha=U_\alpha\Lambda_\alpha$. Thus, for each $i$,
\[
    \norm{(U_\alpha)_{i*}}\leq t_i^{-\alpha}\norm{X_i}\norm{\Lambda_\alpha^{-1/2}}.
\]
Using Equations~\eqref{Eq::TBound}, \eqref{Eq::LambdaBound} and Condition~\ref{Assump::Standing},
\[
    \norm{(U_\alpha)_{i*}}=O\left(\delta_n^{-\alpha}\rho_n^{1/2}\delta_n^{-1/2+\alpha}\right)=O(n^{-1/2}).
\]
The statement follows by taking the maximum over $i$.

\smallskip
\noindent\emph{Degree and adjacency concentration}\newline
Throughout the remainder of the proof, we write $C$ for a positive constant independent of $n$, whose values may change from line to line. Additionally, we write $M$ for a sufficiently large positive constant, chosen as needed. For Equation~\eqref{Eq::D-TBound}, note that  $d_i-t_i$ is a sum of mean-zero random variables. Furthermore, the summands are bounded by $1$ and the sum of the variances is bounded by $t_i$. By Equation~\eqref{Eq::TBound}, $\sum_{j\neq i}p_{ij}=t_{i}\leq C\delta_n$. Thus, by Bernstein's inequality (see, for example, \citet{boucheron2013inequalities}), for all $u>0$,
\[
    \mathbb P\left(\left\lvert d_i-t_i\right\rvert>u\right)\leq2\exp\left(-\frac{u^2}{2(C\delta_n+u/3)}\right).
\]
Choosing $u=M\delta_n^{1/2}$, we have, for $n$ sufficiently large,
\begin{align*}
    \mathbb P\left(|d_i-t_i|>M\delta_n^{1/2}\right)&\leq2\exp\left(-\frac{M^2\delta_n}{2(C\delta_n+M\delta_n^{1/2}/3)}\right)\\
    &\leq2\exp\left(-\frac{M^2}{2(C+M/(3\delta_n^{1/2}))}\right)\\
    &\leq2\exp\left(-CM^2\right).
\end{align*}
The first statement follows. The second statement is similar, choosing $u=M\delta_n^{1/2}\sqrt{\log{n}}$ and taking a union bound. 

Equations~\eqref{Eq::DBound} and \eqref{Eq::D-TAlphaBound} follow easily from Equations~\eqref{Eq::TBound} and \eqref{Eq::D-TBound}.

For Equation~\eqref{Eq::EABound}, note that we may write $E_A$ as a sum of independent, symmetric, mean-zero matrices by defining $Y_{ij}=(E_A)_{ij}(e_ie_j^\top+e_je_i^\top)$, where $e_k$ denotes the $k$th standard basis vector in $\mathbb{R}^n$. Then 
\[
    \norm{Y_{ij}}\leq|(E_A)_{ij}|\,\norm{e_ie_j^\top+e_je_i^\top}\leq 1.
\]
Additionally,
\[
    \sum_{i<j}\mathbb E(Y_{ij}^2)=\sum_{i<j}p_{ij}(1-p_{ij})(e_ie_i^\top+e_je_j^\top),
\]
so 
\[ 
    \norm{\sum_{i<j}\mathbb E(Y_{ij}^2)}\leq\max_i\sum_{j\neq i}p_{ij}(1-p_{ij})\leq\max_i\sum_{j\neq i}p_{ij}=\max_i t_i\leq C\delta_n.
\]

Thus, by the matrix Bernstein inequality (see, for example, \citet{tropp2012user}), for all $u>0$,
\[
    \mathbb P\left(\norm{E_A}>u\right)\leq 2n\exp\left(-\frac{u^2}{2(C\delta_n+u/3)}\right).
\]
Choosing $u=M\delta_n^{1/2}\sqrt{\log{n}}$,
\begin{align*}
    \mathbb P\left(\norm{E_A}>M\delta_n^{1/2}\sqrt{\log{n}}\right)&\leq 2n\exp\left(-\frac{M^2\delta_n\log{n}}{2(C\delta_n+M\delta_n^{1/2}\sqrt{\log{n}}/3)}\right)\\
    &\leq 2n\exp\left(-CM^2\log{n}\right).
\end{align*}
The statement follows by choosing $M$ sufficiently large.

For Equation~\eqref{Eq::EARowBound}, we write  
\[
    \sum_j |(E_A)_{ij}|=\sum_{j\neq i} |a_{ij}-p_{ij}|\leq\sum_{j\neq i} a_{ij}+\sum_{j\neq i} p_{ij}=d_i+t_i.
\]
By Equation~\eqref{Eq::TBound}, we have $t_i=O(\delta_n)$ and by Equation~\eqref{Eq::D-TBound}, we have $d_i=t_i+O_P(\delta_n^{1/2}\sqrt{\log{n}})=O_P(\delta_n)$, since $\sqrt{\log{n}}=o(\delta_n^{1/2})$. The statement follows. 

\smallskip
\noindent\emph{Taylor expansion terms}\newline
For Equation~\eqref{Eq::BDBound}, we write $(B_D)_{ii}=\alpha t_i^{-\alpha-1}(t_i-d_i)$ and use Equations~\eqref{Eq::TBound} and \eqref{Eq::D-TBound}.

For Equation~\eqref{Eq::RDBound}, we use Taylor's theorem to write 
\[
    d_i^{-\alpha}=t_i^{-\alpha}+\alpha t_i^{-\alpha-1}(t_i-d_i)+O\left(t_i^{-\alpha-2}(d_i-t_i)^2\right),
\]
then apply Equations~\eqref{Eq::TBound} and \eqref{Eq::D-TBound} to the last term.

\smallskip
\noindent\emph{Error bounds}\newline
For Equation~\eqref{Eq::EBound}, we first write
\[
    E_\alpha=\T E_A\T+B_D P\T+\T P B_D+B_D E_A\T+\T E_AB_D+\mathcal R_E,
\]
where
\begin{align*}
    \mathcal R_E&=R_D(P+E_A)\T+\T(P+E_A)R_D \\
    &\quad+B_D(P+E_A)B_D+B_D(P+E_A)R_D+R_D(P+E_A)B_D+R_D(P+E_A)R_D \\
    &\quad-(\T+B_D+R_D)P_{\diag}(\T+B_D+R_D).
\end{align*}

By Equations~\eqref{Eq::TBound} and
\eqref{Eq::EABound},
\[
    \norm{\T E_A\T}\leq \norm{\T}^2\norm{E_A}=O_P\left(\delta_n^{1/2-2\alpha}\sqrt{\log n}\right).
\]

Each of the remaining terms may be individually bounded. Because this decomposition appears several times throughout the appendix in slightly different forms, we prove a slightly more general form in Lemma~\ref{Lm::TaylorBookkeeping}. Using the notation of the lemma, let $H=P+E_A=A+P_{\diag}$. Since $H$ is symmetric and $\max_i\sum_j\left\lvert h_{ij}\right\rvert=O_P(\delta_n)$, the operator-norm version of Lemma~\ref{Lm::TaylorBookkeeping}, applied with $V=I_n$ gives 
\[
    \norm{B_DH\T}+\norm{\T HB_D}
    =O_P\left(\delta_n^{1/2-2\alpha}\sqrt{\log n}\right)
\]
and
\[
    \norm{\mathcal R_{H,I_n}}
    =O_P(\delta_n^{-2\alpha}\log n)
    =o_P\left(\delta_n^{1/2-2\alpha}\sqrt{\log n}\right).
\]
This bounds the first six terms of $\mathcal{R}_E$. For the last term, use Equations~\eqref{Eq::TBound}, \eqref{Eq::BDBound}, \eqref{Eq::RDBound}, and \eqref{Eq::PBound} to get
\begin{align*}
    \norm{(\T+B_D+R_D)P_{\diag}(\T+B_D+R_D)}&\leq\norm{\T+B_D+R_D}^2\norm{P_{\diag}} \\
    &=O_P(\delta_n^{-2\alpha}\rho_n)=O_P(\delta_n^{1/2-2\alpha}\sqrt{\log{n}}),
\end{align*}
where we have used the fact that $\rho_n\leq 1$ and $\sqrt{\log{n}}=o(\delta_n^{1/2})$.

The proof of Equation~\eqref{Eq::ERowBound} is the same as that of Equation~\eqref{Eq::EBound}, with the matrix infinity norm in place of the operator norm. In particular, for $H=P+E_A=A+P_{\diag}$,
\[
    \norm{H}_\infty=\max_i(d_i+p_{ii})=O_P(\delta_n),
\]
so the infinity-norm version of Lemma~\ref{Lm::TaylorBookkeeping} applies to the first six terms of $\mathcal R_E$ with $V=I_n$.
\end{proof}

\begin{remark}
    For the remainder of the appendix, whenever we establish concentration or remainder bounds, we condition on the latent positions and work on the event on which the population bounds in Lemma~\ref{Lm::CommonBounds} hold. The conditional estimates below are uniform over this event, and since the probability of this event goes to $1$, they imply the corresponding unconditional $O_P$ bounds. We suppress this conditioning from the notation.
\end{remark}

\subsubsection{Taylor Remainder Bound}
\label{Sec::TaylorRemainderBound}
\begin{lemma}
    \label{Lm::TaylorBookkeeping}
    Let
    \[
        (D^{-\alpha}HD^{-\alpha}-T^{-\alpha}HT^{-\alpha})V=B_DHT^{-\alpha}V+T^{-\alpha}HB_DV+\mathcal R_{H,V},
    \] 
    where $H$ is any symmetric random matrix satisfying $\max_i\sum_j\left\lvert h_{ij}\right\rvert=O_P(\delta_n)$, $V$ is any random matrix satisfying $\norm{V}_{2\to\infty}=O_P(v_n)$, and
    \begin{align*}
        \mathcal R_{H,V}&=R_DHT^{-\alpha}V+T^{-\alpha}HR_DV+B_DHB_DV+B_DHR_DV+R_DHB_DV+R_DHR_DV.
    \end{align*}
    The first two terms satisfy
    \[
        \norm{B_DHT^{-\alpha}V}_{2\to\infty}+\norm{T^{-\alpha}HB_DV}_{2\to\infty}=O_P\left(\delta_n^{1/2-2\alpha}\sqrt{\log{n}}\,v_n\right),
    \]
    and the remaining terms satisfy
    \[
        \norm{\mathcal R_{H,V}}_{2\to\infty}=O_P\left(\delta_n^{-2\alpha}\log{n}\,v_n\right).
    \]
    The same bounds hold if every occurrence of the $2\to\infty$ norm, including the assumption on $V$, is replaced by the operator norm or the matrix infinity norm. 
\end{lemma}

\begin{proof}
    We begin with the two terms $ \norm{B_DHT^{-\alpha}V}_{2\to\infty}+\norm{T^{-\alpha}HB_DV}_{2\to\infty}$. Write
    \[
        M_L,M_R\in\{T^{-\alpha},B_D,R_D\}.
    \]
For each row $i$,
\begin{align*}
    \norm{(M_LHM_RV)_{i*}}&\leq|(M_L)_{ii}|\sum_{j=1}^n|h_{ij}||(M_R)_{jj}|\norm{V_{j*}}  \\
    &\leq\norm{M_L}\left(\max_i\sum_{j=1}^n |h_{ij}|\right)\norm{M_R}\norm{V}_{2\to\infty}.
\end{align*}
By assumption,
\[
    \max_i\sum_{j=1}^n |h_{ij}|=O_P(\delta_n)\qquad\text{and}\qquad\norm{V}_{2\to\infty}=O_P(v_n).
\]

By Equations~\eqref{Eq::BDBound} and \eqref{Eq::TBound}, if one of $M_L,M_R$ is $B_D$ and the other is $T^{-\alpha}$, then
\begin{align*}
    \norm{M_LHM_RV}_{2\to\infty}&=O_P\left(\delta_n^{-\alpha-1/2}\sqrt{\log{n}}\cdot \delta_n\cdot \delta_n^{-\alpha}\cdot v_n\right) \\
    &=O_P\left(\delta_n^{1/2-2\alpha}\sqrt{\log{n}} v_n\right).
\end{align*}
This proves the statement on the first two terms. 

The remaining terms contain $R_D$ or at least two factors from $\{B_D,R_D\}$. In each case, $\norm{M_L}\norm{M_R} =O_P(\delta_n^{-2\alpha-1}\log n)$. Thus,
\begin{align*}
    \norm{M_LHM_RV}_{2\to\infty}&=O_P\left(\delta_n^{-2\alpha-1}\log{n}\cdot \delta_n\cdot v_n\right) \\
    &=O_P\left(\delta_n^{-2\alpha}\log{n} v_n\right).
\end{align*}
This proves the statement for the remaining terms.

Since $H$ is symmetric,
\[
    \norm{H}\leq \max_i\sum_j |h_{ij}|=O_P(\delta_n),
\]
so the same argument, together with submultiplicativity, applies with every occurrence of the $2\to\infty$ norm replaced by the appropriate norm.
\end{proof}

\subsection{Derivation of the Spectral Expansion and Proof of Negligibility}
\label{Sec::SpectralBounds}
\begin{lemma}
    \label{Lm::SpectralExpansion}
    The following identity holds:
    \[
        \hat{U}_\alpha\hat{\Lambda}_\alpha^{1/2}W_nV_\alpha^\top-U_\alpha\Lambda_\alpha^{1/2}V_\alpha^\top=E_\alpha U_\alpha\Lambda_\alpha^{-1/2}V_\alpha^\top+\mathcal{R}_{\mathrm{sp},\alpha}V_\alpha^\top,
    \]
    where
    \begin{align*}
        \mathcal{R}_{\mathrm{sp},\alpha}&=U_\alpha\left(W_n^\top\hat{\Lambda}_\alpha^{1/2}W_n-\Lambda_\alpha^{1/2}-U_\alpha^\top E_\alpha U_\alpha \Lambda_\alpha^{-1/2}+(U_\alpha^\top\hat U_\alpha W_n-I_r)W_n^\top\hat{\Lambda}_\alpha^{1/2}W_n\right)\\
        &\qquad+\Pi^\perp(U_\alpha)E_\alpha\left(\hat{U}_\alpha\hat{\Lambda}_\alpha^{-1/2}W_n-U_\alpha\Lambda_\alpha^{-1/2}\right).
    \end{align*}
\end{lemma}

\begin{proof}
Using $\Pi^\perp(U_\alpha)=I_n-U_\alpha U_\alpha^\top$,
\begin{align*}
    \hat{U}_\alpha\hat{\Lambda}_\alpha^{1/2}W_n-U_\alpha\Lambda_\alpha^{1/2}&=U_\alpha U_\alpha^\top\left(\hat{U}_\alpha\hat{\Lambda}_\alpha^{1/2}W_n-U_\alpha\Lambda_\alpha^{1/2}\right)+\Pi^\perp(U_\alpha)\left(\hat{U}_\alpha\hat{\Lambda}_\alpha^{1/2}W_n-U_\alpha\Lambda_\alpha^{1/2}\right).
\end{align*}

Note that $\Pi^\perp(U_\alpha) U_\alpha=0$. Since $\hat{L}_\alpha\hat{U}_\alpha=\hat{U}_\alpha\hat{\Lambda}_\alpha$, we have $\hat{U}_\alpha\hat{\Lambda}_\alpha^{1/2}W_n=\hat L_\alpha \hat U_\alpha\hat{\Lambda}_\alpha^{-1/2}W_n$. Thus,
\begin{align*}
    \hat{U}_\alpha\hat{\Lambda}_\alpha^{1/2}W_n-U_\alpha\Lambda_\alpha^{1/2}&=U_\alpha U_\alpha^\top\left(\hat{U}_\alpha\hat{\Lambda}_\alpha^{1/2}W_n-U_\alpha\Lambda_\alpha^{1/2}\right)+\Pi^\perp(U_\alpha)\hat{U}_\alpha\hat{\Lambda}_\alpha^{1/2}W_n\\
    &=U_\alpha U_\alpha^\top\left(\hat{U}_\alpha\hat{\Lambda}_\alpha^{1/2}W_n-U_\alpha\Lambda_\alpha^{1/2}\right)+\Pi^\perp(U_\alpha) \hat{L}_\alpha\hat U_\alpha\hat{\Lambda}_\alpha^{-1/2}W_n\\
    &=U_\alpha U_\alpha^\top\left(\hat{U}_\alpha\hat{\Lambda}_\alpha^{1/2}W_n-U_\alpha\Lambda_\alpha^{1/2}\right)+\Pi^\perp(U_\alpha) (L_\alpha+E_\alpha)\hat U_\alpha\hat{\Lambda}_\alpha^{-1/2}W_n\\
    &=U_\alpha U_\alpha^\top\left(\hat{U}_\alpha\hat{\Lambda}_\alpha^{1/2}W_n-U_\alpha\Lambda_\alpha^{1/2}\right)+\Pi^\perp(U_\alpha) E_\alpha\hat U_\alpha\hat{\Lambda}_\alpha^{-1/2}W_n,
\end{align*}
where in the last step we have used the fact that $\Pi^\perp(U_\alpha) L_\alpha=\Pi^\perp(U_\alpha) U_\alpha\Lambda_\alpha U_\alpha^\top=0$. 

For the first term,
\begin{align}
    U_\alpha U_\alpha^\top\left(\hat{U}_\alpha\hat{\Lambda}_\alpha^{1/2}W_n-U_\alpha\Lambda_\alpha^{1/2}\right)&=U_\alpha\left(U_\alpha^\top\hat U_\alpha\hat{\Lambda}_\alpha^{1/2}W_n-\Lambda_\alpha^{1/2}\right)\nonumber\\
    &=U_\alpha\left(U_\alpha^\top\hat U_\alpha W_nW_n^\top \hat{\Lambda}_\alpha^{1/2}W_n-\Lambda_\alpha^{1/2}\right)\nonumber\\
    &=U_\alpha\left(W_n^\top \hat{\Lambda}_\alpha^{1/2}W_n
    -\Lambda_\alpha^{1/2}+(U_\alpha^\top\hat U_\alpha W_n-I_r)W_n^\top \hat{\Lambda}_\alpha^{1/2}W_n\right)\label{Eq::SpecTerm1}.
\end{align}

For the second term, adding and subtracting $U_\alpha\Lambda_\alpha^{-1/2}$, we get
\begin{align}
    \Pi^\perp(U_\alpha)&E_\alpha\hat U_\alpha\hat{\Lambda}_\alpha^{-1/2}W_n\nonumber\\
    &\qquad=\Pi^\perp(U_\alpha) E_\alpha U_\alpha\Lambda_\alpha^{-1/2}+\Pi^\perp(U_\alpha) E_\alpha\left(\hat U_\alpha\hat{\Lambda}_\alpha^{-1/2}W_n-U_\alpha\Lambda_\alpha^{-1/2}\right)\nonumber\\
    &\qquad=E_\alpha U_\alpha\Lambda_\alpha^{-1/2}-U_\alpha U_\alpha^\top E_\alpha U_\alpha\Lambda_\alpha^{-1/2}+\Pi^\perp(U_\alpha) E_\alpha\left(\hat U_\alpha\hat{\Lambda}_\alpha^{-1/2}W_n-U_\alpha\Lambda_\alpha^{-1/2}\right)\label{Eq::SpecTerm2}.
\end{align}

Combining Equations~\eqref{Eq::SpecTerm1} and \eqref{Eq::SpecTerm2} gives 
\begin{align*}
    \hat{U}_\alpha\hat{\Lambda}_\alpha^{1/2}W_n&-U_\alpha\Lambda_\alpha^{1/2}=E_\alpha U_\alpha\Lambda_\alpha^{-1/2} -U_\alpha U_\alpha^\top E_\alpha U_\alpha\Lambda_\alpha^{-1/2}+\Pi^\perp(U_\alpha) E_\alpha\left(\hat U_\alpha\hat{\Lambda}_\alpha^{-1/2}W_n-U_\alpha\Lambda_\alpha^{-1/2}\right)\\
    &\quad+U_\alpha\left(W_n^\top \hat{\Lambda}_\alpha^{1/2}W_n-\Lambda_\alpha^{1/2}+(U_\alpha^\top\hat U_\alpha W_n-I_r)W_n^\top \hat{\Lambda}_\alpha^{1/2}W_n\right),
\end{align*}

and factoring the second term finally gives
\begin{align*}
    \hat{U}_\alpha\hat{\Lambda}_\alpha^{1/2}W_n&-U_\alpha\Lambda_\alpha^{1/2}=E_\alpha U_\alpha\Lambda_\alpha^{-1/2}\\
    &\quad+U_\alpha\left(W_n^\top\hat\Lambda_\alpha^{1/2}W_n-\Lambda_\alpha^{1/2}-U_\alpha^\top E_\alpha U_\alpha\Lambda_\alpha^{-1/2}+(U_\alpha^\top\hat U_\alpha W_n-I_r)W_n^\top\hat\Lambda_\alpha^{1/2}W_n\right)\\
    &\quad+\Pi^\perp(U_\alpha) E_\alpha\left(\hat U_\alpha\hat{\Lambda}_\alpha^{-1/2}W_n-U_\alpha\Lambda_\alpha^{-1/2}\right).
\end{align*}
Multiplication on the right by $V_{\alpha}^\top$ completes the proof. 
\end{proof}

\begin{lemma}
    \label{lem:EUrow}
    Under Condition~\ref{Assump::Standing},
    \[
        \norm{E_\alpha U_\alpha}_{2\to\infty}=O_P\left(n^{-1/2}\delta_n^{1/2-2\alpha}\sqrt{\log{n}}\right)=O_P\left(\rho_n^{1/2}\delta_n^{-2\alpha}\sqrt{\log{n}}\right).
    \]
\end{lemma}

\begin{proof}
We first write
\[
    E_\alpha U_\alpha=\T E_A\T U_\alpha+B_D P\T U_\alpha+\T P B_DU_\alpha+B_D E_A\T U_\alpha+\T E_AB_D U_\alpha+\mathcal R_E U_\alpha,
\]
where
\begin{align*}
\mathcal R_E&=R_D(P+E_A)\T+\T(P+E_A)R_D \\
&\quad+B_D(P+E_A)B_D+B_D(P+E_A)R_D+R_D(P+E_A)B_D+R_D(P+E_A)R_D \\
&\quad-(\T+B_D+R_D)P_{\diag}(\T+B_D+R_D).
\end{align*}
For the first term, conditional on the latent positions, the $i$th row is a weighted sum of independent mean-zero vectors,
\[
    e_i^\top \T E_A\T U_\alpha=t_i^{-\alpha}\sum_j (E_A)_{ij}t_j^{-\alpha}(U_\alpha)_{j*}.
\]
The weights satisfy
\[
    \max_{i,j}\norm{t_i^{-\alpha}t_j^{-\alpha}(U_\alpha)_{j*}}\leq Cn^{-1/2}\delta_n^{-2\alpha}
\]
and
\[
    \max_i t_i^{-2\alpha}\sum_j p_{ij}t_j^{-2\alpha}\norm{(U_\alpha)_{j*}}^2\leq Cn^{-1}\delta_n^{1-4\alpha}.
\]
Applying Bernstein's inequality (see, e.g., \citet[Theorem~2.8.4]{vershynin2018high}) coordinatewise and taking a union bound over $i$ and the fixed number of coordinates gives
\begin{align*}
    \norm{\T E_A\T U_\alpha}_{2\to\infty}&=O_P\left(n^{-1/2}\delta_n^{1/2-2\alpha}\sqrt{\log{n}}+n^{-1/2}\delta_n^{-2\alpha}\log n\right)\\
    &=O_P\left(n^{-1/2}\delta_n^{1/2-2\alpha}\sqrt{\log{n}}\right),
\end{align*}
where the second equality uses the assumption that $\delta_n=\omega((\log n)^2)$. For the second term, since $L_\alpha U_\alpha=U_\alpha\Lambda_\alpha$ and $L_\alpha=\T P\T$, $P\T U_\alpha=T^\alpha U_\alpha\Lambda_\alpha$. Thus, we may write
\[
    \norm{(P\T U_\alpha)_{i*}}\leq C\delta_n^\alpha n^{-1/2}\delta_n^{1-2\alpha}=Cn^{-1/2}\delta_n^{1-\alpha}.
\]
Using Equation~\eqref{Eq::BDBound},
\[
    \norm{B_D P\T U_\alpha}_{2\to\infty}\leq\norm{B_D}\max_i\norm{(P\T U_\alpha)_{i*}}=O_P\left(n^{-1/2}\delta_n^{1/2-2\alpha}\sqrt{\log{n}}\right).
\]
Similarly, for the third term,
\begin{align*}
    \norm{(\T PB_DU_\alpha)_{i*}}&\leq C\delta_n^{-\alpha}\sum_j p_{ij}\norm{B_D}\norm{(U_\alpha)_{j*}}\\
    &\leq C\delta_n^{-\alpha}\cdot\delta_n\cdot\delta_n^{-\alpha-1/2}\sqrt{\log{n}}\cdot n^{-1/2}=O_P\left(n^{-1/2}\delta_n^{1/2-2\alpha}\sqrt{\log{n}}\right).
\end{align*}
For the fourth and fifth terms, using Equation~\eqref{Eq::EARowBound} gives
\begin{align*}
    \norm{B_DE_A\T U_\alpha}_{2\to\infty}+\norm{\T E_AB_DU_\alpha}_{2\to\infty}&\leq2\norm{B_D}\left(\max_i\sum_j|(E_A)_{ij}|\right)\norm{\T}\norm{U_\alpha}_{2\to\infty}\\
    &=O_P\left(n^{-1/2}\delta_n^{1/2-2\alpha}\sqrt{\log n}\right).
\end{align*}
For the final term involving $\mathcal{R}_E$, let $H=P+E_A=A+P_{\diag}$, and note that the first six terms are then $\mathcal R_{H,U_\alpha}$ from Lemma~\ref{Lm::TaylorBookkeeping}. 
By Equations~\eqref{Eq::DBound} and \eqref{Eq::PBound}, we have
\[
    \max_i\sum_j|h_{ij}|=\max_i(d_i+p_{ii})=O_P(\delta_n).
\]
Lemma~\ref{Lm::TaylorBookkeeping} and Equation~\eqref{Eq::URowBound} then give
\[
    \norm{\mathcal R_{H,U_\alpha}}_{2\to\infty}=O_P\left(n^{-1/2}\delta_n^{-2\alpha}\log n\right)=o_P\left(n^{-1/2}\delta_n^{1/2-2\alpha}\sqrt{\log n}\right).
\]

Finally, since $\T+B_D+R_D=D^{-\alpha}$, Equations~\eqref{Eq::DBound}, \eqref{Eq::PBound}, and \eqref{Eq::URowBound} give
\begin{align*}
    \norm{D^{-\alpha}P_{\diag}D^{-\alpha}U_\alpha}_{2\to\infty}
    &\leq\norm{D^{-\alpha}}^2\norm{P_{\diag}}\norm{U_\alpha}_{2\to\infty}\\
    &=O_P\left(n^{-1/2}\rho_n\delta_n^{-2\alpha}\right)\\
    &=o_P\left(n^{-1/2}\delta_n^{1/2-2\alpha}\sqrt{\log n}\right).
\end{align*}
Combining the preceding bounds proves the result.
\end{proof}

\begin{lemma}
    \label{Lm::UTEU}
    Under Condition~\ref{Assump::Standing},
    \[
        \norm{U_\alpha^\top E_\alpha U_\alpha}=O_P\left(\delta_n^{-2\alpha}\log{n}\right).
    \]
\end{lemma}

\begin{proof}
We first write
\begin{align*}
    U_\alpha^\top E_\alpha U_\alpha&=U_\alpha^\top T^{-\alpha} E_AT^{-\alpha} U_\alpha+U_\alpha^\top B_D PT^{-\alpha} U_\alpha+U_\alpha^\top T^{-\alpha}P B_D U_\alpha\\
    &\quad+U_\alpha^\top B_D E_AT^{-\alpha} U_\alpha+U_\alpha^\top T^{-\alpha}E_AB_D U_\alpha+U_\alpha^\top\mathcal R_E U_\alpha,
\end{align*}
where
\begin{align*}
    \mathcal R_E&=R_D(P+E_A)T^{-\alpha}+T^{-\alpha}(P+E_A)R_D \\
    &\quad+B_D(P+E_A)B_D+B_D(P+E_A)R_D+R_D(P+E_A)B_D+R_D(P+E_A)R_D \\
    &\quad-(T^{-\alpha}+B_D+R_D)P_{\diag}(T^{-\alpha}+B_D+R_D).
\end{align*}
For the first term, for $k,\ell\in\{1,\ldots,r\}$,
\[
    \left(U_\alpha^\top T^{-\alpha}E_AT^{-\alpha}U_\alpha\right)_{k\ell}=\sum_{a<b} (E_A)_{ab}\,t_a^{-\alpha}t_b^{-\alpha}\left((U_\alpha)_{ak}(U_\alpha)_{b\ell}+(U_\alpha)_{bk}(U_\alpha)_{a\ell}\right).
\]
Conditional on the latent positions, this is a weighted sum of independent mean-zero random variables. By Equations~\eqref{Eq::TBound} and \eqref{Eq::URowBound}, the weights satisfy
\[
    \left\lvert t_a^{-\alpha}t_b^{-\alpha}\left((U_\alpha)_{ak}(U_\alpha)_{b\ell}+(U_\alpha)_{bk}(U_\alpha)_{a\ell}\right)\right\rvert\leq Cn^{-1}\delta_n^{-2\alpha}.
\]
Consequently,
\begin{align*}
    \operatorname{Var}\left(\left(U_\alpha^\top T^{-\alpha}E_AT^{-\alpha}U_\alpha\right)_{k\ell}\right)&\leq\sum_{a<b} t_a^{-2\alpha}p_{ab}t_b^{-2\alpha}\left\lvert\left((U_\alpha)_{ak}(U_\alpha)_{b\ell}+(U_\alpha)_{bk}(U_\alpha)_{a\ell}\right)\right\rvert^2\\
    &\leq C\delta_n^{-4\alpha}n^{-2}\sum_{a<b} p_{ab} \\
    &\leq Cn^{-1}\delta_n^{1-4\alpha}.
\end{align*}

By Chebyshev's inequality, for every $M>0$,
\begin{align*}
    \mathbb P\left(\left\lvert\left(U_\alpha^\top T^{-\alpha}E_AT^{-\alpha}U_\alpha\right)_{k\ell}\right\rvert>M\rho_n^{1/2}\delta_n^{-2\alpha}\right)\leq\frac{Cn^{-1}\delta_n^{1-4\alpha}}{M^2\rho_n\delta_n^{-4\alpha}}=\frac{C}{M^2}.
\end{align*}
Hence, 
\begin{equation}
    \label{Eq::UTEATUEntryBound}
    \left(U_\alpha^\top T^{-\alpha}E_AT^{-\alpha}U_\alpha
    \right)_{k\ell}=O_P(\rho_n^{1/2}\delta_n^{-2\alpha}).
\end{equation}
Because $r$ is fixed,
\[
    \norm{U_\alpha^\top T^{-\alpha}E_AT^{-\alpha}U_\alpha}\leq r\cdot\max_{k,\ell}\left\lvert \left(U_\alpha^\top T^{-\alpha}E_AT^{-\alpha}U_\alpha\right)_{k\ell}\right\rvert.
\]
Combining this with Equation~\eqref{Eq::UTEATUEntryBound} and noting that $\rho_n\leq1$, we have 
\begin{equation}
    \label{Eq::UTZTUBound}
    \norm{U_\alpha^\top T^{-\alpha}E_AT^{-\alpha}U_\alpha}=O_P\left(\delta_n^{-2\alpha}\right).
\end{equation}

For the second term, we write $U_\alpha^\top B_DP\T U_\alpha=U_\alpha^\top B_DT^\alpha U_\alpha\Lambda_\alpha$, and
\begin{align*}
    (U_\alpha^\top B_DT^\alpha U_\alpha)_{k\ell}&=\alpha\sum_{a} (U_\alpha)_{ak}(U_\alpha)_{a\ell}\frac{t_a-d_a}{t_a}\\
    &=-\alpha\sum_{a}(U_\alpha)_{ak}(U_\alpha)_{a\ell}t_a^{-1}\sum_{b\neq a}(E_A)_{ab}\\
    &=-\alpha\sum_{a<b}(E_A)_{ab}\left(t_a^{-1}(U_\alpha)_{ak}(U_\alpha)_{a\ell}+t_b^{-1}(U_\alpha)_{bk}(U_\alpha)_{b\ell}\right).
\end{align*}
Conditional on the latent positions, this is again a weighted sum of independent mean-zero random variables. By Equations~\eqref{Eq::TBound} and \eqref{Eq::URowBound}, the weights satisfy 
\[
    \left\lvert t_a^{-1}(U_\alpha)_{ak}(U_\alpha)_{a\ell}+t_b^{-1}(U_\alpha)_{bk}(U_\alpha)_{b\ell}\right\rvert\leq Cn^{-1}\delta_n^{-1}
\]
and 
\begin{align*}
    \sum_{a<b} p_{ab}\left\lvert\left(t_a^{-1}(U_\alpha)_{ak}(U_\alpha)_{a\ell}+t_b^{-1}(U_\alpha)_{bk}(U_\alpha)_{b\ell}\right)\right\rvert^2&\leq C\delta_n^{-2}n^{-2}\sum_{a<b} p_{ab} \\
    &\leq Cn^{-1}\delta_n^{-1}.
\end{align*}
An analogous argument to that used to obtain Equation~\eqref{Eq::UTZTUBound} gives 
\[
    \norm{U_\alpha^\top B_DT^\alpha U_\alpha}=O_{P}\left(\delta_n^{-1}\right).
\]

After multiplication by $\Lambda_\alpha$, Equation~\eqref{Eq::LambdaBound} gives
\[
    \norm{U_\alpha^\top B_DP\T U_\alpha}=O_{P}\left(\delta_n^{-2\alpha}\right).
\]

The third term, $U_\alpha^\top\T PB_DU_\alpha$, is the transpose of the second and thus shares the same bound.

The fourth and fifth terms are bounded by 
\[
    \norm{U_\alpha^\top B_DE_A\T U_\alpha}+\norm{U_\alpha^\top \T E_AB_DU_\alpha}\leq2\norm{B_D}\norm{E_A}\norm{\T}=O_P(\delta_n^{-2\alpha}\log{n}).
\]
For the final term involving $\mathcal{R}_E$, let $H=P+E_A=A+P_{\diag}$, and note that the first six terms are then $\mathcal{R}_{H,I_n}$ from Lemma~\ref{Lm::TaylorBookkeeping}. Then
\[
    \norm{\mathcal{R}_{H,I_n}}=O_P(\delta_n^{-2\alpha}\log{n}).
\]
Moreover,
\[
    \norm{D^{-\alpha}P_{\diag}D^{-\alpha}}=O_P(\delta_n^{-2\alpha}\rho_n)
\]
and
\[
    \norm{\mathcal R_E}=O_P(\delta_n^{-2\alpha}\log n).
\]

Thus, 
\[
    \norm{U_\alpha^\top E_\alpha U_\alpha}=O_P(\delta_n^{-2\alpha}\log{n}).
\]
\end{proof}

\begin{lemma}
    \label{lem:signal}
    Under Condition~\ref{Assump::Standing}, for any fixed $q\in\mathbb R$,
    \[
        \norm{U_\alpha^\top\hat U_\alpha\hat\Lambda_\alpha^q-\Lambda_\alpha^q U_\alpha^\top\hat U_\alpha}=O_P\left(\delta_n^{q(1-2\alpha)-1}\log{n}\right).
    \]
    Consequently, the following three bounds hold:
    \begin{align}
        \norm{W_n^\top\hat\Lambda_\alpha^q W_n-\Lambda_\alpha^q}&=O_P\left(\delta_n^{q(1-2\alpha)-1}\log{n}\right)\label{Eq::WLW-LBound}\\
        \norm{\left(U_\alpha^\top\hat U_\alpha W_n-I_r\right)W_n^\top\hat\Lambda_\alpha^{1/2}W_n}&=O_P\left(\delta_n^{-\alpha-1/2}\log{n}\right)\label{Eq::UUW-IWLWBound}\\
        \norm{\hat U_\alpha\hat\Lambda_\alpha^{-1/2}W_n-U_\alpha\Lambda_\alpha^{-1/2}}_F&=O_P(\delta_n^{\alpha-1}\sqrt{\log{n}})\label{Eq::ULW-ULBound}.
    \end{align}
\end{lemma}

\begin{proof}
Let $\left\{\phi_1,\phi_2,\ldots,\phi_r\right\}$ be the principal angles between the two subspaces spanned by the columns of $U_\alpha$ and $\hat{U}_\alpha$. Then the singular values of $\Pi^\perp(U_\alpha)\hat{U}_\alpha=(I_n-U_\alpha U_\alpha^\top)\hat U_\alpha$ are given by $\left\{\sin(\phi_1),\ldots,\sin(\phi_r)\right\}$. Thus, by Davis--Kahan (see, for example, \citet{yu2015useful}),
\begin{equation}
    \label{Eq::DK}
    \norm{(I_n-U_\alpha U_\alpha^\top)\hat U_\alpha}=O_P\left(\frac{\norm{E_\alpha}}{\min\left\lvert\Lambda_\alpha\right\rvert}\right)=O_P(\delta_n^{-1/2}\sqrt{\log{n}}).
\end{equation}

We now write
\begin{align*}
    (L_\alpha+E_\alpha)\hat U_\alpha&=\hat U_\alpha\hat\Lambda_\alpha\\
    U_\alpha^\top\hat U_\alpha\hat\Lambda_\alpha-\Lambda_\alpha U_\alpha^\top\hat U_\alpha&=U_\alpha^\top E_\alpha\hat U_\alpha\\
    U_\alpha^\top\hat U_\alpha\hat\Lambda_\alpha-\Lambda_\alpha U_\alpha^\top\hat U_\alpha&=U_\alpha^\top E_\alpha\left(U_\alpha U_\alpha^\top\hat U_\alpha+\Pi^\perp(U_\alpha)\hat U_\alpha\right)\\
    U_\alpha^\top\hat U_\alpha\hat\Lambda_\alpha-\Lambda_\alpha U_\alpha^\top\hat U_\alpha&=U_\alpha^\top E_\alpha U_\alpha U_\alpha^\top\hat U_\alpha+U_\alpha^\top E_\alpha\Pi^\perp(U_\alpha)\hat U_\alpha.
\end{align*}
Thus, by Lemma~\ref{Lm::UTEU} and Equations~\eqref{Eq::EBound} and \eqref{Eq::DK},
\begin{align*}
    \norm{U_\alpha^\top\hat U_\alpha\hat\Lambda_\alpha-\Lambda_\alpha U_\alpha^\top\hat U_\alpha}
    &\leq\norm{U_\alpha^\top E_\alpha U_\alpha}\norm{U_\alpha^\top\hat{U}_\alpha}+\norm{E_\alpha}\norm{\Pi^\perp(U_\alpha)\hat U_\alpha} \\
    &=O_P(\delta_n^{-2\alpha}\log{n}).
\end{align*}

We now show this for a general fixed $q$. We may explicitly calculate
\[
    U_\alpha^\top\hat U_\alpha\hat\Lambda_\alpha^q-\Lambda_\alpha^q U_\alpha^\top\hat U_\alpha=H_q\circ\left(U_\alpha^\top\hat U_\alpha\hat\Lambda_\alpha-\Lambda_\alpha U_\alpha^\top\hat U_\alpha\right),
\]
where $\circ$ denotes the entrywise product and
\[
    (H_q)_{ij}=
    \begin{cases}
    \displaystyle
    \frac{\hat\lambda_{j,\alpha}^{\,q}-\lambda_{i,\alpha}^{\,q}}
         {\hat\lambda_{j,\alpha}-\lambda_{i,\alpha}},
    & \hat\lambda_{j,\alpha}\neq \lambda_{i,\alpha}, \\[2ex]
    q\lambda_{i,\alpha}^{q-1},
    & \hat\lambda_{j,\alpha}= \lambda_{i,\alpha}.
    \end{cases}
\]

Now Weyl's inequality gives
\[
    \max_{1\leq k\leq r}\left\lvert\lambda_k(\hat\Lambda_\alpha)-\lambda_k(\Lambda_\alpha)\right\rvert\leq \norm{E_\alpha}.
\]
By Equations~\eqref{Eq::EBound} and \eqref{Eq::LambdaBound}, it follows that, with high probability,
\[
    \lambda_{\min}(\hat\Lambda_\alpha)\asymp\lambda_{\max}(\hat\Lambda_\alpha)\asymp \delta_n^{1-2\alpha}.
\]

Thus, for each fixed $q$, there exists a constant $C_q>0$ such that
\[
    \max_{i,j}|(H_q)_{ij}|\leq C_q \delta_n^{(q-1)(1-2\alpha)}.
\]
Since $r$ is fixed, this implies
\begin{align}
    \label{Eq::UUlambdaH-LambdaHUUBound}
    \norm{U_\alpha^\top\hat U_\alpha\hat\Lambda_\alpha^q-\Lambda_\alpha^q U_\alpha^\top\hat U_\alpha}=O_P\left(\delta_n^{q(1-2\alpha)-1}\log{n}\right),
\end{align}
which proves the first bound. 

We now show that $\norm{U_\alpha^\top\hat{U}_{\alpha}W_n-I_r}=O_P(\delta_n^{-1}\log{n})$. This follows by definition, because $W_n$ is the matrix that solves the orthogonal Procrustes problem between $\hat{U}_\alpha$ and $U_\alpha$. It is easy to see that $\norm{U_\alpha^\top\hat{U}_{\alpha}W_n-I_r}=\max_i\left\lvert\cos(\phi_i)-1\right\rvert$. Since $1-\cos(x)\leq\sin^2(x)$, by Equation~\eqref{Eq::DK}, we have
\begin{equation}
    \label{Eq::UUW-IBound}
    \norm{U_\alpha^\top\hat{U}_{\alpha}W_n-I_r}=O_P(\delta_n^{-1}\log{n}).
\end{equation}

We now show Equation~\eqref{Eq::WLW-LBound}. Since $W_n$ is orthogonal, we may multiply the bound in Equation~\eqref{Eq::UUlambdaH-LambdaHUUBound} by $W_n$ on the right to get 
\begin{align}
    \label{Eq::UUlambdaQW-LambdaQUUBoundW}
    \norm{U_\alpha^\top\hat U_\alpha W_n W_n^\top\hat\Lambda_\alpha^qW_n-\Lambda_\alpha^q U_\alpha^\top\hat U_\alpha W_n}=O_P\left(\delta_n^{q(1-2\alpha)-1}\log{n}\right).
\end{align}
We then have 
\begin{align*}
    W_n^\top\hat\Lambda_\alpha^q W_n-\Lambda_\alpha^q&=\left[U_\alpha^\top\hat U_\alpha W_n\left(W_n^\top\hat\Lambda_\alpha^qW_n\right)-\Lambda_\alpha^q U_\alpha^\top\hat U_\alpha W_n\right]+\left[I_r-U_\alpha^\top\hat U_\alpha W_n\right]\left(W_n^\top\hat\Lambda_\alpha^qW_n\right) \\
    &\quad+\Lambda_\alpha^q\left[U_\alpha^\top\hat U_\alpha W_n-I_r\right].
\end{align*}

The second and third terms are $O_P(\delta_n^{q(1-2\alpha)-1}\log{n})$ by Equations~\eqref{Eq::UUW-IBound} and \eqref{Eq::LambdaBound}. Combining this with Equation~\eqref{Eq::UUlambdaQW-LambdaQUUBoundW} proves Equation~\eqref{Eq::WLW-LBound}. Taking $q=1/2$, the bound on the second term is exactly Equation~\eqref{Eq::UUW-IWLWBound}.

Finally, for Equation~\eqref{Eq::ULW-ULBound}, using $W_nW_n^\top=I_r$, we write
\begin{align*}
    \hat U_\alpha\hat\Lambda_\alpha^{-1/2}W_n-U_\alpha\Lambda_\alpha^{-1/2} &=\hat U_\alpha W_n\left(W_n^\top\hat\Lambda_\alpha^{-1/2}W_n\right)-U_\alpha\Lambda_\alpha^{-1/2} \\
    &=(\hat U_\alpha W_n-U_\alpha)\Lambda_\alpha^{-1/2}+\hat U_\alpha W_n\left(W_n^\top\hat\Lambda_\alpha^{-1/2}W_n-\Lambda_\alpha^{-1/2}\right).
\end{align*}
Thus,
\begin{align*}
    \norm{\hat U_\alpha\hat\Lambda_\alpha^{-1/2}W_n-U_\alpha\Lambda_\alpha^{-1/2}}_F&\leq\norm{\hat U_\alpha W_n-U_\alpha}_F\norm{\Lambda_\alpha^{-1/2}}+\norm{\hat U_\alpha W_n}\norm{W_n^\top\hat\Lambda_\alpha^{-1/2}W_n-\Lambda_\alpha^{-1/2}}_F.
\end{align*}
By Equations~\eqref{Eq::DK} and \eqref{Eq::LambdaBound}, the first term is $O_P(\delta_n^{\alpha-1}\sqrt{\log{n}})$. By Equation~\eqref{Eq::WLW-LBound}, the second term is $O_P(\delta_n^{\alpha-1}\sqrt{\log{n}})$. Thus,
\[
    \norm{\hat U_\alpha\hat\Lambda_\alpha^{-1/2}W_n-U_\alpha\Lambda_\alpha^{-1/2}}_F=O_P(\delta_n^{\alpha-1}\sqrt{\log{n}}),
\]
which completes the proof.
\end{proof}

\begin{lemma}
    \label{Lem::FrobSpectralRemainder}
    Under Condition~\ref{Assump::Standing},
    \[
        \norm{\mathcal{R}_{\mathrm{sp},\alpha}}_F=O_P\left(\delta_n^{-\alpha-1/2}\log{n}\right).
    \]
\end{lemma}

\begin{proof}
    We recall that 
    \begin{align*}
        \mathcal{R}_{\mathrm{sp},\alpha}&=U_\alpha\left(W_n^\top\hat{\Lambda}_\alpha^{1/2}W_n-\Lambda_\alpha^{1/2}-U_\alpha^\top E_\alpha U_\alpha \Lambda_\alpha^{-1/2}+(U_\alpha^\top\hat U_\alpha W_n-I_r)W_n^\top\hat{\Lambda}_\alpha^{1/2}W_n\right)\\
        &\qquad+\Pi^\perp(U_\alpha)E_\alpha\left(\hat{U}_\alpha\hat{\Lambda}_\alpha^{-1/2}W_n-U_\alpha\Lambda_\alpha^{-1/2}\right).
    \end{align*}

    The first term is $O_P(\delta_n^{-\alpha-1/2}\log{n})$ by Lemma~\ref{lem:signal} (with $q=1/2$), Lemma~\ref{Lm::UTEU} and Equation~\eqref{Eq::LambdaBound}, and Equation~\eqref{Eq::UUW-IWLWBound} of Lemma~\ref{lem:signal}. The second term is $O_P(\delta_n^{-\alpha-1/2}\log{n})$ by Equation~\eqref{Eq::EBound} and Equation~\eqref{Eq::ULW-ULBound} of Lemma~\ref{lem:signal}. Combining these two bounds completes the proof. 
\end{proof}

\begin{proposition}
    \label{Prop::SpecBounds}
    Under Condition~\ref{Assump::Standing}, for each fixed row index $i$,
    \[
        n^{1/2}\delta_n^\alpha\norm{\left(\mathcal{R}_{\mathrm{sp},\alpha}\right)_{i*}}=o_P(1).
    \]
    Consequently, for every $\varepsilon>0$, 
    \[
        \mathbb{P}\left(n^{1/2}\delta_n^\alpha\norm{\left(\mathcal{R}_{\mathrm{sp},\alpha}\right)_{i*}}>\varepsilon\mid\xi_i\right)\overset{p}\to0.
    \]
\end{proposition}

\begin{proof}
By Lemma~\ref{Lem::FrobSpectralRemainder},
\[
    \delta_n^{2\alpha}\norm{\mathcal{R}_{\mathrm{sp},\alpha}}_F^2=O_P\left(\delta_n^{-1}(\log{n})^2\right)=o_P(1),
\]
where the last equality uses Condition~\ref{Assump::Sparsity}.

Because the latent positions are iid and the graph model is invariant under a relabeling of the vertices, and because $W_n$ is the usual alignment between $\hat{U}_\alpha$ and $U_\alpha$, the rows of $\mathcal{R}_{\mathrm{sp},\alpha}$ are exchangeable. Thus, for any $\varepsilon>0$,
\begin{align*}
    \mathbb P\left(n^{1/2}\delta_n^\alpha\norm{(\mathcal{R}_{\mathrm{sp},\alpha})_{i*}}>\varepsilon\right)&=\mathbb E\left[\frac1n\sum_{j=1}^n\mathbf 1\left\{n^{1/2}\delta_n^\alpha\norm{(\mathcal{R}_{\mathrm{sp},\alpha})_{j*}}>\varepsilon\right\}\right].
\end{align*}
For every realization of $\mathcal{R}_{\mathrm{sp},\alpha}$,
\begin{align*}
    \frac1n\sum_{j=1}^n\mathbf 1\left\{n^{1/2}\delta_n^\alpha\norm{(\mathcal{R}_{\mathrm{sp},\alpha})_{j*}}>\varepsilon\right\}&\leq1\wedge\frac{1}{n}\sum_{j=1}^n\frac{n\delta_n^{2\alpha}\norm{(\mathcal{R}_{\mathrm{sp},\alpha})_{j*}}^2}{\varepsilon^2} \\
    &=1\wedge\frac{\delta_n^{2\alpha}\norm{\mathcal{R}_{\mathrm{sp},\alpha}}_F^2}{\varepsilon^2}.
\end{align*}
The right-hand side converges to zero in probability and is bounded by $1$. Therefore its expectation converges to zero. Thus,
\[
    \mathbb P\left(n^{1/2}\delta_n^\alpha\norm{(\mathcal{R}_{\mathrm{sp},\alpha})_{i*}}>\varepsilon\right)\to 0,
\]
which proves the proposition. The conditional statement follows from the fact that 
\[
    \E\left[\mathbb{P}\left(n^{1/2}\delta_n^\alpha\norm{(\mathcal{R}_{\mathrm{sp},\alpha})_{i*}}>\varepsilon\mid\xi_i\right)\right]\to0.
\]
An application of Markov's inequality completes the proof. 
\end{proof}

\subsection{Proof of Theorem~\ref{Thm::Main}}
\label{Sec::TheoremProof}
We begin this section with a proof that three of the terms arising from the expansion of $\hat{X}_\alpha W_nV_\alpha^\top-X_\alpha$ are negligible. For these three lemmas, we note that since
\[
    U_\alpha\Lambda_\alpha^{-1/2}V_\alpha^\top=T^{-\alpha}XG_{n,\alpha}^{-1}
\]
and $V_\alpha$ is orthogonal, for any $n\times n$ matrix $M$ and any row $i$,
\[
    \norm{\left(MU_\alpha\Lambda_\alpha^{-1/2}\right)_{i*}}=\norm{\left(MT^{-\alpha}XG_{n,\alpha}^{-1}\right)_{i*}}.
\]
Intuitively, we may drop the $V_\alpha$ term in the proof for clarity, with the understanding that this orthogonal matrix does not affect the row norms. 
\begin{lemma}
    \label{Lm::TPBULBound}
    For each fixed row $i$,
    \[
        n^{\alpha+1/2}\rho_n^\alpha\norm{\left(T^{-\alpha}P^\circ B_DU_\alpha\Lambda_\alpha^{-1/2}\right)_{i*}}=o_P(1).
    \]
\end{lemma}

\begin{proof}
Since $U_\alpha \Lambda_\alpha^{-1/2}V_\alpha^\top=T^{-\alpha}XG_{n,\alpha}^{-1}$,
we have
\[
    B_DU_\alpha \Lambda_\alpha^{-1/2}V_\alpha^\top=\alpha T^{-\alpha-1}(T-D)T^{-\alpha}XG_{n,\alpha}^{-1}=\alpha T^{-2\alpha-1}(T-D)XG_{n,\alpha}^{-1}.
\]
We thus write
\[
    \left(T^{-\alpha}P^\circ B_DU_\alpha\Lambda_\alpha^{-1/2}V_\alpha^\top\right)_{i*}=\alpha t_i^{-\alpha}\sum_{j\neq i}p_{ij}(t_j-d_j)\left(t_j^{-2\alpha-1}X_j^\top G_{n,\alpha}^{-1}\right).
\]
Define $b_j:=t_j^{-2\alpha-1}X_j^\top G_{n,\alpha}^{-1}$. We begin by bounding the quantity $\mathbb E\norm{\sum_{j\neq i}p_{ij}(t_j-d_j)b_j}^2$.
\begin{align}
    \E\norm{\sum_{j\neq i}p_{ij}(t_j-d_j)b_j}^2&=\mathbb E\left[\sum_{j\neq i}\sum_{\ell\neq i}p_{ij}p_{i\ell}(t_j-d_j)(t_\ell-d_\ell)b_jb_\ell^\top\right]\nonumber\\
    &=\mathbb E\left[\sum_{j\neq i}\sum_{\ell\neq i}p_{ij}p_{i\ell}\mathrm{Cov}\left(t_j-d_j,t_\ell-d_\ell\right)b_jb_\ell^\top\right]\label{Eq::EBBound}
\end{align}

If $j=\ell$, then the covariance term is simply $\mathrm{Var}(t_j-d_j)$. Since $t_j$ is fixed, this simplifies to 
\[
    \sum_{k\neq j}\mathrm{Var}\left(a_{jk}\right)=\sum_{k\neq j}p_{jk}(1-p_{jk})\leq\sum_{k\neq j}p_{jk}=t_j.
\] 
By Equation~\eqref{Eq::TBound}, this is $O(\delta_n)$.

If $j\neq\ell$, then all terms are independent except for the shared edge $(j,\ell)$. Thus, the covariance term becomes 
\[
    \mathrm{Cov}\left(t_j-d_j,t_\ell-d_\ell\right)=\mathrm{Cov}\left(p_{j\ell}-a_{j\ell},p_{\ell j}-a_{\ell j}\right)=\mathrm{Var}(a_{j\ell})\leq p_{j\ell}.
\]
Since $p_{j\ell}=X_j^\top X_{\ell}=\rho_n\xi_j^\top \xi_\ell$ and we assume the latent positions are uniformly bounded, we have 
\[
    \mathrm{Cov}\left(t_j-d_j,t_\ell-d_\ell\right)=O(\rho_n). 
\]

Plugging these bounds into Equation~\eqref{Eq::EBBound} gives
\begin{align}
    \mathbb E\norm{\sum_{j\neq i}p_{ij}(t_j-d_j)b_j}^2&\leq\sum_{j\neq i}p_{ij}^2\mathrm{Var}(t_j-d_j)\norm{b_j}^2+\sum_{j\neq i}\sum_{\ell\neq i,j}p_{ij}p_{i\ell}\mathrm{Cov}\left(t_j-d_j,t_\ell-d_\ell\right)b_j b_\ell^\top\nonumber\\
    &\leq C\delta_n\sum_{j\neq i}p_{ij}^2\norm{b_j}^2+C'\rho_n\sum_{j\neq i}\sum_{\ell\neq i}p_{ij}p_{i\ell}\norm{b_j}\norm{b_\ell}\nonumber\\
    &\leq C\delta_n\sum_{j\neq i}p_{ij}^2\norm{b_j}^2+C'\rho_n\left(\sum_{j\neq i}p_{ij}\norm{b_j}\right)^2\label{Eq::EPT-DBBound}.
\end{align}

We now bound $\max_{j\neq i}\norm{b_j}$. By Equations~\eqref{Eq::TBound}, \eqref{Eq::GBound}, and the fact that the latent positions are uniformly bounded, we have 
\begin{equation}
    \label{Eq::BJBound}
    \max_{j\neq i}\norm{b_j}\leq C\delta_n^{-2\alpha-1}\cdot\rho_n^{1/2}\cdot\delta_n^{2\alpha-1}=C\rho_n^{1/2}\delta_n^{-2}.
\end{equation}
Combining these bounds with Equations~\eqref{Eq::PBound} and \eqref{Eq::EPT-DBBound} gives
\begin{align*}
    \mathbb E\norm{\sum_{j\neq i}p_{ij}(t_j-d_j)b_j}^2
    &\leq C\delta_n\sum_{j\neq i}p_{ij}^2\norm{b_j}^2+C'\rho_n\left(\sum_{j\neq i}p_{ij}\norm{b_j}\right)^2  \\
    &\leq C\delta_n\cdot n\rho_n^2\cdot \rho_n\delta_n^{-4}+C'\rho_n\cdot\left(n\rho_n\cdot\rho_n^{1/2}\delta_n^{-2}\right)^2\\
    &\leq C\rho_n^2\delta_n^{-2}
\end{align*}

Thus, 
\[
    \norm{\sum_{j\neq i}p_{ij}(t_j-d_j)b_j}=O_P(\rho_n\delta_n^{-1}).
\]
Recalling that 
\[
    \left(T^{-\alpha}P^\circ B_DU_\alpha\Lambda_\alpha^{-1/2}V_\alpha^\top\right)_{i*}=\alpha t_i^{-\alpha}\sum_{j\neq i}p_{ij}(t_j-d_j)\left(t_j^{-2\alpha-1}X_j^\top G_{n,\alpha}^{-1}\right),
\]
we have 
\begin{align*}
    \norm{\left(T^{-\alpha}P^\circ B_DU_\alpha\Lambda_\alpha^{-1/2}V_\alpha^\top\right)_{i*}}=O_P(\delta_n^{-\alpha}\cdot\rho_n\delta_n^{-1})=O_P(\rho_n\delta_n^{-\alpha-1}).
\end{align*}
Finally, we have
\[
    n^{\alpha+1/2}\rho_n^\alpha\left\|\left(T^{-\alpha}P^\circ B_DU_\alpha\Lambda_\alpha^{-1/2}\right)_{i*}\right\|=O_P(n^{-1/2})=o_P(1),
\]
which completes the proof. 
\end{proof}

\begin{lemma}
    \label{Lm::BAPTULBound}
    For each fixed row $i$,
    \[
        n^{\alpha+1/2}\rho_n^\alpha\norm{\left(B_D(A-P^\circ)T^{-\alpha}U_\alpha \Lambda_\alpha^{-1/2}\right)_{i*}}=o_P(1).
    \]
\end{lemma}

\begin{proof}
Since $U_\alpha \Lambda_\alpha^{-1/2}V_\alpha^\top=T^{-\alpha}XG_{n,\alpha}^{-1}$, we have $T^{-\alpha}U_\alpha\Lambda_\alpha^{-1/2}V_\alpha^\top=T^{-2\alpha}XG_{n,\alpha}^{-1}$. We thus write
\begin{align*}
    \left(B_DE_AT^{-\alpha}U_\alpha\Lambda_\alpha^{-1/2}V_\alpha^\top\right)_{i*}=\alpha t_i^{-\alpha-1}(t_i-d_i)\sum_{j\neq i}(E_A)_{ij}t_j^{-2\alpha}X_j^\top G_{n,\alpha}^{-1},
\end{align*}
where $(E_A)_{ij}$ satisfies
\[
    \mathbb E[(E_A)_{ij}^2]=p_{ij}(1-p_{ij})\leq p_{ij}.
\]
Additionally, by Equations~\eqref{Eq::TBound}, \eqref{Eq::GBound}, and the fact that the latent positions are uniformly bounded, we have
\[
    \norm{t_j^{-2\alpha}X_j^\top G_{n,\alpha}^{-1}}\leq C\delta_n^{-2\alpha}\rho_n^{1/2}\delta_n^{2\alpha-1}=C\rho_n^{1/2}\delta_n^{-1}.
\]
Thus,
\begin{align*}
    \mathbb E\norm{\sum_{j\neq i}(E_A)_{ij}t_j^{-2\alpha}X_j^\top G_{n,\alpha}^{-1}}^2&\leq\sum_{j\neq i}p_{ij}\norm{t_j^{-2\alpha}X_j^\top G_{n,\alpha}^{-1}}^2 \\
    &\leq
    C\rho_n\delta_n^{-2}
    \sum_{j\neq i}p_{ij} \\
    &=
    C\rho_n\delta_n^{-2}t_i \\
    &\leq
    C\rho_n\delta_n^{-1}.
\end{align*}

Combining this with Equations~\eqref{Eq::TBound} and \eqref{Eq::D-TBound} gives 
\[
    \norm{\left[B_DE_AT^{-\alpha}U_\alpha\Lambda_\alpha^{-1/2}\right]_{i*}}=O_P\left(\rho_n^{1/2}\delta_n^{-\alpha-1}\right).
\]
Finally,
\begin{align*}
    n^{\alpha+1/2}\rho_n^\alpha\norm{\left(B_DE_AT^{-\alpha}U_\alpha\Lambda_\alpha^{-1/2}\right)_{i*}}&=O_P\left(n^{\alpha+1/2}\rho_n^\alpha\cdot\rho_n^{1/2}\delta_n^{-\alpha-1}\right) \\
    &=O_P\left(\delta_n^{\alpha+1/2}\delta_n^{-\alpha-1}\right) \\
    &=O_P(\delta_n^{-1/2})=o_P(1),
\end{align*}
which completes the proof.
\end{proof}

\begin{lemma}
    \label{Lm::TAPBULBound}
    For each fixed row $i$, 
    \begin{align*}
        n^{\alpha+1/2}\rho_n^\alpha\left\lVert\left(T^{-\alpha}(A-P^\circ)B_DU_\alpha \Lambda_\alpha^{-1/2}\right)_{i*}\right\rVert&=o_P(1).
    \end{align*}
\end{lemma}

\begin{proof}
Since $U_\alpha \Lambda_\alpha^{-1/2}V_\alpha^\top=T^{-\alpha}XG_{n,\alpha}^{-1}$, we have
\[
    B_DU_\alpha \Lambda_\alpha^{-1/2}V_\alpha^\top=\alpha T^{-\alpha-1}(T-D)T^{-\alpha}XG_{n,\alpha}^{-1}=\alpha T^{-2\alpha-1}(T-D)XG_{n,\alpha}^{-1}.
\]
We thus write
\begin{align*}
    \left(T^{-\alpha}(A-P^\circ)B_DU_\alpha \Lambda_\alpha^{-1/2}V_\alpha^\top\right)_{i*} &=\alpha t_i^{-\alpha}\sum_{j\neq i}(a_{ij}-p_{ij})t_j^{-2\alpha-1}(t_j-d_j)X_j^\top G_{n,\alpha}^{-1}\\
    &=\alpha t_i^{-\alpha}\sum_{j\neq i}(E_A)_{ij}(t_j-d_j)b_j,
\end{align*}
where $(E_A)_{ij}=a_{ij}-p_{ij}$ and $b_j=t_j^{-2\alpha-1}X_j^\top G_{n,\alpha}^{-1}$. 

For fixed $i$ and $j\neq i$, define
\[
    D_j^{(-i)}=\sum_{k\neq i,j}(a_{jk}-p_{jk}),
\]
so that 
\[
    t_j-d_j= -(E_A)_{ij}-D_j^{(-i)}.
\]
Thus,
\[
    \sum_{j\neq i}(E_A)_{ij}(t_j-d_j)b_j=-\sum_{j\neq i}(E_A)_{ij}^2b_j-\sum_{j\neq i}(E_A)_{ij}D_j^{(-i)}b_j.
\]

We bound the two terms separately. For the first term,
\begin{equation}
    \label{Eq::EAsquaredbound}
    \mathbb E[(E_A)_{ij}^2]=p_{ij}(1-p_{ij})\leq p_{ij}.
\end{equation}
Recalling Equation~\eqref{Eq::BJBound}, which states $\max_{j\neq i}\norm{b_j}\leq C\rho_n^{1/2}\delta_n^{-2}$, we have
\begin{align*}
    \mathbb E\left\lVert\sum_{j\neq i}(E_A)_{ij}^2b_j\right\rVert&\leq\sum_{j\neq i}\mathbb E[(E_A)_{ij}^2]\norm{b_j}\\
    &\leq\sum_{j\neq i}p_{ij}\|b_j\| \\
    &\leq\left(\max_{j\neq i}\|b_j\|\right) \sum_{j\neq i}p_{ij}\\
    &\leq C\rho_n^{1/2}\delta_n^{-2}\cdot \delta_n=C\rho_n^{1/2}\delta_n^{-1}.
\end{align*}
Thus,
\begin{equation}
    \label{Eq::H1Bound}
    \sum_{j\neq i}(E_A)_{ij}^2b_j=O_P\left(\rho_n^{1/2}\delta_n^{-1}\right).
\end{equation}

For the second term, we have
\[
    \mathbb E\left\lVert\sum_{j\neq i}(E_A)_{ij}D_j^{(-i)}b_j\right\rVert^2
    =\sum_{j,\ell\neq i}b_jb_\ell^\top\mathbb E\left[(E_A)_{ij}(E_A)_{i\ell} D_j^{(-i)}D_\ell^{(-i)}\right].
\]
If $j\neq\ell$, then $((E_A)_{ij},(E_A)_{i\ell})$ is independent of $(D_j^{(-i)},D_\ell^{(-i)})$. Since $(E_A)_{ij}$ and $(E_A)_{i\ell}$ are independent and mean-zero, we have 
\begin{align}
    \label{Eq::BZDBound}
    \mathbb E\norm{\sum_{j\neq i}(E_A)_{ij}D_j^{(-i)}b_j}^2&=\sum_{j\neq i}\norm{b_j}^2\mathbb E\left[(E_A)_{ij}^2(D_j^{(-i)})^2\right].
\end{align}

By Equation~\eqref{Eq::EAsquaredbound}, $\mathbb E[(E_A)_{ij}^2]\leq C\rho_n$. Additionally,
\[
    \mathbb E[(D_j^{(-i)})^2]=\sum_{k\neq i,j}p_{jk}(1-p_{jk})\leq t_j\leq C\delta_n.
\]
Thus,
\[
    \mathbb E[(E_A)_{ij}^2(D_j^{(-i)})^2]\leq C\rho_n\delta_n.
\]

Plugging this into Equation~\eqref{Eq::BZDBound} and using Equation~\eqref{Eq::BJBound} gives 
\[
    \mathbb E\left\lVert\sum_{j\neq i}(E_A)_{ij}D_j^{(-i)}b_j\right\rVert^2\leq Cn\rho_n\delta_n^{-4}\cdot\rho_n\delta_n=C\rho_n\delta_n^{-2}.
\]
Thus,
\begin{equation}
    \label{Eq::H2Bound}
    \sum_{j\neq i}(E_A)_{ij}D_j^{(-i)}b_j=O_P\left(\rho_n^{1/2}\delta_n^{-1}\right).
\end{equation}
Combining Equations~\eqref{Eq::H1Bound} and \eqref{Eq::H2Bound} gives
\[
    \sum_{j\neq i}(E_A)_{ij}(t_j-d_j)b_j=O_P\left(\rho_n^{1/2}\delta_n^{-1}\right).
\]

Recalling that 
\begin{align*}
    \left(T^{-\alpha}(A-P^\circ)B_DU_\alpha \Lambda_\alpha^{-1/2}V_\alpha^\top\right)_{i*}=\alpha t_i^{-\alpha}\sum_{j\neq i}(E_A)_{ij}(t_j-d_j)b_j
\end{align*}
and using Equation~\eqref{Eq::TBound}, we have
\[
    \norm{\left(T^{-\alpha}(A-P^\circ)B_DU_\alpha \Lambda_\alpha^{-1/2}\right)_{i*}}=O_P\left(\delta_n^{-\alpha}\rho_n^{1/2}\delta_n^{-1}\right)=O_P\left(\rho_n^{1/2}\delta_n^{-\alpha-1}\right).
\]

Finally,
\begin{align*}
    n^{\alpha+1/2}\rho_n^\alpha\left\|\left(T^{-\alpha}(A-P^\circ)B_DU_\alpha \Lambda_\alpha^{-1/2}\right)_{i*}\right\|&=O_P\left(n^{\alpha+1/2}\rho_n^\alpha\rho_n^{1/2}\delta_n^{-\alpha-1}\right)\\
    & =O_P\left(\delta_n^{\alpha+1/2}\delta_n^{-\alpha-1}\right) \\
    &=O_P\left(\delta_n^{-1/2}\right)=o_P(1),
\end{align*}
which completes the proof.
\end{proof}

We are now ready to prove Theorem~\ref{Thm::Main}.
\begin{proof}
We begin by writing 
\[
    \hat{X}_\alpha Q_n-X_\alpha=\hat{U}_\alpha\hat{\Lambda}_\alpha^{1/2} W_nV_\alpha^\top-U_\alpha\Lambda_\alpha^{1/2}V_\alpha^\top.
\]
Again, since $V_\alpha$ is orthogonal, the extra factor preserves row norms. Just as in the previous lemmas, for clarity, we omit this factor, with the understanding that we multiply by this matrix at the end to obtain the final result. By Lemma~\ref{Lm::SpectralExpansion},
\[
    \hat{U}_\alpha\hat{\Lambda}_\alpha^{1/2}W_n-U_\alpha\Lambda_\alpha^{1/2}=E_\alpha U_\alpha\Lambda_\alpha^{-1/2}+\mathcal{R}_{\mathrm{sp},\alpha},
\]
where
\begin{align*}
    \mathcal{R}_{\mathrm{sp},\alpha}&=U_\alpha\left(W_n^\top\hat{\Lambda}_\alpha^{1/2}W_n-\Lambda_\alpha^{1/2}-U_\alpha^\top E_\alpha U_\alpha \Lambda_\alpha^{-1/2}+(U_\alpha^\top\hat U_\alpha W_n-I_r)W_n^\top\hat{\Lambda}_\alpha^{1/2}W_n\right)\\
    &\qquad+\Pi^\perp(U_\alpha) E_\alpha\left(\hat{U}_\alpha\hat{\Lambda}_\alpha^{-1/2}W_n-U_\alpha\Lambda_\alpha^{-1/2}\right).
\end{align*}
By Proposition~\ref{Prop::SpecBounds}, for every $\varepsilon>0$,
\[
    \mathbb P\left(n^{\alpha+1/2}\rho_n^\alpha\norm{\left(\mathcal{R}_{\mathrm{sp},\alpha}\right)_{i*}}>\varepsilon\,\middle|\, \xi_i\right)\overset{p}\longrightarrow 0.
\] 
Thus,
\[
    n^{\alpha+1/2}\rho_n^\alpha\left[\left(\hat U_\alpha\hat \Lambda_\alpha^{1/2}W_n\right)_{i*}-\left(U_\alpha \Lambda_\alpha^{1/2}\right)_{i*}\right]=n^{\alpha+1/2}\rho_n^\alpha\left(E_\alpha U_\alpha \Lambda_\alpha^{-1/2}\right)_{i*}+r_{n,i},
\]
where $\mathbb{P}(\norm{r_{n,i}}>\varepsilon\mid\xi_i)\overset{p}\to0$ for every $\varepsilon>0$.

We expand $E_\alpha$ using a Taylor expansion to get
\begin{align*}
    E_\alpha&=D^{-\alpha}AD^{-\alpha}-T^{-\alpha}PT^{-\alpha} \\
    &=\left(D^{-\alpha}AD^{-\alpha}-T^{-\alpha}AT^{-\alpha}\right)
    +T^{-\alpha}(A-P^\circ)T^{-\alpha}
    -T^{-\alpha}P_{\diag}T^{-\alpha} \\
    &=B_DAT^{-\alpha}+T^{-\alpha}AB_D+T^{-\alpha}(A-P^\circ)T^{-\alpha}-T^{-\alpha}P_{\diag}T^{-\alpha}+\mathcal R_{\mathrm{Tay},\alpha},
    \end{align*}
where
    \begin{align*}
        \mathcal R_{\mathrm{Tay},\alpha}&=B_DAB_D+R_D A T^{-\alpha}+T^{-\alpha}A R_D+R_D A B_D+B_D A R_D+R_D A R_D.
    \end{align*}
Using $A=P^\circ+(A-P^\circ)$, we get
    \begin{align*}
        E_\alpha&=T^{-\alpha}(A-P^\circ)T^{-\alpha}+B_D P^\circ T^{-\alpha}+T^{-\alpha}P^\circ B_D \\
        &\qquad+B_D(A-P^\circ)T^{-\alpha}+T^{-\alpha}(A-P^\circ)B_D-T^{-\alpha}P_{\diag}T^{-\alpha}+\mathcal R_{\mathrm{Tay},\alpha}.
    \end{align*}
Finally, we multiply on the right by $U_\alpha\Lambda_\alpha^{-1/2}$ to get 
\begin{equation}
    \label{Eq::FinalTermsWithRemainder}
    E_\alpha U_\alpha\Lambda_\alpha^{-1/2}=T^{-\alpha}(A-P^\circ)T^{-\alpha}U_\alpha\Lambda_\alpha^{-1/2}+B_D P^\circ T^{-\alpha}U_\alpha\Lambda_\alpha^{-1/2}+\tilde{\mathcal R}_{\alpha,n},
\end{equation}
where
\begin{align*}
    \tilde{\mathcal R}_{\alpha,n}&=T^{-\alpha}P^\circ B_DU_\alpha\Lambda_\alpha^{-1/2}+B_D(A-P^\circ)T^{-\alpha}U_\alpha\Lambda_\alpha^{-1/2}+T^{-\alpha}(A-P^\circ)B_DU_\alpha\Lambda_\alpha^{-1/2} \\
    &\qquad-T^{-\alpha}P_{\diag}T^{-\alpha}U_\alpha\Lambda_\alpha^{-1/2}+\mathcal R_{\mathrm{Tay},\alpha}U_\alpha\Lambda_\alpha^{-1/2}.
\end{align*}
We now show that $n^{\alpha+1/2}\rho_n^\alpha\norm{\left(\tilde{\mathcal R}_{\alpha,n}\right)_{i*}}=o_P(1)$ by considering each term separately. The first three terms are handled by Lemmas~\ref{Lm::TPBULBound},~\ref{Lm::BAPTULBound} and~\ref{Lm::TAPBULBound}, respectively.  Since $A$ is symmetric and, by Equation~\eqref{Eq::DBound},
\[
    \max_i\sum_j|a_{ij}|=\max_i d_i=O_P(\delta_n),
\]
Lemma~\ref{Lm::TaylorBookkeeping} applies with $H=A$ and $V=U_\alpha\Lambda_\alpha^{-1/2}$. By Equations~\eqref{Eq::URowBound} and \eqref{Eq::LambdaBound}, we have $\norm{V}_{2\to\infty}=O_P(n^{-1/2}\delta_n^{\alpha-1/2})$. Thus, 
\begin{align*}
    n^{\alpha+1/2}\rho_n^{\alpha}\norm{\left(\mathcal R_{\mathrm{Tay},\alpha}U_\alpha\Lambda_\alpha^{-1/2}\right)_{i*}}&=O_P(n^{\alpha+1/2}\rho_n^{\alpha}\cdot\delta_n^{-2\alpha}\log{n}\cdot n^{-1/2}\delta_n^{\alpha-1/2})\\
    &=O_P(\delta_n^{-1/2}\log{n})=o_P(1).
\end{align*}

The diagonal term is easily bounded using Equations~\eqref{Eq::TBound}, \eqref{Eq::PBound}, \eqref{Eq::URowBound} and \eqref{Eq::LambdaBound}. The same bounds, together with Equation~\eqref{Eq::BDBound} show that the additional term $B_DP_{\diag}\T U_\alpha\Lambda_\alpha^{-1/2}$ is also negligible after scaling.

Using the identity
\[
    U_\alpha\Lambda_\alpha^{-1/2}V_\alpha^\top=T^{-\alpha}XG_{n,\alpha}^{-1},
\]
together with the preceding fixed-row remainder bounds, we obtain, for each fixed row $i$,
\begin{align*}
    \left(E_\alpha U_\alpha\Lambda_\alpha^{-1/2}V_\alpha^\top\right)_{i*}^\top&=t_i^{-\alpha}G_{n,\alpha}^{-1}\sum_{j\neq i}(a_{ij}-p_{ij})t_j^{-2\alpha}X_j\\
    &\quad+\alpha t_i^{-\alpha-1}(t_i-d_i)X_i
    +o_P\left(n^{-\alpha-1/2}\rho_n^{-\alpha}\right).
\end{align*}
Then
\begin{align*}
    n^{\alpha+1/2}\rho_n^\alpha&\left(E_\alpha U_\alpha \Lambda_\alpha^{-1/2}V_\alpha^\top\right)_{i*}^\top\\&=\left(\frac{t_i}{\delta_n}\right)^{-\alpha}\left(\delta_n^{1-2\alpha}G_{n,\alpha}^{-1}\right)\sum_{j\neq i}\frac{a_{ij}-p_{ij}}{\sqrt{\delta_n}}\left(\frac{t_j}{\delta_n}\right)^{-2\alpha}\xi_j \\
    &\qquad-\alpha\left(\frac{t_i}{\delta_n}\right)^{-\alpha-1}\xi_i\sum_{j\neq i}\frac{a_{ij}-p_{ij}}{\sqrt{\delta_n}}+o_P(1)\\
    &=\left(\frac{t_i}{\delta_n}\right)^{-\alpha}\left(\delta_n^{1-2\alpha}G_{n,\alpha}^{-1}\right)\sum_{j\neq i}\frac{a_{ij}-p_{ij}}{\sqrt{\delta_n}}\left[\left(\frac{t_j}{\delta_n}\right)^{-2\alpha}\xi_j-\frac{\alpha\delta_n^{2\alpha-1}G_{n,\alpha}}{\frac{t_i}{\delta_n}}\xi_i\right]\\&\quad+o_P(1).
\end{align*}

Conditional on $\xi_i$, we have $\frac{t_i}{\delta_n}\overset{p}{\to}\langle \xi_i,\mu\rangle$ and $\delta_n^{2\alpha-1}G_{n,\alpha}\overset{p}{\to}\Upsilon_\alpha$. These convergences follow from Equation~\eqref{Eq::TDeltaRatio}, the law of large numbers, and the argument used to prove Equation~\eqref{Eq::GBound}.

It only remains to justify replacing $\left(t_j/\delta_n\right)^{-2\alpha}$ inside the sum. By Equation~\eqref{Eq::TDeltaRatio} and Conditions~\ref{Assump::BoundedSupport} and~\ref{Assump::Nonzero}, the error caused by this replacement is $o_P(1)$ uniformly in $j$. Conditional on the latent positions, the resulting difference is a centered sum whose conditional variance is bounded by $C(t_i/\delta_n)$ times the square of this error, and is therefore $o_P(1)$.

Slutsky's theorem then gives
\begin{align*}
    &n^{\alpha+1/2}\rho_n^\alpha\left[\left(\hat{X}_\alpha Q_n\right)_{i*}-\left(X_\alpha\right)_{i*}\right]^\top\\
    &\quad=\frac{1}{\langle \xi_i,\mu\rangle^\alpha}\Upsilon_\alpha^{-1}\sum_{j\neq i}\frac{a_{ij}-p_{ij}}{\sqrt{\delta_n}}\left(\frac{\xi_j}{\langle \xi_j,\mu\rangle^{2\alpha}}-\frac{\alpha \Upsilon_\alpha \xi_i}{\langle \xi_i,\mu\rangle}\right)+o_P(1).
\end{align*}

Now consider the sum 
\[
    \sum_{j\neq i}\frac{a_{ij}-p_{ij}}{\sqrt{\delta_n}}\left(\frac{\xi_j}{\langle \xi_j,\mu\rangle^{2\alpha}}-\frac{\alpha \Upsilon_\alpha \xi_i}{\langle \xi_i,\mu\rangle}\right)=\frac{1}{\sqrt n}\sum_{j\neq i}\frac{a_{ij}-p_{ij}}{\sqrt{\rho_n}}\left(\frac{\xi_j}{\langle \xi_j,\mu\rangle^{2\alpha}}-\frac{\alpha \Upsilon_\alpha \xi_i}{\langle \xi_i,\mu\rangle}\right).
\]
Conditional on $\xi_i$, the summands are independent, mean-zero random vectors. By Conditions~\ref{Assump::BoundedSupport} and~\ref{Assump::Nonzero}, the vector in parentheses is uniformly bounded. Thus, each summand has norm bounded by $C/\sqrt{n\rho_n}\to0$, and the conditional Lindeberg condition holds. Additionally, $p_{ij}=\rho_n \langle \xi_i,\xi_j\rangle$. Writing $\xi'$ for an independent copy of a latent position from distribution $F$, the conditional covariance is 
\begin{align*}
    &\operatorname{Cov}
    \left(\frac{a_{ij}-p_{ij}}{\sqrt{\rho_n}}\left(\frac{\xi_j}{\langle \xi_j,\mu\rangle^{2\alpha}}-\frac{\alpha \Upsilon_\alpha \xi_i}{\langle \xi_i,\mu\rangle}\right)\right) \\
    &\quad =\mathbb E\left[\langle \xi_i,\xi'\rangle\left(1-\rho_n\langle \xi_i,\xi'\rangle\right)\left(\frac{\xi'}{\langle \xi',\mu\rangle^{2\alpha}}-\frac{\alpha \Upsilon_\alpha \xi_i}{\langle \xi_i,\mu\rangle}\right)\left(\frac{\xi'}{\langle \xi',\mu\rangle^{2\alpha}}-\frac{\alpha \Upsilon_\alpha \xi_i}{\langle \xi_i,\mu\rangle}\right)^\top\right].
\end{align*}

The conditional covariance of the sum thus converges to $\Gamma_{\rho,\alpha}(\xi_i)$. Consequently, an application of the Lindeberg--Feller central limit theorem gives 
\[
    \sum_{j\neq i}\frac{a_{ij}-p_{ij}}{\sqrt{\delta_n}}\left(\frac{\xi_j}{\langle \xi_j,\mu\rangle^{2\alpha}}-\frac{\alpha \Upsilon_\alpha \xi_i}{\langle \xi_i,\mu\rangle}\right)\overset{d}{\to}N\left(0,\Gamma_{\rho,\alpha}(\xi_i)\right).
\]
Another application of Slutsky's theorem gives
\[
    n^{\alpha+1/2}\rho_n^\alpha\left[\left(\hat X_\alpha Q_n\right)_{i*}-\left(X_\alpha\right)_{i*}\right]^\top\overset{d}{\to}N\left(0,\frac{\Upsilon_\alpha^{-1}\Gamma_{\rho,\alpha}(\xi_i)\Upsilon_\alpha^{-1}}{\langle \xi_i,\mu\rangle^{2\alpha}}\right),
\]
which completes the proof. 
\end{proof}

\section{Additional Calculations for Block Models}
\label{Sec::PropertiesofDiagnostic}
\subsection{Numerical Computation}
\label{Sec::AdditionalNumericalDetails}
In this section, we describe how the quantity $\alpha^*$ in Equation~\eqref{Eq::AlphaStar} is computed. We evaluate the diagnostic on a grid of $100$ equally spaced values of $\alpha$ in $[0,1]$ and approximate the solution to Equation~\eqref{Eq::AlphaStar} by the grid point at which $\mathcal E_\alpha^{(n)}$ is smallest. For each such $\alpha$, the population centers $m_{\alpha,k}$ and covariances $\Sigma_{\alpha,k}$ are given in closed form by Corollary~\ref{Cor::DCS}. The projected law $\mathcal{Q}_{k,\alpha}^{(n)}$ also has a closed-form density on the circle, obtained by integrating the distribution $N\!\left(m_{\alpha,k},\tfrac1n\Sigma_{\alpha,k}\right)$ along each ray through the origin. We therefore compute $\mathcal E_\alpha^{(n)}$ by locating the decision boundaries between the two projected densities weighted by their community proportions and integrating. As $\alpha\to1$, the covariance $\Sigma_{\alpha,k}$ can become numerically singular, so we add a small ridge $\varepsilon I_2$ with $\varepsilon=10^{-6}$ before inverting the covariance matrix. The selected $\alpha^*$ is insensitive to $\varepsilon$ across several orders of magnitude.

As described in Section~\ref{Sec::NumEval}, we set $n=200$ in these computations. To confirm that this is a sensible choice, we recompute $\alpha^*$ over the sample sizes $n\in\{50,100,200,400\}$ for a coarse collection of twelve $(p,q)$ configurations with $p>q$, fixing $\pi_1=0.7$. Figure~\ref{Fig::AlphaStableInN} plots $\alpha^*$ against $n$ for each $(p,q)$ pair. In every configuration, the value is relatively stable between $n=200$ and $n=400$, justifying its choice. 

\begin{figure}[t]
    \centering
    \includegraphics[width=0.8\textwidth]{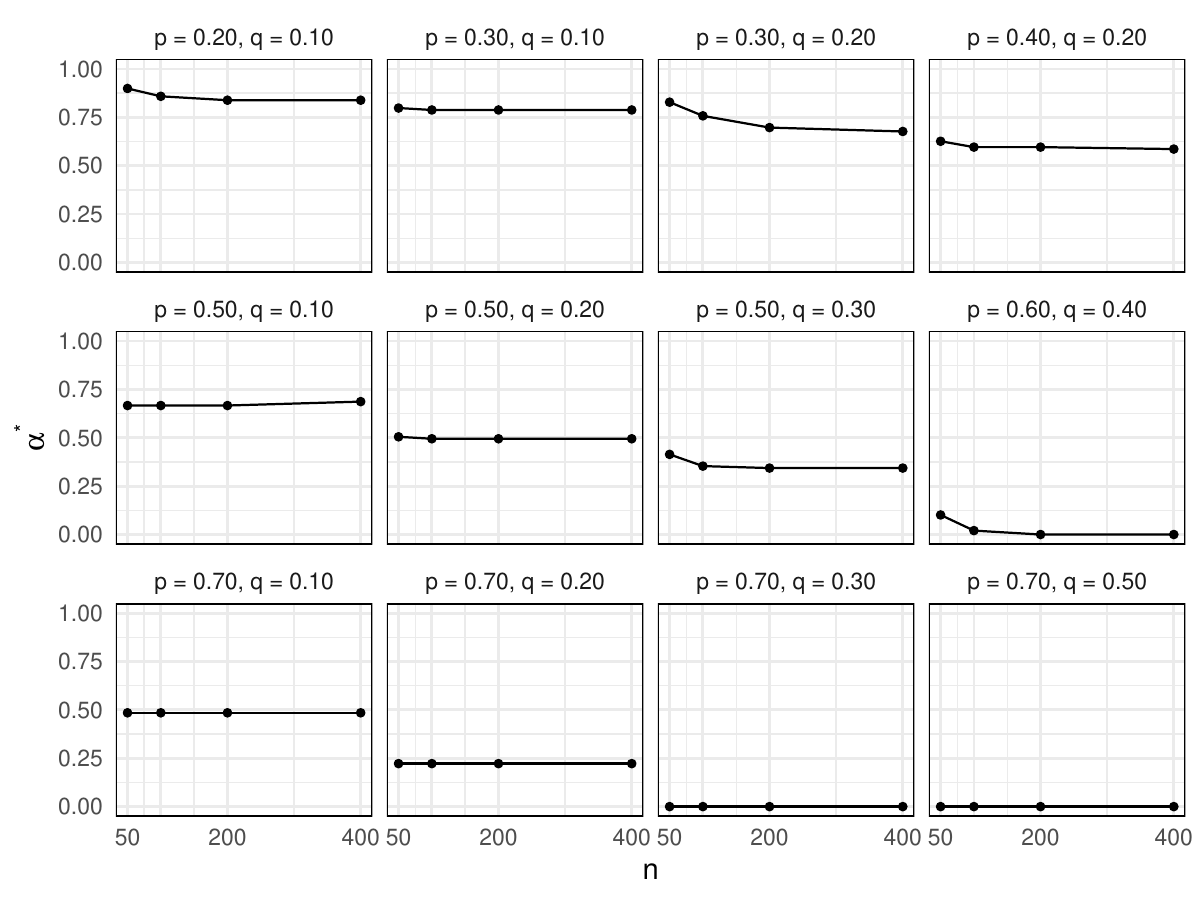}
    \caption{Value of $\alpha^*$ across $n\in\{50,100,200,400\}$ for selected block matrices with $\pi_1=0.7$. Although the minimizer can vary at smaller $n$, it is stable between $n=200$ and $n=400$ for the parameter settings considered.}
    \label{Fig::AlphaStableInN}
\end{figure}

\subsection{Calculations for Two-Community Models}
\label{Sec::SymmetricTwoCommunityCalculations}
In this section, we derive two analytic results for the balanced symmetric two-community model. The first gives the explicit expression for $\mathcal{F}_\alpha$ in Section~\ref{Sec::BlockModels}, while the second derives the Chernoff-information benchmark used below. Both calculations use the following common RDPG representation. The model is
\[
    \pi_1=\pi_2=\frac12,\qquad B=\begin{bmatrix}p&q\\q&p\end{bmatrix},\qquad p>q>0,
\]
and we write
\[
    s=\frac{p+q}{2},\qquad h=\frac{p-q}{2}, \qquad\kappa=\frac{h}{s}=\frac{p-q}{p+q}.
\]
A convenient RDPG representation is
\[
    \nu_1=\begin{bmatrix}\sqrt{s}\\\sqrt{h}\end{bmatrix},
    \qquad
    \nu_2=\begin{bmatrix}\sqrt{s}\\-\sqrt{h}\end{bmatrix},
\]
so that 
\[
    \nu_1^\top\nu_1=\nu_2^\top\nu_2=p,\qquad\nu_1^\top\nu_2=q.
\]

\subsubsection{Derivation of the Variance Ratio}
\label{Sec::FlatteningRatioDerivation}
We work in the sparse regime $\rho_n\to0$. Our goal is to derive the closed-form expression for $\mathcal F_\alpha$ given in Section~\ref{Sec::BlockModels}. We show that, using the representation chosen in the preceding subsection, finding the normal and tangential components of the covariance matrix reduces to finding the diagonal entries of $\Sigma_{\alpha}(\vartheta,1)$. We then compute the corresponding entries. By symmetry between the two communities, it is sufficient to consider a node from community $1$.

Let $\bar\theta=\mathbb E[\Theta]$. By symmetry,
\[
    \mu=\frac{\bar\theta}{2}(\nu_1+\nu_2)=\bar\theta\begin{bmatrix}\sqrt{s}\\0\end{bmatrix},\qquad\gamma:=\gamma_1=\gamma_2=\langle\nu_k,\mu\rangle=\bar\theta s.
\]
Define the two moments
\[
    a_\alpha=\mathbb E[\Theta^{2-2\alpha}],
    \qquad
    b_\alpha=\mathbb E[\Theta^{3-4\alpha}].
\]

We then have
\[
    \tilde\Upsilon_\alpha=\frac{a_\alpha}{\gamma^{2\alpha}}\begin{bmatrix}s&0\\0&h\end{bmatrix}.
\]

Recall that $h_\alpha^\top\nu_k=\gamma_k^\alpha$ for $k=1,2$. Then
\[
    h_\alpha=\frac{\gamma^\alpha}{\sqrt{s}}\begin{bmatrix}1\\0\end{bmatrix},
    \qquad
    u_\alpha=\begin{bmatrix}1\\0\end{bmatrix},
    \qquad
    \Pi_\alpha=\begin{bmatrix}0&0\\0&1\end{bmatrix}.
\]
Recalling that
\[
    \sigma^2_{N,\alpha}(\vartheta)=u_\alpha^\top\Sigma_\alpha(\vartheta)u_\alpha\qquad\text{and}\qquad\sigma^2_{T,\alpha}(\vartheta)=\operatorname{tr}\!\left(\Pi_\alpha\Sigma_\alpha(\vartheta)\Pi_\alpha\right),
\]
finding $\mathcal F_\alpha$ is equivalent to computing the ratio of the first and second diagonal entries of $\Sigma_\alpha(\vartheta,1)$. Here, the quantities are independent of $k$ by symmetry of the model. 

We now compute these values. Let $(\Theta',Z')$ be an independent copy of $(\Theta,Z)$. Then in the sparse regime, Corollary~\ref{Cor::DCS} gives
\[
    \Sigma_\alpha(\vartheta,1)=\frac{\vartheta^{1-2\alpha}}{\gamma^{2\alpha}}\mathbb E_{\Theta',Z'}\left[\Delta_{\alpha,1}\Delta_{\alpha,1}^\top\right].
\]
Note that, when taking the ratio $\sigma^2_{N,\alpha,k}/\sigma^2_{T,\alpha,k}$, the scalar factor will cancel. We now only need to compute the quantities 
\[
    \E_{\Theta',Z'}\left[(\Delta_{\alpha,1})_1^2\right],
    \qquad
    \E_{\Theta',Z'}\left[(\Delta_{\alpha,1})_2^2\right].
\]

Substituting the preceding expressions into the definition of $\Delta_{\alpha,1}$ gives
\begin{align}
    \Delta_{\alpha,1}&=\sqrt{\Theta'p}\begin{bmatrix}\displaystyle\frac{1}{\sqrt{s}}\left(\frac{(\Theta')^{1-2\alpha}}{a_\alpha}-\frac{\alpha}{\bar\theta}\right)\\[1.2em]
    \displaystyle\frac{1}{\sqrt{h}}\left(\frac{(\Theta')^{1-2\alpha}}{a_\alpha}-\frac{\alpha\kappa}{\bar\theta}\right)
    \end{bmatrix},
    && Z'=1,\label{Eq::Deltaalpha11}\\[1em]
    \Delta_{\alpha,1}
    &=\sqrt{\Theta'q}\begin{bmatrix}\displaystyle\frac{1}{\sqrt{s}}\left(\frac{(\Theta')^{1-2\alpha}}{a_\alpha}-\frac{\alpha}{\bar\theta}\right)\\[1.2em]
    \displaystyle
    \frac{1}{\sqrt{h}}\left(-\frac{(\Theta')^{1-2\alpha}}{a_\alpha}-\frac{\alpha\kappa}{\bar\theta}\right)\end{bmatrix},
    && Z'=2\label{Eq::Deltaalpha12}.
\end{align}

We thus have
\begin{align}
   \mathbb E_{\Theta',Z'}\left[(\Delta_{\alpha,1})_1^2\right]
   &=\mathbb E_{\Theta'}\left[\frac{\Theta'}{s}\left(\frac{(\Theta')^{1-2\alpha}}{a_\alpha}-\frac{\alpha}{\bar\theta}\right)^2\left(\frac{p+q}{2}\right)\right]\nonumber\\
    &=\mathbb E_{\Theta'}\left[\Theta'\left(\frac{(\Theta')^{1-2\alpha}}{a_\alpha}-\frac{\alpha}{\bar\theta}\right)^2\right]\nonumber\\
    &=\frac{\E\left[(\Theta')^{3-4\alpha}\right]}{a^2_\alpha}-\frac{2\alpha\E\left[(\Theta')^{2-2\alpha}\right]}{a_\alpha\bar\theta}+\frac{\alpha^2\E[\Theta']}{\bar\theta^2}\nonumber\\
    &=\frac{b_\alpha}{a_\alpha^2}-\frac{2\alpha}{\bar\theta}+\frac{\alpha^2}{\bar\theta}\nonumber\\
    &=\frac{b_\alpha}{a_\alpha^2}-\frac{\alpha(2-\alpha)}{\bar\theta}\label{Eq::NormalVar}.
\end{align}

Additionally,
\begin{align}
    \E_{\Theta',Z'}\left[(\Delta_{\alpha,1})_2^2\right]&=\mathbb E_{\Theta'}\left[\frac{\Theta'}{2h}\left(p\left(\frac{(\Theta')^{1-2\alpha}}{a_\alpha}-\frac{\alpha\kappa}{\bar\theta}\right)^2+q\left(-\frac{(\Theta')^{1-2\alpha}}{a_\alpha}-\frac{\alpha\kappa}{\bar\theta}\right)^2\right)\right]\nonumber\\
    &=\mathbb E_{\Theta'}
    \left[\Theta'\left(\frac{1}{\kappa}\frac{(\Theta')^{2-4\alpha}}{a_\alpha^2}-\frac{2\alpha\kappa}{\bar\theta}\frac{(\Theta')^{1-2\alpha}}{a_\alpha}+\frac{\alpha^2\kappa}{\bar\theta^2}\right)\right]\nonumber\\
    &=\frac{b_\alpha}{\kappa a_\alpha^2}-\frac{2\alpha\kappa}{\bar\theta}+\frac{\alpha^2\kappa}{\bar\theta}\nonumber\\
    &=\frac{b_\alpha}{\kappa a_\alpha^2}-\frac{\kappa\alpha(2-\alpha)}{\bar\theta}\label{Eq::TanVar}.
\end{align}

Defining
\[
    r_\alpha=\frac{a_\alpha^2}{\bar\theta b_\alpha}=\frac{\left(\mathbb E[\Theta^{2-2\alpha}]\right)^2}{\bar\theta\mathbb E[\Theta^{3-4\alpha}]}, 
\]
and combining Equations~\eqref{Eq::NormalVar} and \eqref{Eq::TanVar} gives
\[
    \mathcal{F}_\alpha=\frac{\sigma^2_{N,\alpha}}{\sigma^2_{T,\alpha}}=\frac{\kappa\left(1-\alpha(2-\alpha)r_\alpha\right)}{1-\kappa^2\alpha(2-\alpha)r_\alpha}.
\]

\subsubsection{Connection to Chernoff Information}
\label{Sec::ChernoffBenchmark}
In Section~\ref{Sec::ComparingDegreeNorm}, we compare degree-$\alpha$ embeddings using the projected Gaussian Bayes error $\mathcal{E}_\alpha^{(n)}$, which incorporates the row-normalization step commonly used in spectral clustering. This diagnostic is flexible, but it is computed numerically and relies on finite-$n$ Gaussian approximations. In this appendix, we consider an alternative perspective based on the Chernoff information analysis of \citet{Cape2019}. In the balanced two-community SBM, this approach yields an analytic benchmark against which our numerical approximations may be compared. Throughout the calculations below, we take $\alpha\in[0,1)$. At $\alpha=1$, the limiting covariance may be singular, so we define the corresponding quantities by continuous extension as $\alpha\uparrow1$.

We first recall some notation from Section~\ref{Sec::ComparingDegreeNorm}. For each $\alpha\in[0,1)$ and $k\in\{1,2\}$, let $m_{\alpha,k}$ and $\Sigma_{\alpha,k}$ denote the limiting population center and covariance matrix of a node in community $k$, respectively. We approximate the distribution of a scaled embedded row from community $k$ by
\[
    \mathcal{G}_{k,\alpha}^{(n)}=N\left(\mu_k,\Sigma_{\alpha,k}\right),\qquad k=1,2,
\]
where $\mu_k=\sqrt{n}m_{\alpha,k}$. For a single node, the idealized community-detection problem may then be viewed as distinguishing between $\mathcal{G}_{1,\alpha}^{(n)}$ and $\mathcal{G}_{2,\alpha}^{(n)}$. Chernoff information provides a way to quantify the difficulty of distinguishing between these two distributions.

\begin{definition}[Chernoff Information]
    Let $F_1$ and $F_2$ be two absolutely continuous multivariate distributions on $\mathbb{R}^r$ and let $f_1$ and $f_2$ be their densities, respectively. Chernoff divergence is defined to be 
        \begin{align*}
        C_t(F_1,F_2)&=-\log\left(\int_{\mathbb{R}^r}f_1^{1-t}(x)f_2^{t}(x)dx\right).
    \end{align*}
    Chernoff information is defined to be 
    \begin{align*}
        C(F_1,F_2)&=-\log\left(\inf_{t\in(0,1)}\int_{\mathbb{R}^r}f_1^{1-t}(x)f_2^{t}(x)dx\right)\\
        &=\sup_{t\in(0,1)}\left(-\log\int_{\mathbb{R}^r}f_1^{1-t}(x)f_2^{t}(x)dx\right)\\
        &=\sup_{t\in(0,1)}C_t(F_1,F_2).
    \end{align*}
\end{definition}

This quantity may be explicitly computed for multivariate Gaussian distributions. Define $\Sigma_{12}(t)=t\Sigma_1+(1-t)\Sigma_2$ and let $F_1=N(\mu_1,\Sigma_1)$ and $F_2=N(\mu_2,\Sigma_2)$. Then 
\begin{align}
    \label{Eq::ChernInfMVNormal1}
    C(F_1,F_2)=\frac{1}{2}\sup_{t\in(0,1)}\left(\left(t(1-t)(\mu_1-\mu_2)^\top\Sigma_{12}^{-1}(t)(\mu_1-\mu_2)\right)+\log\frac{\left\lvert\Sigma_{12}(t)\right\rvert}{\left\lvert\Sigma_{1}\right\rvert^t\left\lvert\Sigma_{2}\right\rvert^{1-t}}\right).
\end{align}

Note that $(\mu_1-\mu_2)^\top\Sigma_{12}^{-1}(t)(\mu_1-\mu_2)$ is the square of the Mahalanobis norm with respect to $\Sigma_{12}(t)$. We thus rewrite Equation~\eqref{Eq::ChernInfMVNormal1} as
\begin{align*}
    C(F_1,F_2)=\frac{1}{2}\sup_{t\in(0,1)}\left(\left(t(1-t)\left\lVert\mu_1-\mu_2\right\rVert^2_{\Sigma_{12}^{-1}(t)}\right)+\log\frac{\left\lvert\Sigma_{12}(t)\right\rvert}{\left\lvert\Sigma_{1}\right\rvert^t\left\lvert\Sigma_{2}\right\rvert^{1-t}}\right).
\end{align*}

We now apply this formula to the two Gaussian approximations associated with the degree-$\alpha$ embedding. Because their means depend on $n$, the resulting Chernoff information, denoted by $\eta_\alpha$, also depends on $n$. We suppress this dependence in the notation, as this dependence will cancel in the ratio after taking a large-sample approximation. 
\[
    \eta_\alpha=\frac{1}{2}\sup_{t\in(0,1)}\left[t(1-t)\left\lVert\mu_1-\mu_2\right\rVert^2_{\Sigma^{-1}_{\alpha,12}(t)}+\log\left(\frac{\left\lvert\Sigma_{\alpha,12}(t)\right\rvert}{\left\lvert\Sigma_{\alpha,1}\right\rvert^t\left\lvert\Sigma_{\alpha,2}\right\rvert^{1-t}}\right)\right],
\]
where $\Sigma_{\alpha,12}(t):=t\Sigma_{\alpha,1}+(1-t)\Sigma_{\alpha,2}$.

From Theorem~\ref{Thm::Main}, the means of the Gaussian distributions are given by the rows of 
\[
    n^{\alpha+1/2}\rho_n^\alpha X_\alpha.
\] 

If $t_i\asymp n\rho_n$, then
\[
    n^{\alpha+1/2}\rho_n^\alpha(X_\alpha)_{i*}=n^{\alpha+1/2}\rho_n^\alpha t_i^{-\alpha}X_i^\top\asymp \sqrt{n\rho_n}.
\]
Thus, in the dense SBM considered below, the quantity $\mu_1-\mu_2$ is of order $\sqrt{n}$. On the other hand, the covariance matrices are independent of $n$. Therefore, for large $n$, the second term in the previous equation is dominated by the first. We then have a large-sample approximation of the Chernoff information as 
\begin{equation}\label{Eq::RhoStarApprox}
    \eta^*_\alpha\approx\frac{1}{2}\sup_{t\in(0,1)}\left[t(1-t)\left\lVert\mu_1-\mu_2\right\rVert^2_{\Sigma^{-1}_{\alpha,12}(t)}\right].
\end{equation}

Equation~\eqref{Eq::RhoStarApprox} implies that the leading approximation to the Chernoff information increases as $n$ increases. Consider the degree-$\alpha$ and degree-$\alpha'$ embeddings. To compare the relative efficiency between the two choices, we may define 
\[
    \eta^*_{\alpha,\alpha'}=\frac{\eta^*_{\alpha}}{\eta^*_{\alpha'}}.
\] 
This quantity describes the relative \emph{approximate} Chernoff information between $\alpha$ and $\alpha'$. A large value suggests that $\alpha$ may provide more separation of the two communities, and thus may be a more appropriate choice for spectral clustering. We now study the behavior of $\eta_{\alpha,\alpha'}^*$ under the explicit two-community SBM considered in this section. 

The asymptotic means and covariances from Theorem~\ref{Thm::Main} simplify in this setting, allowing an explicit form for $\eta_{\alpha,\alpha'}^*$ to be computed. Carrying out the calculation (see Section~\ref{Sec::RhoStarDerivation}) gives
\begin{equation}\label{Eq::RhoStarExplicit}
    \eta_{\alpha,\alpha'}^*=\frac{C_2\alpha'^2+C_1\alpha'+C_0}{C_2\alpha^2+C_1\alpha+C_0},
\end{equation}
where 
\begin{align*}
    C_2&=(p-q)^2(p^2-p+q^2-q),\\
    C_1&=-2(p-q)^2(p+q)(p+q-1),\\
    C_0&=(p+q)^2(p^2-p+q^2-q). 
\end{align*}

Note that when $\alpha=\alpha'$, the ratio simplifies to $1$, as expected. When $\alpha=0$ (corresponding to the adjacency matrix) and $\alpha'=1/2$ (corresponding to the symmetric Laplacian), the ratio simplifies to equation (14) of \citet{Cape2019}. If we fix $\alpha'=0$, we may determine when the ratio is largest, indicating the $\alpha$ at which the relative Chernoff information is maximized. For a fixed $p$ and $q$, the ratio is clearly maximized at $-\frac{C_1}{2C_2}$. Thus, for each $p$ and $q$, the optimal $\alpha$ may be computed to be 
\begin{equation}\label{Eq::AlphaOpt}
    \alpha_{\mathrm{Ch}}=\begin{cases}
        0\text{ if }-\frac{C_1}{2C_2}<0\\
        1\text{ if }-\frac{C_1}{2C_2}>1\\
        -\frac{C_1}{2C_2}\text{ otherwise}
    \end{cases}
\end{equation}

\begin{figure}[t]
    \centering
    \includegraphics[width=0.6\textwidth]{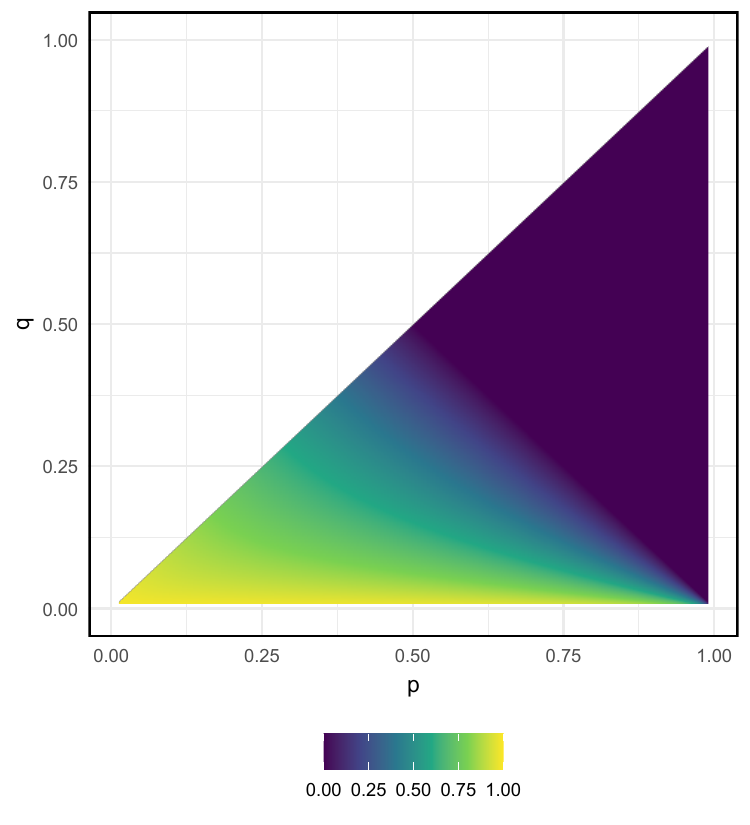}
    \caption{Chernoff-optimal $\alpha_{\mathrm{Ch}}$ for the balanced symmetric two-community SBM. The dark region denotes $\alpha_{\mathrm{Ch}}=0$, corresponding to adjacency spectral clustering being preferred. Lighter colors represent higher $\alpha_{\mathrm{Ch}}$, with stronger normalization being preferred when block probabilities are smaller.}
    \label{Fig::HomogeneousBalancedComparison}
\end{figure}

Figure~\ref{Fig::HomogeneousBalancedComparison} shows $\alpha_{\mathrm{Ch}}$ over the region $0<q<p<1$. At each point, the optimal $\alpha$ is computed using Equation~\eqref{Eq::AlphaOpt}. The most striking feature of the graph is the phase transition, corresponding to the case when $C_1=0$. This occurs when either $p=q$, which corresponds to the degenerate model, or when $p+q=1$, which is visible as the diagonal line in the figure. Above this boundary, adjacency spectral clustering provides higher Chernoff information for clustering, and below this boundary, some amount of degree normalization is preferred. In general, as the block probabilities decrease, increasing $\alpha$ provides higher Chernoff information.

\subsubsection{Derivation of the Relative Chernoff Information}
\label{Sec::RhoStarDerivation}
We now derive Equation~\eqref{Eq::RhoStarExplicit} for the balanced symmetric two-community SBM. We recall that in this section, we work in the dense regime $\rho_n\equiv1$ and take $\Theta\equiv1$. Hence $\bar\theta=1$, $\gamma_1=\gamma_2=s$, and
\[
    \tilde\Upsilon_\alpha=s^{-2\alpha}\begin{bmatrix}s & 0\\0 & h\end{bmatrix}.
\]
Define
\[
    v_p=p(1-p),\qquad v_q=q(1-q),
\]
and, for $k\in\{1,2\}$, write
\[
    m_{\alpha,k}^{(n)}:=\sqrt{n}\,m_{\alpha,k}=\sqrt{n}\,s^{-\alpha}\nu_k,\qquad\Sigma_{\alpha,k}:=\Sigma_\alpha(1,k).
\]

Using Equations~\eqref{Eq::Deltaalpha11} and \eqref{Eq::Deltaalpha12} in the specific SBM setting gives
\begin{align}
    \Delta_{\alpha,1}&=\sqrt{p}
    \begin{bmatrix}
        \dfrac{1-\alpha}{\sqrt{s}}\\[0.8em]
        \dfrac{1-\alpha\kappa}{\sqrt{h}}
    \end{bmatrix},
    && Z'=1,\\[1em]
    \Delta_{\alpha,1}
    &= \sqrt{q}
    \begin{bmatrix}
        \dfrac{1-\alpha}{\sqrt{s}}\\[0.8em]
        -\dfrac{1+\alpha\kappa}{\sqrt{h}}
    \end{bmatrix},
    && Z'=2.
\end{align}

Since $\mathbb P(Z'=1)=\mathbb P(Z'=2)=1/2$, the dense-regime covariance formula in Corollary~\ref{Cor::DCS} gives
\[
    \Sigma_{\alpha,1}=\frac{1}{s^{2\alpha}}\mathbb E_{Z'}\left[\left(1-B_{1Z'}\right)\Delta_{\alpha,1}\Delta_{\alpha,1}^{\top}
    \right].
\]
Expanding this expectation, we obtain
\[
    \Sigma_{\alpha,1}=\begin{bmatrix}A_\alpha & B_\alpha\\B_\alpha & C_\alpha\end{bmatrix},
\]
where, writing
\[
    \gamma_+(\alpha)=s+\alpha h,\qquad\gamma_-(\alpha)=s-\alpha h,
\]
we have
\[
    A_\alpha=\frac{(1-\alpha)^2(v_p+v_q)}{2s^{2\alpha+1}},
\]
\[
    B_\alpha=\frac{(1-\alpha)\left(v_p\gamma_-(\alpha)-v_q\gamma_+(\alpha)\right)}{2\sqrt{h}\,s^{2\alpha+3/2}},
\]
and
\[
    C_\alpha=\frac{v_p\gamma_-^2(\alpha)+v_q\gamma_+^2(\alpha)}{2h\,s^{2\alpha+2}}.
\]
By symmetry, the covariance for a node in community $2$ is
\[
    \Sigma_{\alpha,2}=\begin{bmatrix}A_\alpha & -B_\alpha\\-B_\alpha & C_\alpha\end{bmatrix}.
\]
We may now compute $\eta^*_{\alpha,\alpha'}$. To do so, we first compute the expression from Equation~\eqref{Eq::RhoStarApprox}, which we restate below:
\[
    \eta^*_\alpha\approx\sup_{t\in(0,1)}\left[\frac{t(1-t)}{2}\left\lVert m_{\alpha,1}^{(n)}-m_{\alpha,2}^{(n)}\right\rVert^2_{\Sigma_{\alpha,12}^{-1}(t)}\right],
\]
where
\[
    \Sigma_{\alpha,12}(t)=t\Sigma_{\alpha,1}+(1-t)\Sigma_{\alpha,2}.
\]
Equivalently,
\[
    \eta^*_\alpha\approx\sup_{t\in(0,1)}\left[\frac{t(1-t)}{2}\left(m_{\alpha,1}^{(n)}-m_{\alpha,2}^{(n)}\right)^\top\left(t\Sigma_{\alpha,1}+(1-t)\Sigma_{\alpha,2}\right)^{-1}\left(m_{\alpha,1}^{(n)}-m_{\alpha,2}^{(n)}\right)\right].
\]

We compute this quantity in two steps. First, we show that the supremum is achieved at $t=1/2$. Second, we use this fact to derive an explicit expression for $\eta_\alpha^*$. We begin with the first claim. Define
\begin{equation}
    \label{Eq::ft}
    f(t)=t(1-t)\left(m_{\alpha,1}^{(n)}-m_{\alpha,2}^{(n)}\right)^\top\left(t\Sigma_{\alpha,1}+(1-t)\Sigma_{\alpha,2}\right)^{-1}\left(m_{\alpha,1}^{(n)}-m_{\alpha,2}^{(n)}\right).
\end{equation}
Since
\[
    t\Sigma_{\alpha,1}+(1-t)\Sigma_{\alpha,2}=\begin{bmatrix}
A_\alpha & (2t-1)B_\alpha\\
(2t-1)B_\alpha & C_\alpha
\end{bmatrix},
\]
we have
\begin{equation}
    \label{Eq::tsigmainverse}
    \left(t\Sigma_{\alpha,1}+(1-t)\Sigma_{\alpha,2}\right)^{-1}=\frac{1}{A_\alpha C_\alpha-(2t-1)^2B_\alpha^2}\begin{bmatrix}C_\alpha & -(2t-1)B_\alpha\\-(2t-1)B_\alpha & A_\alpha\end{bmatrix}.
\end{equation}
Additionally,
\begin{equation}
    \label{Eq::mdiff}
    m_{\alpha,1}^{(n)}-m_{\alpha,2}^{(n)}=2\sqrt{nh}\,s^{-\alpha}\begin{bmatrix}0\\1\end{bmatrix}.
\end{equation}
Plugging Equations~\eqref{Eq::tsigmainverse} and \eqref{Eq::mdiff} into Equation~\eqref{Eq::ft} gives
\[
    f(t)=\frac{4nh\,s^{-2\alpha}t(1-t)A_\alpha}{A_\alpha C_\alpha-(2t-1)^2B_\alpha^2}.
\]
Elementary calculus gives
\[
    f'(t)=-\frac{4nhs^{-2\alpha}A_\alpha(2t-1)(A_\alpha C_\alpha-B_\alpha^2)}{(A_\alpha C_\alpha-(2t-1)^2B_\alpha^2)^2}.
\]
Note that 
\[
    A_\alpha C_\alpha-B_\alpha^2
    =
    \frac{(1-\alpha)^2v_pv_q}{hs^{4\alpha+1}}>0,
    \qquad \alpha\in[0,1),
\]
and this quantity is exactly zero at $\alpha=1$. Additionally, the denominator is clearly positive, and the quantities $4nhs^{-2\alpha}$ and $A_\alpha$ are also positive. Thus, the sign of $f'(t)$ is completely determined by the sign of $-(2t-1)$. The derivative is positive when $t<1/2$ and negative when $t>1/2$, and $f(t)$ is uniquely maximized at $t=1/2$. At $t=1/2$, we have
\[
    \Sigma_{\alpha,12}(1/2)=\frac{1}{2}\left(\Sigma_{\alpha,1}+\Sigma_{\alpha,2}\right)=\begin{bmatrix}A_\alpha & 0\\0 & C_\alpha\end{bmatrix}.
\]
Thus,
\[
    \Sigma_{\alpha,12}^{-1}(1/2)=\begin{bmatrix}A_\alpha^{-1} & 0\\0 & C_\alpha^{-1}\end{bmatrix}.
\]
Since
\[
    m_{\alpha,1}^{(n)}-m_{\alpha,2}^{(n)}=2\sqrt{nh}\,s^{-\alpha}\begin{bmatrix}0\\1\end{bmatrix},
\]
we have
\[
    \left\lVert m_{\alpha,1}^{(n)}-m_{\alpha,2}^{(n)}\right\rVert^2_{\Sigma_{\alpha,12}^{-1}(1/2)}=\frac{8nh^2s^2}{v_p\gamma_-(\alpha)^2+v_q\gamma_+(\alpha)^2}.
\]
We thus have
\[
    \eta_{\alpha,\alpha'}^*=\frac{v_p\gamma_-(\alpha')^2+v_q\gamma_+(\alpha')^2}{v_p\gamma_-(\alpha)^2+v_q\gamma_+(\alpha)^2}.
\]
Plugging in the values for $v_p$, $v_q$, $\gamma_+$, and $\gamma_-$ gives
\[
    \eta_{\alpha,\alpha'}^*=\frac{C_2\alpha'^2+C_1\alpha'+C_0}{C_2\alpha^2+C_1\alpha+C_0},
\]
where
\[
    C_2=(p-q)^2(p^2-p+q^2-q),
\]
\[
    C_1=-2(p-q)^2(p+q)(p+q-1),
\]
and
\[
    C_0=(p+q)^2(p^2-p+q^2-q).
\]
This establishes Equation~\eqref{Eq::RhoStarExplicit}.
\end{document}